\documentclass[]{fairmeta}
\usepackage{amssymb}
\usepackage{amsmath}
\usepackage{amsthm}
\usepackage{caption}
\usepackage{thmtools,thm-restate}
\newtheorem{theorem}{Theorem}
\newtheorem{corollary}[theorem]{Corollary}
\newtheorem{proposition}[theorem]{Proposition}
\newcommand{\E}{\mathbb{E}}
\newcommand{\Prob}{\mathbb{P}}

\usepackage{xspace}
\usepackage{url}
\usepackage{xcolor}
\usepackage{colortbl}       
\usepackage{fontawesome}
\usepackage{algorithm}
\usepackage{algpseudocode}
\makeatletter
\newcommand{\algparbox}[1]{\parbox[t]{\dimexpr\linewidth-\ALG@thistlm-1pt\relax}{#1\strut}}
\makeatother
\usepackage{listings}
\lstdefinestyle{pseudopy}{
  language=Python, basicstyle=\ttfamily\footnotesize,
  commentstyle=\color{gray}\itshape, keywordstyle=\bfseries,
  columns=fullflexible, keepspaces=true, breaklines=true,
  frame=single, framerule=0.4pt, xleftmargin=6pt, xrightmargin=6pt,
  showstringspaces=false, escapechar=|}
\definecolor{lightestblue}{HTML}{F5FAFF}
\definecolor{titleblue}{HTML}{E4F1FF} 
\definecolor{frameblue}{HTML}{C0DFFF} 
\definecolor{textblack}{HTML}{000000} 
\tcbset{
  mybox/.style={
    colback=lightestblue,     
    colframe=frameblue,     
    coltitle=textblack,         
    fonttitle=\bfseries,      
    colbacktitle=titleblue,   
    enhanced,
    attach boxed title to top left={yshift=-1mm,xshift=2mm}, 
    boxed title style={
      colframe=frameblue,
      colback=titleblue,
      boxrule=1.2pt,
    },
    boxrule=1.2pt,            
    arc=2mm,                  
    top=4pt, bottom=4pt, left=6pt, right=6pt,
  }
}

\definecolor{harnessblue}{RGB}{235,243,252}
\definecolor{harnessframe}{RGB}{125,165,205}
\usepackage{tcolorbox}
\newtcolorbox{takeaways}{colback=harnessblue, colframe=harnessframe, boxrule=0.75pt, arc=1.5pt, left=7pt, right=7pt, top=5pt, bottom=5pt}

\title{Sharpening Tax in Post-Training}
\usepackage{enumitem}
\definecolor{basegreen}{RGB}{22,120,60}
\definecolor{basegreenbg}{RGB}{240,249,242}
\definecolor{rlred}{RGB}{178,34,34}
\definecolor{rlredbg}{RGB}{253,243,242}
\definecolor{goalgray}{RGB}{92,92,92}
\definecolor{loopamber}{RGB}{176,108,0}
\definecolor{loopamberbg}{RGB}{255,248,230}

\newlist{trajsteps}{enumerate}{1}
\setlist[trajsteps]{label=\scriptsize\arabic*.,leftmargin=1.4em,itemsep=1pt,topsep=2pt,parsep=0pt}

\newtcolorbox{trajpanel}[5][trajpanelgrp]{%
  enhanced, colback=#3, colframe=#2, boxrule=0.8pt, arc=1.5mm,
  left=2mm, right=2mm, top=1.5mm, bottom=1.5mm,
  title={\small\bfseries\strut #4},
  colbacktitle=#2, coltitle=white, toptitle=1mm, bottomtitle=1mm,
  equal height group=#1,
  before upper={\notblank{#5}{\noindent\hfill
    {\footnotesize\color{goalgray}#5}\par\vspace{1pt}}{}}}

\newtcolorbox{taskbanner}{enhanced, colback=gray!6, colframe=gray!60,
  boxrule=0.6pt, arc=1.5mm, left=2mm, right=2mm, top=1mm, bottom=1mm}

\author[1,2,*]{Changdae Oh}
\author[1]{Qi Zeng}
\author[1]{Qi Qi}
\author[]{Andrey Zhmoginov}
\author[1]{Deren Lei}
\author[1]{Yun He}
\author[1,3,*]{Hoang Phan}
\author[4]{\mbox{Hangoo Kang}}
\author[4]{\mbox{Azalia Mirhoseini}}
\author[2]{Sharon Li}

\affiliation[1]{Meta Superintelligence Labs}
\affiliation[2]{University of Wisconsin--Madison}
\affiliation[3]{NYU}
\affiliation[4]{Stanford University}

\contribution[*]{Work done at Meta}

\abstract{
An emerging hypothesis about reinforcement learning (RL) post-training of large language models (LLMs) is that it merely sharpens existing behaviors of a base model, improving single-shot accuracy at the cost of solution coverage. Although this trade-off has been observed in math and coding tasks, it need not extend to agentic tasks, where multi-turn tool use and interaction may require capabilities newly acquired during post-training.
Our surprising finding is that pre-trained LLMs, equipped with a light inference harness, can serve as capable agents. Despite far lower accuracy (pass@$1$), they often surpass their post-trained counterparts in solution coverage (pass@$K$) given a sufficient test-time budget. We further analyze the underlying mechanism and show that post-training pushes tasks toward two extremes, always solved or never solved, and thereby improves sampling efficiency and consistency at the cost of solution coverage. To measure this cost, we propose \textbf{Sharpening Tax}, a diagnostic metric that quantifies the loss in test-time scalability after post-training. Across 14 base/post-trained model pairs from four families and three agentic benchmarks (42 cases in total), the tax is prevalent in most settings, can be estimated from a few rollouts, and correlates well with other metrics. Finally, we present \textbf{posterior-tempered group sampling (PTGS)}, a simple plug-and-play Bayesian sampler that adapts the sampling temperature per prompt to its estimated difficulty. Applied during RL training in two agentic environments, PTGS pays a smaller tax than the fixed-temperature baseline, solving more tasks under repeated sampling while also improving single-shot accuracy.
}

\correspondence{Changdae Oh (\email{changdae@cs.wisc.edu}) and Sharon Li (\email{sharonli@cs.wisc.edu})}
\metadata[Project Page]{\url{https://changdaeoh.github.io/sharpening-tax/}}
\metadata[Code]{\url{https://github.com/changdaeoh/sharpening-tax}} 

\begin{document}
\addtocontents{toc}{\protect\setcounter{tocdepth}{-1}}

\maketitle

\section{Introduction}
\label{section:intro}

Post-training at scale, particularly via reinforcement learning (RL), has transformed pre-trained general text completion machines into reasoning models~\citep{shao2024deepseekmath,guo2025deepseek}. Beyond competition-level math and coding reasoners, post-trained large language models (LLMs) now serve as goal-oriented agents that call tools, manage long context, and interact with an environment across multiple turns~\citep{anthropic_claude_code,openai_codex,google2025antigravity,meta2026musesparkblog,xai2026grokbuild}. A common belief behind this progress is that frontier-level reasoning capability emerges during intensive post-training, e.g., RL unlocks new capabilities and pushes the reasoning boundary beyond what the pre-trained base model can already reach~\citep{uesato2022solving,guo2025deepseek}. However, this belief has recently been called into question.

\textit{Does RL post-training create fundamentally new capabilities, or does it merely amplify a few rewarding behaviors that the base model already possesses, i.e., \textbf{distribution sharpening}?} This question has attracted broad attention~\citep{yue2025does,zhao2025echo,wu2025invisible,yuan2026f,wen2026reinforcement,shen2026rlvr,zhou2026rlvr}. A dominant observation so far is that an RL-trained policy gains accuracy (pass@$1$) at the expense of solution coverage (pass@$K$)~\citep{yue2025does,zhao2025echo}. Despite lots of reasonable positions, from advocacy to skepticism, with training-dependent~\citep{liu2025prorl} and data-dependent~\citep{zhang2025interplay,shen2026understanding} viewpoints in between, current evidence is mainly limited to math and coding tasks~\citep{yue2025does,zhao2025echo,he2025rewarding,shao2026spurious}.
Unfortunately, LLM coverage analysis in these domains does not settle the question since (1) pre-training provides enormous exposure to math and coding; (2) evaluation often checks only a final answer, so a lucky rollout with flawed reasoning earns credit~\citep{NEURIPS2022_3d719fee, pan2024feedback, wen2026reinforcement}.

We argue that agentic tasks provide a stress testing evaluation suite to test the sharpening hypothesis. Unlike math and coding problems, for which a base model may already encode many candidate solution strategies, a model in agentic tasks must execute structured actions~\citep{patil2025bfcl}, invoke diverse tools through scenario-specific interfaces~\citep{schick2023toolformer}, incorporate heterogeneous environmental feedback under uncertainty~\citep{oh2026uncertainty}, and stay coherent over long-horizon trajectories~\citep{froger2026gaia2}. Those behaviors are far less common in pre-training data mix and are widely believed to be learned in post-training~\citep{zeng2024agenttuning, su2026scaling}. Therefore, one might expect that post-training does more than redistribute probability mass among behaviors already present in the base model; it may be necessary to make successful agentic trajectories reachable at all. 

\begin{figure*}[t]
     \centering
     \includegraphics[width=0.95\textwidth]{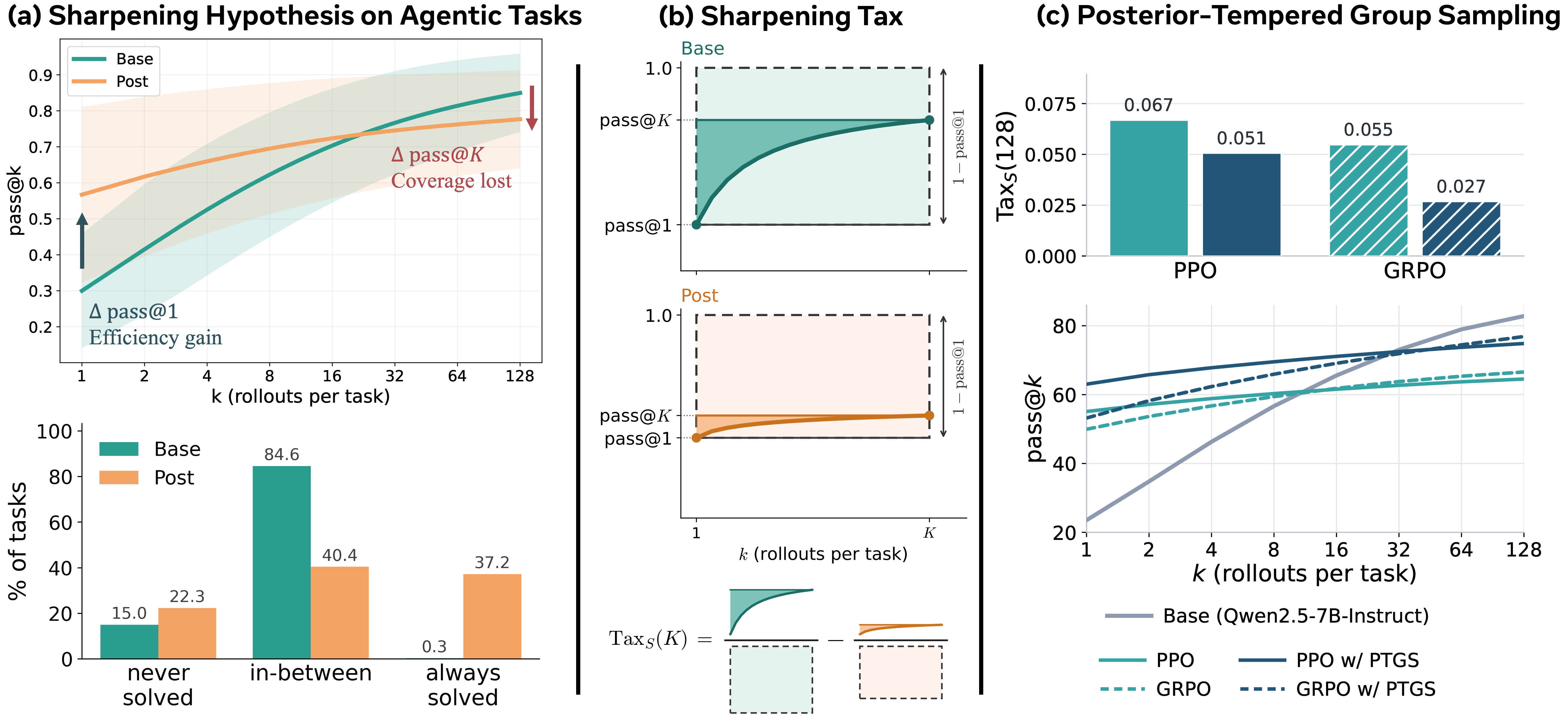}
     \vspace{-0.675em}
     \caption{\textbf{Project overview}. We investigate whether post-training fundamentally broadens the agentic reasoning capacity of its base model. (a) Our observations suggest that post-training sharpens the policy to improve sampling efficiency and consistency while compromising coverage. To quantify this systematically, (b) we present \textbf{Sharpening Tax}, measuring the difference in test-time scalability between base and post-trained LLMs. To lower the tax, (c) we then propose \textbf{Posterior-Tempered Group Sampling (PTGS)}, dynamically adjusting the sampling temperature based on task difficulty.}
     \vspace{-1.2em}
     \label{fig:overview}
\end{figure*}

To test this expectation, we systematically study LLMs' reasoning boundaries on agentic tasks before and after post-training. In particular, we first focus on the publicly available open-source checkpoint pairs of the base and post-trained LLMs to understand how large-scale general post-training affects the agentic reasoning boundary of base LLMs. 
Our surprising observation (\S\ref{section:observation}) is that while the post-trained LLMs lead in single-shot accuracy, pre-trained base LLMs can also perform agentic reasoning when paired with a lightweight harness, and they even surpass their post-trained counterparts in terms of solution coverage given enough test-time compute. We then analyze how post-training changes the per-task success probability distribution, revealing that it bimodalizes the mass towards two extremes, always solved or never solved, thereby forgoing the benefits of test-time parallel scaling of multiple rollouts to improve single-shot accuracy. 

Motivated by these test-time scaling dynamics, we propose \textbf{Sharpening Tax} (\S\ref{section:diagnosis}), a diagnostic metric that quantifies how much post-training shrinks the effective reasoning coverage as a single scalar. Across 14 model backbones from popular open source model families and three representative basic benchmarks for agentic tasks, we find that the tax is pervasive, implying that modern post-training consistently trades coverage for sampling efficiency; then we highlight the practical usefulness of Sharpening Tax, which can be estimated from a handful of rollouts to predict the future tax of many rollouts as well as other performance metrics.

Finally, we show that even task-specific RL tuning on a single domain pays Sharpening Tax. To mitigate this, we present \textbf{posterior-tempered group sampling (PTGS)}, a general sampler that balances exploration and exploitation. It adapts the sampling temperature based on per-task difficulty to smooth policy for hard prompts and sharpen it for easy ones. Plugged into common RL algorithms (e.g., PPO and GRPO), PTGS pays a smaller tax while improving accuracy and coverage simultaneously over the baseline in agentic environments (\S\ref{section:solution}). We further provide theoretical analyses (\S\ref{section:theory}) to explain what the tax measures, why sharpening charges it, and how PTGS improves RL. Fig.~\ref{fig:overview} shows the overview, and contributions are summarized as follows:
\begin{itemize}[leftmargin=*, itemsep=2pt, topsep=2pt, parsep=0pt]
    \item We systematically study how post-training affects the average performance (pass@$1$) and solution coverage (pass@$K$) of LLMs on multi-turn, interactive agentic tasks, and show that harness-equipped base models can match or surpass their post-trained counterparts in terms of coverage.
    \item To enable diagnosis at scale, we propose \textit{Sharpening Tax}, a metric that summarizes the effect of post-training on test-time scalability in a single number. Across 42 model-benchmark combinations, we find that the tax is substantial and pervasive, predictable from a few rollouts, and closely tied to other evaluation metrics. 
    \item We present \textit{posterior-tempered group sampling (PTGS)}, a plug-and-play sampler which adaptively sharpens or smoothens the sampling temperature per prompt; it mitigates the tax, improving accuracy and coverage simultaneously in multi-turn RL of LLM agents by increasing the chance of getting informative rollout.
\end{itemize}

\section{Preliminaries}
\label{section:prelim}

\noindent\textbf{Metrics.}
LLM reasoning boundaries are commonly probed through pass@$k$ evaluation~\citep{yue2025does,zhao2025echo}. For example, an LLM policy generates $k$ independent rollouts for each task $x_i$, and success is evaluated across these trials. Let $n_i$ be the total number of rollouts sampled for task $x_i \in \mathcal{D}$, of which $c_i \le n_i$ succeed. Given a budget $k \le n_i$, we consider three complementary metrics, each computed with the standard unbiased estimator~\citep{chen2021evaluating,yao2024tau} below.

\begin{itemize}[itemsep=2pt, topsep=2pt, parsep=0pt]
    \item $\text{pass}@1_i = \frac{c_i}{n_i}$ \textbf{(Accuracy)}: Empirical pass rate of a single trial.
    \item $\text{pass}@k_i = 1 - \binom{n_i - c_i}{k} \big/ \binom{n_i}{k}$ \textbf{(Coverage)}: Probability that at least one of $k$ rollouts succeeds.
    \item $\text{pass}^k_i = \binom{c_i}{k} \big/ \binom{n_i}{k}$ \textbf{(Consistency)}: Probability that all $k$ rollouts succeed.
\end{itemize}

We report dataset-level metrics by averaging over all tasks, e.g., $\text{pass}@k = \frac{1}{|\mathcal{D}|} \sum_{i} \text{pass}@k_i$. Intuitively, pass@$1$ captures the sampling efficiency of a policy, pass@$k$ its solution coverage under a finite budget, and $\text{pass}^k$ its success reliability across multiple attempts.

\noindent\textbf{Datasets and environments.}
Most prior work evaluates pass@$K$ reasoning boundary on math and coding benchmarks, which sometimes check only the final outcome. As a result, an incorrect reasoning trajectory that stumbles onto a lucky final answer still counts as a success, inflating pass@$K$~\citep{wen2026reinforcement}. We instead evaluate on three agentic benchmarks that require multi-turn tool calling for final goal achievement: BFCL v4 multi-turn base split~\citep{patil2025bfcl}, WebShop~\citep{yao2022webshop}, and ACEBench~\citep{chen2025acebench}. In these environments, intermediate actions and state transitions are checked by design, offering a faithful setup for stress testing of the agentic reasoning boundary. See Appendix~\ref{appendix:setup} for more details.

\noindent\textbf{Models.}
Our question is whether post-training extends the reasoning boundary of the base model. To explore this at scale, rather than developing the base and post-trained LLMs from scratch, we leverage open-source checkpoint pairs consisting of a pre-trained base model and its post-trained counterpart (e.g., \texttt{gemma-4-31B} vs. \texttt{gemma-4-31B-it}). Specifically, the evaluation spans 14 backbones from \texttt{Gemma-4}~\citep{team2026gemma}, \texttt{Ministral-3}~\citep{liu2026ministral}, \texttt{Qwen2.5}~\citep{yang2024qwen25}, and \texttt{Qwen3.5}~\citep{qwen2026qwen35} families in HuggingFace collections, ranging from 3B to 35B effective parameters (Tab.~\ref{table:model_pairs}). The exact training recipes of these models are not fully disclosed, so we use the term loosely for RL post-trained models; since SFT and DPO~\citep{rafailov2023direct} also sharpen the policy distribution~\citep{huang2025self}, our analysis does not hinge on the exact recipe. Meanwhile, for fair comparison across different model backbones, we turn off thinking mode for all post-trained models that support the enforced explicit reasoning (we also ablate thinking mode in Figure~\ref{fig:thinkingmode}). See Appendix~\ref{appendix:models} for justification and details.

\begin{figure*}[t]
     \centering
     \vspace{-1.15em}
     \includegraphics[width=0.88\textwidth]{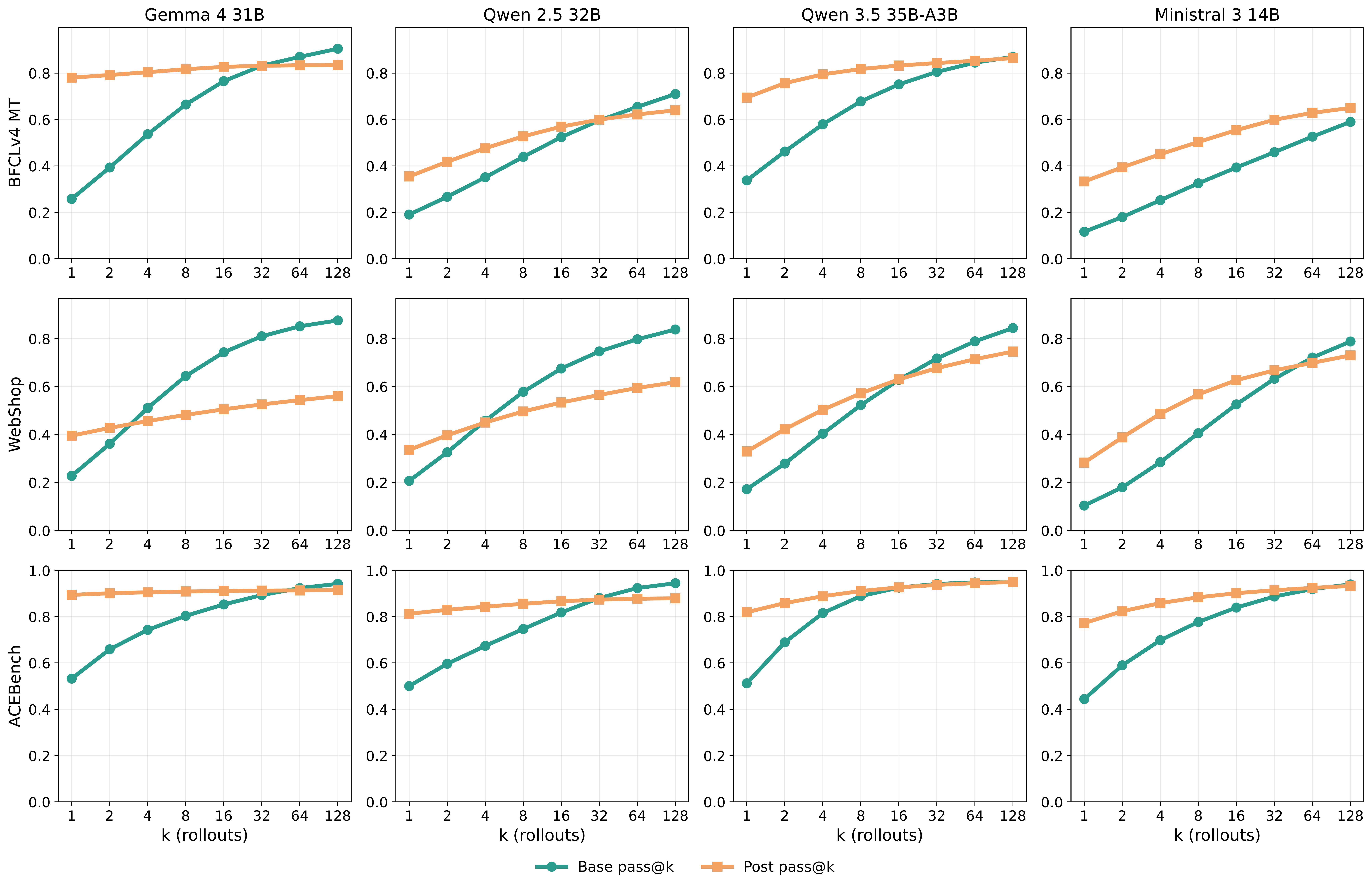}
     \vspace{-1.2em}
     \caption{\textbf{Base models catch up to their post-trained counterparts on agentic tasks}. We draw pass@$K$ curves (up to 128 rollouts per task) for four different models on three benchmarks. Post-trained models (orange) dominate at small $K$, but pre-trained base models equipped with a light harness (teal) scale more steeply and cross over in most settings.}
     \vspace{-1.4em}
     \label{fig:observation-tts}
\end{figure*}
\section{The Accuracy-Coverage Tension of Post-trained Agent Policy}
\label{section:observation}

Prior observations that base models have a higher performance upper bound (coverage) than post-trained models come almost entirely from math and coding domains. In agentic domains, however, base models have likely seen far less data in those formats, and they are often assumed to lack the instruction-following and tool-calling skills that agentic tasks require~\citep{zeng2024agenttuning, su2026scaling}. Thus, whether previous coverage analysis transfers to agentic domains remains as a non-trivial gap which motivates our work.

\noindent\textbf{Base models in the harness handle agentic tasks, catch up to post-trained ones given test-time compute.}
Figure~\ref{fig:observation-tts} presents our first observation, which holds consistently across model backbones and benchmarks. Pre-trained models equipped with a simple harness perform agentic reasoning surprisingly well, catching up to their post-trained counterparts as the rollout budget $K$ grows. This is reminiscent of \textit{large language monkeys}~\citep{borel1913mecanique,brown2024large}, but now disciplined under the harness.

By \emph{harness}, we mean a dataset-independent, model-agnostic scaffolding that combines a simple system prompt with relaxed tool-calling and parsing interfaces similar to \cite{yang2024swe} or \cite{wang24executable}. As shown in Table~\ref{table:harness_ablation}, harness dramatically improves both pass@$1$ and pass@$32$ on benchmarks that base models particularly struggle with, while leaving already-manageable tasks largely unchanged. Meanwhile, we observed that the same harness hurts the post-trained models' performance (Table~\ref{table:harness_ablation_full}), consistent with evidence that prompt engineering is not uniformly beneficial for advanced models~\citep{wang2026advanced}. More generally, harness effectiveness depends on how it interacts with post-training~\citep{kim2026interplay}. Unless specified otherwise, we evaluate base models with the harness and post-trained models with the dataset-default scaffolding; see Appendix~\ref{appendix:harness} for details and full results. 
\begin{figure*}[t]
     \centering
     \includegraphics[width=0.975\textwidth]{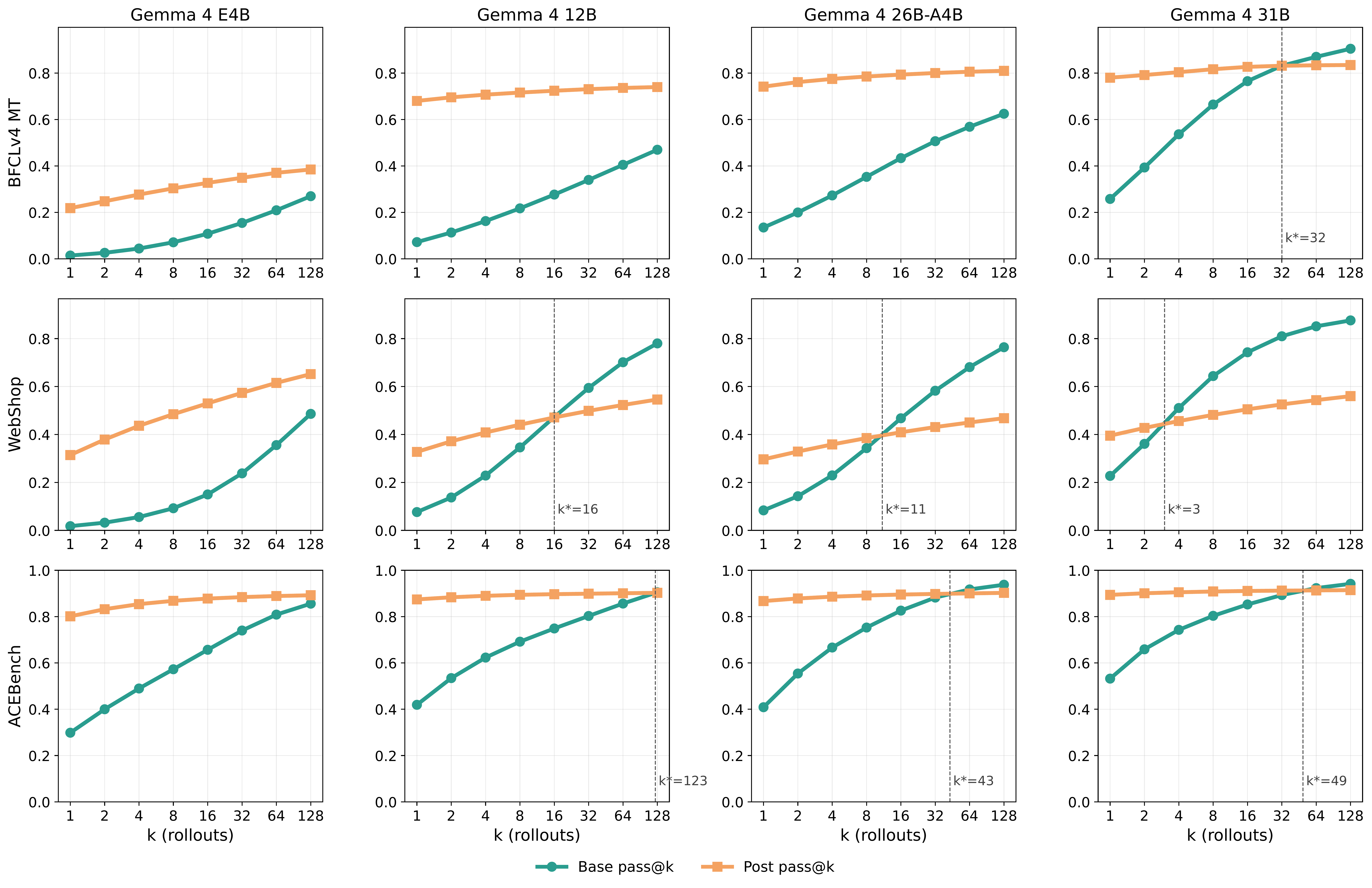}
     \vspace{-0.6em}
     \caption{\textbf{Larger models pull the crossover points earlier.}  Pass@$K$ test-time scaling curves of harness-equipped \texttt{gemma-4} base models (teal) vs. their post-trained counterparts (orange) with model size increasing from left to right. Overall, the crossover budget at which the base model overtakes post-trained model shrinks as model scale grows.}
     \vspace{-0.5em}
     \label{fig:observation-tts-modelsize}
     \vspace{-0.6em}
\end{figure*}

\begin{wraptable}{r}{0.42\textwidth}
\vspace{-2em}
\centering
\caption{\textbf{Harness ablation study of base models.} Average pass@$1$ and pass@32 of eight base models from the Gemma-4, Ministral-3, Qwen2.5, and Qwen3.5 families with (w/) and without (w/o) harness. Table~\ref{table:harness_ablation_full} provides per-model results.}
\label{table:harness_ablation}
\vspace{-0.75em}
{\small\setlength{\tabcolsep}{4pt}
\begin{tabular}{@{}lcccc@{}}
\toprule
& \multicolumn{2}{c}{pass@$1$} & \multicolumn{2}{c}{$\text{pass}@32$} \\
\cmidrule(lr){2-3}\cmidrule(l){4-5}
& w/o & w/ & w/o & w/ \\
\midrule
BFCL & 5.59 & 15.63 & 15.19 & 49.13 \\
WebShop & 6.32 & 9.73 & 39.48 & 49.15 \\
ACEBench & 44.24 & 42.63 & 89.25 & 87.69 \\
\bottomrule
\end{tabular}}
\vspace{-0.5em}
\end{wraptable}
We see a clear trade-off between the two policies. Post-trained models achieve higher single-sample accuracy (pass@$1$), but base models scale much more steeply in solution coverage (pass@$K$). As the budget grows to 128 rollouts per task, the curves of base models cross over and eventually surpass the post-trained ones across most benchmarks, e.g., over 85\% $\text{pass}@128$ on WebShop vs. 56\% for RL with \texttt{gemma-4-31B}. The natural next question is how post-training reshapes the policy to produce these outcomes. A popular explanation is the \emph{\textbf{sharpening hypothesis}}: (RL) post-training amplifies a few high-reward behavior modes that the base model already contains, while suppressing the rest. We dive deeper into this observation.

\noindent\textbf{Is sharpening a curse or a blessing?}
The scaling curves of base and post-trained models cross over, but when does the crossover happen? Figure~\ref{fig:observation-tts-modelsize} plots pass@$K$ curves across four model scales of \texttt{Gemma-4} (4B, 12B, 26B, and 31B). We observe that larger models pull the crossover point $k^*$ to a smaller rollout budget. In WebShop, for example, the crossover budget decreases from $k^* > 128$ for 4B to $k^* \approx 3$ for 31B. Appendix~\ref{appendix:agg_scopes} provides the full analysis results, including the similar conclusion from other model families.

This trend aligns with existing accounts of model capacity, saying larger models can store more behavior modes with less interference~\citep{huang2026larger} and extract more structured information from the same data~\citep{kaplan2020scaling,finzi2026entropy}, which helps them discover solutions within a small sampling budget~\citep{wei2022chain}. Therefore, sharpening cuts both ways. For smaller models, it raises the performance floor (pass@$1$) with little coverage to lose; for larger models, it quickly lowers the performance ceiling (pass@$K$). 
Consequently, whether sharpening is harmful or beneficial should be interpreted along with model scale, budget, and application.
For instance, a user who runs a large model with an abundant budget and only needs one promising solution among many rollouts is better served by the base model than by its post-trained counterpart. This is attractive whenever diversity and creativity matter, e.g., scientific discovery~\citep{agarwal2025autodiscovery,Ghareeb2026,yuksekgonul2026learning}, automated research~\citep{yamada2025ai,lu2026towards}, and long-standing open problems~\citep{lee2025banel} given a scalable verifier~\citep{kwok2026llm}.

\noindent\textbf{How does post-training shrink solution coverage?}
To look beyond test-time scaling curves, we categorize every task by its outcomes over 128 rollouts: \emph{always pass} ($c_i = K$), \emph{pass given compute} ($0 < c_i < K$), and \emph{always fail} ($c_i = 0$). As shown in Figure~\ref{fig:pertasktts-bimodalization} (and Table~\ref{table:bimodal_freq_full}, Figure~\ref{fig:taskpassrate_dist_gemma4}--\ref{fig:taskpassrate_dist_ministral3} in Appendix), post-training drastically reduces the middle category, \emph{pass given compute} (e.g., from 87.6\% to 30.0\% on WebShop with \texttt{gemma-4-31B}), and pushes those tasks to both extremes, i.e., \emph{always pass} grows (from 0.0\% to 26.0\%), but so does \emph{always fail} (from 12.4\% to 44.0\%), which is aligned with the recent findings of \cite{shen2026understanding} stating RL post-training amplifies not only the correct modes of actions but also incorrect ones that base policy already prefers.
\begin{figure}[bht]
    \centering
    \includegraphics[width=0.9\linewidth]{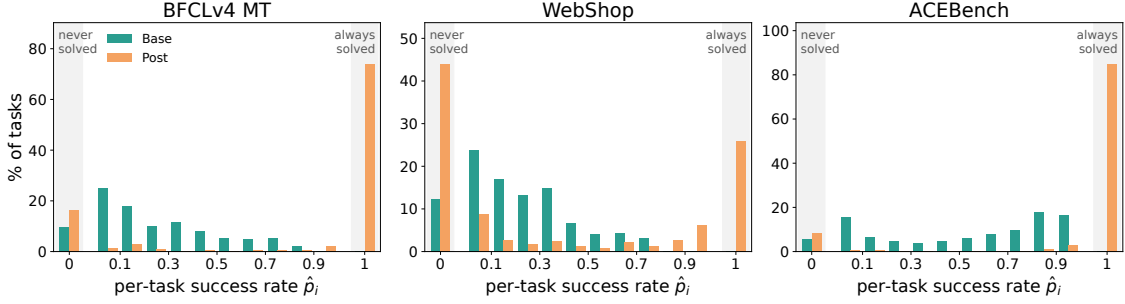}
    \vspace{-0.35em}
    \caption{\textbf{Post-training bimodalizes per-task success rates} (\texttt{gemma-4-31B}, $K{=}128$ rollouts). The base model shows a spread distribution, whereas post-trained model polarizes it into a bimodal shape by pushing most of the intermediate \emph{pass given compute} mass to the two extremes: always solved and never solved. See Appendix~\ref{appendix:perfamily} for full results.}
    \label{fig:pertasktts-bimodalization}
    \vspace{-1em}
\end{figure}

In summary, from a distributional viewpoint over per-task empirical success rates, post-training shows a \emph{bimodalization} effect. Base models spread most of their mass over the region of intermediate success rates, exactly where additional rollouts keep converting failures into successes, whereas the post-trained policies concentrate their mass onto two extrema, leaving little probability in between and hence little room to improve with more rollouts. Theorem~\ref{thm:proportional-tax} in \S\ref{section:theory} formalizes this. 

\begin{wrapfigure}{r}{0.42\textwidth}
\centering
\vspace{-0.75em}
     \includegraphics[width=0.96\linewidth]{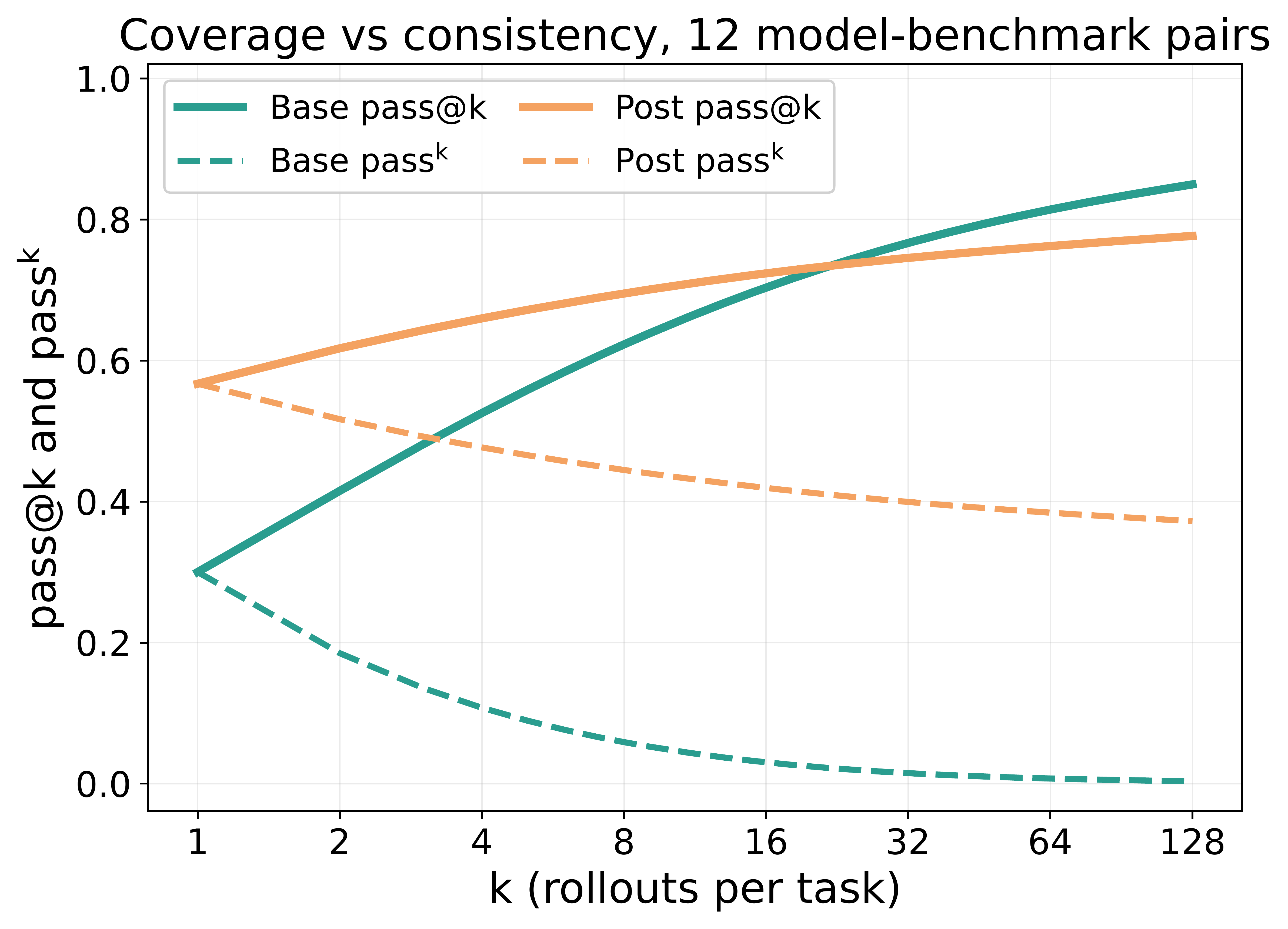}
     \vspace{-1.2em}
     \caption{\textbf{Post-training trades coverage for consistency.} Mean pass@$K$ and $\text{pass}^K$ over 12 model-benchmark pairs.}
     \label{fig:consistency-coverage}
\vspace{-1.2em}
\end{wrapfigure}
Meanwhile, this bimodalization also makes the policy act more consistently, i.e., for a given prompt, the rollouts within a group show high agreement. Figure~\ref{fig:consistency-coverage} confirms this by showing that, averaged over the 12 model-benchmark pairs (largest backbone per family $\times$ three benchmarks), the post-trained policy's coverage (pass@$K$) and consistency ($\text{pass}^K$) curves stay comparatively close, whereas the base model shows a remarkably wide gap, i.e., its consistency collapses to zero while its coverage eventually surpasses RL. In short, post-training buys sampling efficiency and consistency by paying with coverage. These phenomena replicate across all model families (See Appendix Figure~\ref{fig:appendix_ttscurve_gemma4}--~\ref{fig:appendix_ttscurve_ministral3}), with \texttt{Qwen3.5} exhibiting the softest sharpening and \texttt{Gemma-4} exhibiting the steepest sharpening.

\begin{tcolorbox}[mybox]
\noindent \textbf{Takeaways.} (1) With a simple harness, pre-trained base models are capable agentic reasoners and eventually surpass their post-trained counterparts in solution coverage; (2) the crossover budget at which this happens shrinks with model scale, so whether sharpening helps or hurts must be judged jointly with model scale and test-time budget; (3) the underlying mechanism is distribution sharpening, i.e., post-training bimodalizes the per-task success rate distribution to extrema.
\end{tcolorbox}

\section{Quantifying the Effect of Post-Training Sharpening}
\label{section:diagnosis}

So far, we have examined how RL post-training reshapes test-time scaling dynamics by visualizing the scaling curve. However, manually inspecting full pass@$K$ curves for every model and dataset quickly becomes neither practical nor scalable. Meanwhile, evaluating models with only the endpoint metrics such as pass@$K$ and $\text{pass}^K$ at a specific $K$ says little about the full landscape of the scaling behavior, e.g., whether success saturates within a few attempts or keeps rising. To this end, we propose \textbf{Sharpening Tax}, a diagnostic metric that summarizes the effect of post-training on test-time scalability in a single number. For an evaluation budget $K$, we define the following quantities.

\begin{itemize}[leftmargin=*,itemsep=2pt,topsep=2pt,parsep=0pt]
    \item \textbf{Raw area scalability.} The cumulative performance recovered by additional compute (parallel rollouts) up to $K$, defined as the area beneath the pass@$K$ ceiling:
    {\abovedisplayskip=2pt \belowdisplayskip=2pt
    \begin{equation}\label{eq:raw_scalability}
        A(K) = \sum_{k=1}^{K-1} \big[ \text{pass}@K - \text{pass}@k \big]
    \end{equation}
    }
    A larger value of $A(K)$ reflects a greater advantage from performing extra rollouts at test time.

    \item \textbf{Calibrated scalability.} We normalize $A(K)$ with the expected single-shot failure rate $(1 - \text{pass}@1)$, i.e., accuracy headroom, to favor high average success rate while bounding $S(K) \in [0, 1]$:
    {\abovedisplayskip=2pt \belowdisplayskip=2pt
    \begin{equation}\label{eq:norm_scalability}
        S(K) = {\color{black}\dfrac{1}{K-1}}\sum_{k=1}^{K-1} \frac{\text{pass}@K-\text{pass}@k}{1-\text{pass}@1}= \frac{A(K)}{(K-1)(1 - \text{pass}@1)}
    \end{equation}
    }
    \item \textbf{Sharpening Tax.} For a given budget $K$, the deficit between the base and post-trained policies:
    {\abovedisplayskip=1.5pt \belowdisplayskip=1.5pt
    \begin{equation}\label{eq:tax}
        \text{Tax}_X(K) = X_{\text{Base}}(K) - X_{\text{Post}}(K),\quad X \in \{A, S\}
    \end{equation}
    }
\end{itemize}
A positive $\text{Tax}_X(K)$ means post-training reduces test-time scalability relative to the base model, charging a ``tax'' on the coverage gains achievable through compute scaling up to $K$. Next, we show how broadly Sharpening Tax is charged across models and benchmarks (Fig.~\ref{fig:sharpeningtax}), and why it is a useful diagnostic (Fig.~\ref{fig:sharpeningtax-predictability}). \S\ref{section:theory} provides deeper theoretical insights for our tax formulation.

\noindent\textbf{Sharpening Tax is charged broadly across modern post-training pipelines.}
Figure~\ref{fig:sharpeningtax} reports both tax variants' values of the largest and smallest backbone of each model family on the three benchmarks as a function of the rollout budget $k$ (Figure~\ref{fig:sharpeningtax_all14} provides the full results). For the largest backbones, the tax is positive at almost every budget and benchmark, and it is largest for the largest checkpoints such as \texttt{gemma-4-31B}. For the smallest backbones, the tax is often negative at small budgets ($k \le 8$), most visibly for \texttt{Ministral-3-3B}, reflecting that sharpening does help when compute is severely constrained. As the budget scales toward $k = 128$, however, the tax rises in most of these cases, e.g., $\text{Tax}_{S}(128) > 0$ in 36 of the 42 combinations. 
In other words, the coverage trade-off we observed in the previous section is not specific to one model or benchmark; the modern post-training pipelines prioritize immediate sampling efficiency and consistency at the expense of the solution coverage upper bound achievable by test-time scaling. 
The two variants also differ in shape. While both are positive in most cases, the raw tax ($\text{Tax}_{A}$) increases almost monotonically with the budget, whereas the calibrated tax ($\text{Tax}_{S}$) can either rise or fall with the budget, owing to its accuracy-headroom calibration. 
Moreover, a positive tax has a concrete probabilistic interpretation (Proposition~\ref{thm:redeemed-depth} in \S\ref{section:theory} and Corollary~\ref{cor:two-metric} in Appendix~\ref{appendix:two_metric})---after post-training, successes rely less on retries, since failed first attempts are recovered less often (lost coverage) or sooner (faster saturation). Table~\ref{table:two_metric_summary_allfam} in Appendix separates the two effects.
\begin{figure*}[!bht]
     \centering
     \vspace{-0.2em}
     \includegraphics[width=0.92\textwidth]{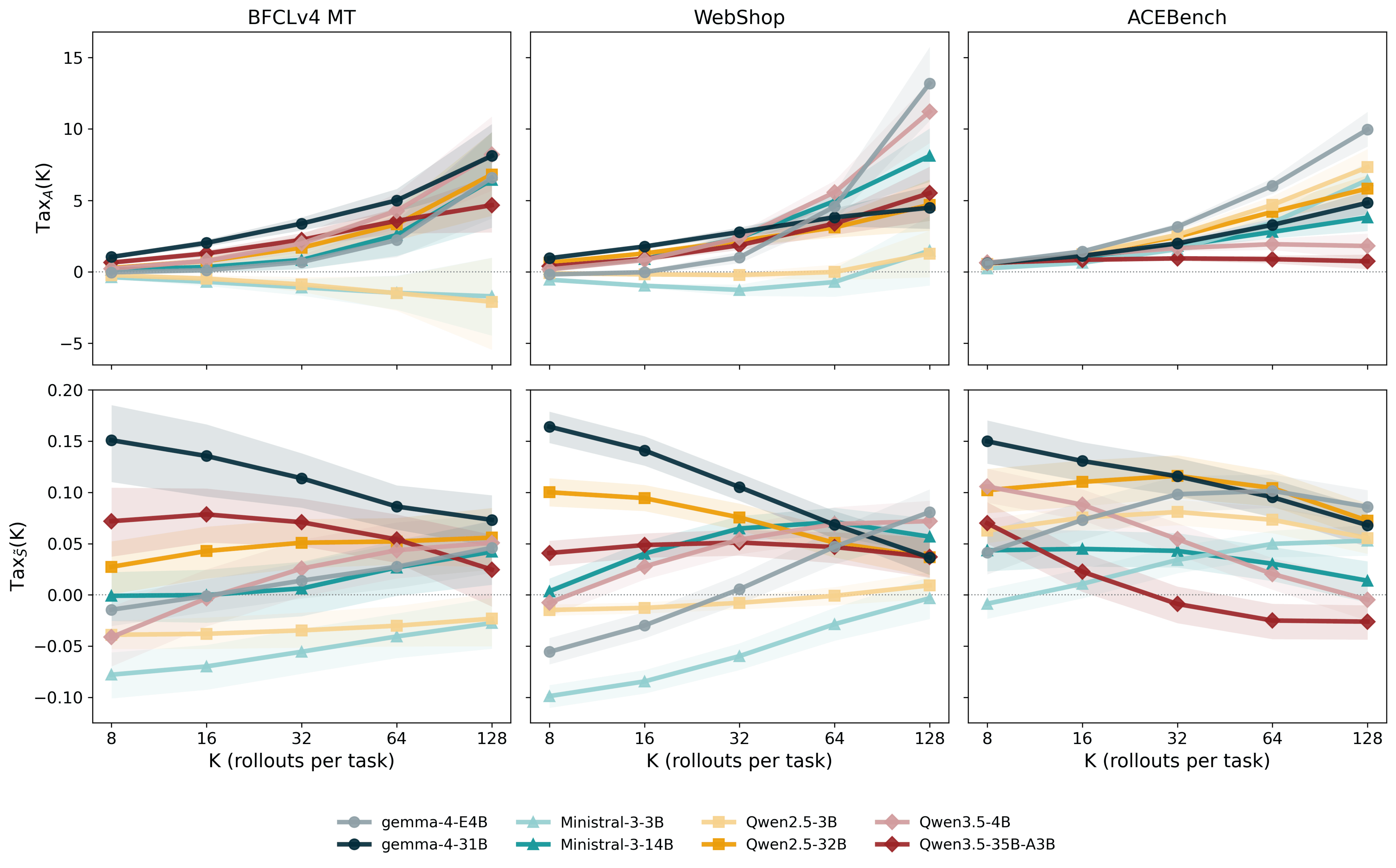}
     \vspace{-0.3em}
     \caption{\textbf{Sharpening Tax across rollout budgets.} $\text{Tax}_{S}(k)$ of eight models (the largest and smallest backbone of each family) across three benchmarks; positive values indicate the base model's scaling advantage over RL. Shades represent 95\% bootstrap confidence intervals. The tax is positive at almost every budget for the largest backbones, and for the smallest backbones it grows from negative values toward zero or above as the test-time budget scales.}
     \label{fig:sharpeningtax}
     \vspace{-0.8em}
\end{figure*}

\noindent\textbf{Sharpening Tax is predictable and strongly correlates with other metrics.}
Beyond describing the observed tax values, we ask two practical questions: can tax be estimated cheaply without full budget evaluation, and what does it tell us about other evaluation metrics? Figure~\ref{fig:sharpeningtax-predictability} provides answers.
\begin{enumerate}[leftmargin=*,itemsep=2pt,topsep=2pt,parsep=0pt]
    \item \textbf{Early predictability}: In Figure~\ref{fig:sharpeningtax-predictability} (Left), $\text{Tax}_{S}(8)$, estimated from 8 rollouts on half of the tasks, strongly predicts $\text{Tax}_{S}(32)$ on the held-out half of the tasks, achieving a remarkably high Spearman correlation $\rho=0.85$. That is, a cheap 8-rollout estimate is a reliable early signal of how the tax will grow, indicating the practicality of the tax estimation in budget-limited setups~\citep{kazdan2025efficient}. By leveraging this predictability, we showcase an interesting application, routing base and post-trained models per task in Figure~\ref{fig:routing}.
    \item \textbf{Correlation with the consistency gap}: As shown in Figure~\ref{fig:sharpeningtax-predictability} (Middle), the tax computed at $K=8$ exhibits a clear positive rank correlation with the consistency gap between the RL and base models, $\Delta \text{pass}^8 = \text{pass}^8_{\text{RL}} - \text{pass}^8_{\text{Base}}$. This supports our earlier finding in Figure~\ref{fig:consistency-coverage}, i.e., the more consistent the post-trained policy becomes, the larger the tax it pays in lost coverage.
    \item \textbf{Comparison with other predictors}: Figure~\ref{fig:sharpeningtax-predictability} (Right) compares early predictors using only 8 rollouts against three future targets from 32 held-out rollouts ($\text{Tax}_{S}(32)$, $\Delta\text{pass}@32$, and $\Delta\text{pass}^{32}$). The early tax $\text{Tax}_{S}(8)$ is the only one that predicts all three metrics with a strong statistical sign, suggesting its promise as a base-model-anchored evaluation metric for the post-trained policy under repeated sampling.
\end{enumerate}

\begin{tcolorbox}[mybox]
\noindent \textbf{Takeaways.} (1) Sharpening Tax is prevalent across numerous models and benchmarks, indicating that modern post-training prioritizes accuracy over coverage; (2) Sharpening Tax is a practical diagnostic of a post-trained policy, which extrapolates to larger budgets and is predictive of other metrics of interest.
\end{tcolorbox}

\begin{figure*}[t]
     \centering
     \vspace{-0.35em}

     \includegraphics[width=0.9\textwidth]{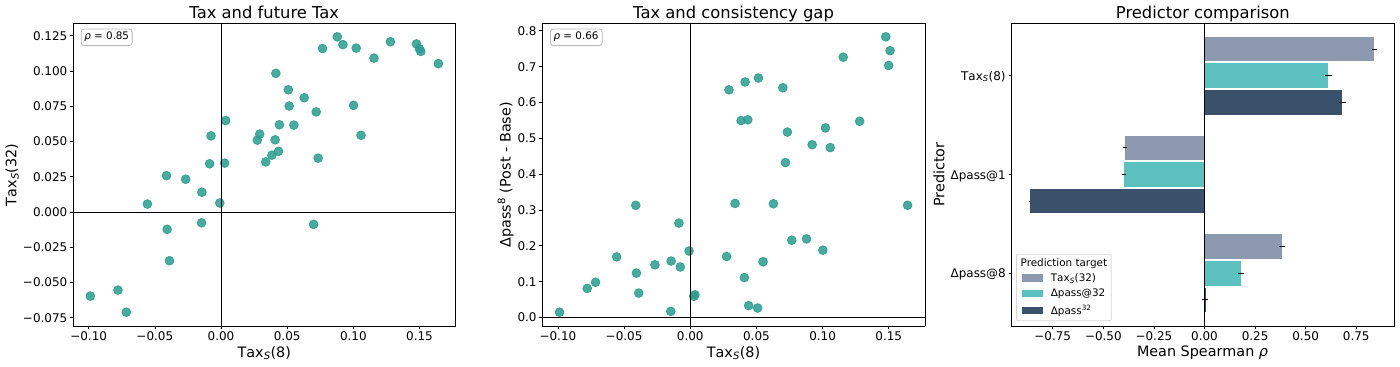}
     \vspace{-0.35em}
     \caption{\textbf{Sharpening Tax extrapolates to larger budgets; predicts the consistency gap and the value of additional test-time compute.} Each panel reports Spearman rank correlation $\rho$ over 42 model-benchmark combinations, estimated and validated over disjoint task subsets. (\textbf{Left}) $\text{Tax}_{S}(8)$ vs. future tax $\text{Tax}_{S}(32)$. (\textbf{Middle}) $\text{Tax}_{S}(8)$ estimated from 8 rollouts vs. consistency gap $\Delta\text{pass}^8$. (\textbf{Right}) Mean Spearman $\rho$ of candidate predictors for other evaluation metrics.}
     \label{fig:sharpeningtax-predictability}
     \vspace{-1.15em}
\end{figure*}

\section{How to Balance Sampling Efficiency and Solution Coverage?}
\label{section:solution}
We so far observed that post-training charges Sharpening Tax, collapsing the policy onto a few rewarding behaviors and compromising the rich coverage that the base model had. Now we ask: \emph{Can we keep the sampling efficiency that post-training buys while securing the coverage it takes away?} The most obvious remedy,  increasing the temperature, improves pass@$K$ but degrades pass@$1$ (Fig.~\ref{fig:temperaturescaling-rl} in Appendix~\ref{appendix:global_temperature}). To tackle this dilemma, we introduce a simple Bayesian sampler pursuing per-task adaptive tempering during RL.

\begin{figure*}[!bh]
     \centering
     \vspace{-0.55em}
     \includegraphics[width=0.975\textwidth]{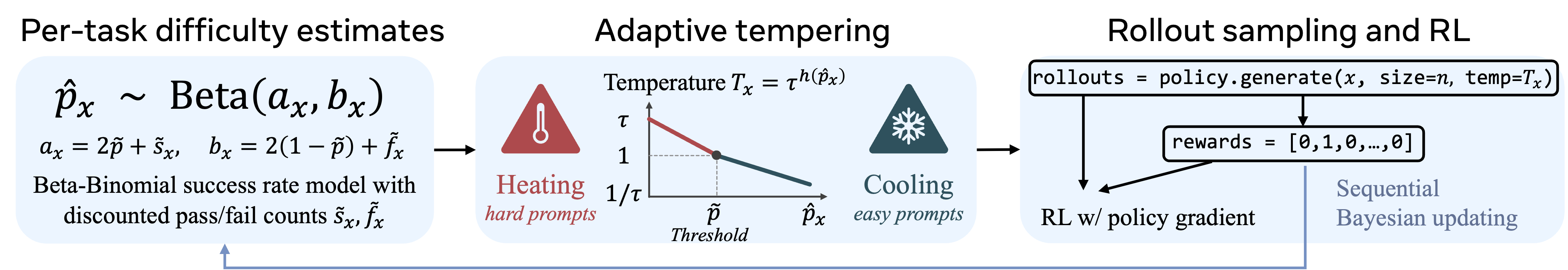}
     \vspace{-0.35em}
     \caption{\textbf{Illustration of posterior-tempered group sampling}. Given a policy, we model its success rate for each task (prompt) $x$ as a dynamic difficulty estimate to adaptively determine rollout sampling temperature in RL training pipeline. It encourages exploratory behavior for hard prompts while inducing conservative behavior for easy ones.}
     \label{fig:ptgs_illustration}
     \vspace{-1em}
\end{figure*}

\noindent
\begin{minipage}[t]{0.53\textwidth}
\noindent\textbf{Posterior-tempered group sampling.}
We propose \textit{posterior-tempered group sampling (PTGS)}, estimating per-task difficulty online from the empirical success rate of the policy and letting this running estimate set the temperature of the next sampling across training progress. 
To be specific, whenever we observe $n$ rollouts per prompt $x$, we track cumulative discounted success and failure counts $(\tilde{s}_x, \tilde{f}_x) \leftarrow (\gamma\tilde{s}_x + s_x,\; \gamma\tilde{f}_x + n - s_x)$ throughout training with a forgetting factor $0\leq \gamma < 1$.
We set them as parameters of a Beta distribution combined with a target success rate $\tilde{p}$, forming the posterior $\mathrm{Beta}(a_x, b_x)$ via $a_x = 2\tilde{p} + \tilde{s}_x$ and $b_x = 2(1-\tilde{p}) + \tilde{f}_x$. Then, PTGS draws $\hat{p}_x \sim \mathrm{Beta}(a_x, b_x)$ by Thompson sampling and decodes the next group of $n$ rollouts when facing the prompt $x$ at the temperature $T_x$ as follows:
\begin{equation}
\begin{aligned}
T_x &= \tau^{\,h(\hat{p}_x)}, \\
h(p) &=
\begin{cases}
\frac{\tilde{p} - p}{\tilde{p}}, & p \le \tilde{p} \quad \text{(hard prompt)},\\[2pt]
\frac{\tilde{p} - p}{1 - \tilde{p}}, & p > \tilde{p}\quad \text{(easy prompt)}.
\end{cases}
\end{aligned}
\label{eq:ptgs}
\end{equation}
\end{minipage}%
\hfill
\begin{minipage}[t]{0.445\textwidth}
\vspace{0pt}\vspace{\dimexpr-\intextsep-0.7\baselineskip\relax}
\begin{algorithm}[H]
\small
\caption{PTGS within an RL training loop}
\label{alg:ptgs}
\begin{algorithmic}[1]
\Require task pool $\mathcal{D}$, group size $n$, target success rate $\tilde{p}$, base temperature $\tau>1$, and forgetting factor $0\leq \gamma < 1$
\State $(\tilde{s}_x, \tilde{f}_x) \gets (0, 0)$ for every $x \in \mathcal{D}$
\For{RL training step $t = 1, 2, \dots$}
    \State Sample a task batch $\mathcal{B}_t \subset \mathcal{D}$
    \For{each prompt $x \in \mathcal{B}_t$}
        \State $(a_x, b_x) \gets \bigl(2\tilde{p} + \tilde{s}_x,\; 2(1 - \tilde{p}) + \tilde{f}_x\bigr)$
        \State Draw $\hat{p}_x \sim \mathrm{Beta}(a_x, b_x)$
        \State Set $T_x \gets \tau^{h(\hat{p}_x)}$ using Eq.~\eqref{eq:ptgs}
        \State \algparbox{Generate $n$ rollouts at temperature $T_x$ and count $s_x$ successes}
        \State $(\tilde{s}_x, \tilde{f}_x) \gets \bigl(\gamma \tilde{s}_x + s_x,\; \gamma \tilde{f}_x + n - s_x\bigr)$
    \EndFor
    \State \algparbox{Apply the chosen RL algorithm's policy update rule}
\EndFor
\end{algorithmic}
\end{algorithm}
\vspace{-\intextsep}
\end{minipage}
\par
where the base temperature $\tau > 1$ defines the spread with a range $T_x \in [1/\tau, \tau]$. With this Bayesian sequential difficulty estimator, a prompt with a lower $\hat{p}_{x}$ is heated up to $\tau$, to encourage exploratory rollouts, and a prompt with a higher $\hat{p}_{x}$ is cooled down to $1/\tau$, for exploitative rollouts. 
Although we drop this adaptive sampler at test time, intervening in the sampling distribution of training experience can improve the resulting policy's behavior after training~\citep{schaul2015prioritized,nath2025adaptive} by reinforcing the model's action sequence towards exploratory and exploitative trajectories depending on the prompt.
Note that \textit{PTGS does not require any changes to RL update itself nor bring computational overhead}, enabling flexible application to any modern RL algorithms. See Algorithm~\ref{alg:ptgs} and Figure~\ref{fig:ptgs_illustration} for description,  and Appendix~\ref{appendix:ptgs_setup} for details.

\noindent\textbf{PTGS gets a tax break: sharper sampling without narrower coverage.}
Since PTGS changes rollout sampling during RL training, we evaluate it by fine-tuning our own agents. We thus study the sharpening from a general post-trained policy to a task-specific RL policy. Following RAGEN~\citep{wang2025ragen,wang2026ragen}, we fine-tune \texttt{Qwen2.5-7B-Instruct} on Sokoban and FrozenLake with two representative RL algorithms, PPO~\citep{schulman2017proximal} and GRPO~\citep{shao2024deepseekmath}, each with and without PTGS. All runs use 200 training steps and signal-to-noise ratio reward-variance filter. We report means over five runs, using the final checkpoint for PPO and the highest validation success checkpoint in each run for GRPO. 
We adopt $\tau \in [1.2,1.5]$ and $\gamma = 0.95$, and gradually raise the target success rate $\tilde{p}$ from $0.25$ to $0.5$ during training. Appendix~\ref{appendix:ptgs_setup} elaborates details.

\begin{wraptable}{r}{0.53\textwidth}
\vspace{-1em}
\centering
\caption{\textbf{PTGS improves accuracy and coverage while reducing tax.} RL training results of Qwen2.5-7B-Instruct, averaged over five random seed runs. $^{\dagger}$Best validation success checkpoint per run; PPO uses the final checkpoint.}
\label{tab:ptgs-combined}
\vspace{-1em}
{\small\setlength{\tabcolsep}{3pt}
\renewcommand{\arraystretch}{1.1}
\begin{tabular}{@{}c|lccc}
\toprule
\multicolumn{1}{@{}c}{} & Method & pass@$1\uparrow$ & pass@$128\uparrow$ & $\text{Tax}_{S}(128)\downarrow$ \\
\midrule
\rowcolor{gray!10}
\cellcolor{white}& Base & 20.7 & 76.6 & - \\
& PPO & 46.5 & 55.0 & 0.094 \\
\rowcolor[HTML]{cbf3f0}
\cellcolor{white} & PPO w/ PTGS & \textbf{61.1} & \textbf{69.7} & \textbf{0.081} \\
& GRPO$^{\dagger}$ & 36.5 & 55.3 & 0.081 \\
\rowcolor[HTML]{cbf3f0}
\cellcolor{white}\multirow{-5}{*}{\rotatebox{90}{Sokoban}}
& GRPO$^{\dagger}$ w/ PTGS & \textbf{39.1} & \textbf{72.5} & \textbf{0.025} \\
\midrule
\rowcolor{gray!10}
\cellcolor{white}& Base & 26.3 & 89.1 & - \\
& PPO & 63.7 & 74.1 & 0.039 \\
\rowcolor[HTML]{cbf3f0}
\cellcolor{white} & PPO w/ PTGS & \textbf{65.0} & \textbf{80.0} & \textbf{0.020} \\
& GRPO$^{\dagger}$ & 63.4 & 77.8 & 0.029 \\
\rowcolor[HTML]{cbf3f0}
\cellcolor{white}\multirow{-5}{*}{\rotatebox{90}{FrozenLake}}
& GRPO$^{\dagger}$ w/ PTGS & \textbf{67.3} & \textbf{81.2} & \textbf{0.028} \\
\bottomrule
\end{tabular}}
\vspace{-0.85em}
\end{wraptable}

Table~\ref{tab:ptgs-combined} shows that task-specific RL still inherits the accuracy-coverage tension. That is, both algorithms improve accuracy over the base model while losing solution coverage. \textbf{PTGS reduces the tax and improves both pass@$1$ and pass@$128$ for two representative RL methods, PPO and GRPO, in both environments.} These joint gains suggest that some of the coverage loss can be avoided by changing how training rollouts are sampled, while retaining the accuracy gains of RL. 

Theorem~\ref{thm:ptgs-signal} in \S\ref{section:theory} offers an explanation through the training signal. If tempering raises a difficult prompt's success probability, PPO receives groups containing a success more often. Besides, rollout groups containing both successes and failures also become more likely, which is important for GRPO to induce nonzero group-relative advantages. Therefore, PTGS can keep difficult prompts contributing to learning under both algorithms, helping preserve coverage as accuracy improves. See Appendix~\ref{appendix:deferred_results} and~\ref{appendix:ptgs_tau_ablation} for further training dynamics and temperature ablations.

\noindent
\begin{minipage}[c]{0.5\textwidth}
Next, Figure~\ref{fig:ptgs-training-fl-reduced} shows the PPO training dynamics on FrozenLake. If we look at the validation success rate, the two methods look similar and reach the same value at the final checkpoint.
However, the average entropy of the output token distribution evolves very differently. Under fixed-temperature PPO, the entropy decreases monotonically, and thus the policy converges to a few action sequences. Under PTGS, in contrast, the entropy stays much higher throughout training and rises repeatedly, since PTGS heats the prompts that the policy keeps failing on. Therefore, the two runs end with the same accuracy but very different policies: PTGS keeps exploratory behavior for challenging prompts, which explains its higher coverage and smaller tax in Table~\ref{tab:ptgs-combined}. 
\end{minipage}\hfill
\begin{minipage}[c]{0.48\textwidth}
    \centering
    \includegraphics[width=\linewidth]{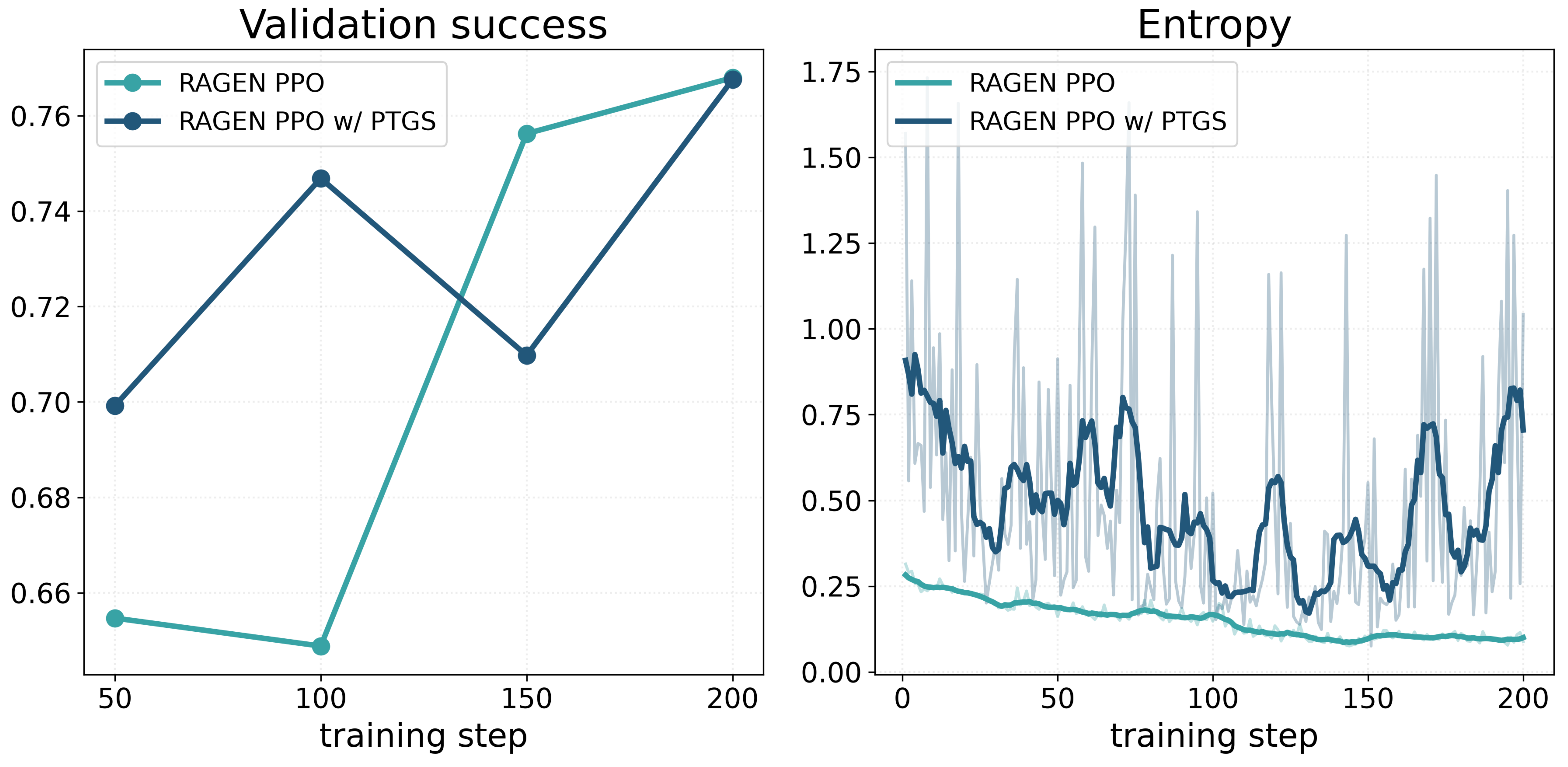}
    \vspace{-1.6em}
    \captionof{figure}{\textbf{Training dynamics on FrozenLake.} Validation success of intermediate checkpoints (left) and average entropy of the output token distribution during training for PPO and PPO with PTGS.}
    \label{fig:ptgs-training-fl-reduced}
\end{minipage}
In short, \textbf{PTGS yields a policy as accurate as fixed-temperature PPO but has a broader space of solvable problems}.

\begin{tcolorbox}[mybox]
\noindent \textbf{Takeaways.} Guided by running estimates of per-prompt difficulty, PTGS exposes the policy to diverging trajectories for hard prompts and converging trajectories for easy ones; it pays a far smaller tax while balancing exploration and exploitation to achieve high pass@1 and pass@$K$ simultaneously.
\end{tcolorbox}

\section{Theoretical Analysis}
\label{section:theory}

In this section, we provide simple theoretical analyses for a better understanding of the Sharpening Tax and PTGS. To be specific, we first explain what Sharpening Tax measures (\S\ref{section:diagnosis}) in Proposition~\ref{thm:redeemed-depth}, then show why the bimodalization of post-training (observed in \S\ref{section:observation}) charges it in Theorem~\ref{thm:proportional-tax}. Finally, we explain how the prompt-adaptive temperature scaling of PTGS improves the rollout groups that RL learns from (\S\ref{section:solution}) in Theorem~\ref{thm:ptgs-signal}. See Appendix~\ref{appendix:theory} for all the deferred proofs.

\noindent\textbf{What the tax measures.}
The raw scalability $A(K)$ of Eq.~\eqref{eq:raw_scalability}, defined as the area under the pass@$K$ ceiling, has a clear probabilistic interpretation---the expected number of failed attempts followed by a success within the budget---connected to classical waiting-time and survival models~\citep{sheps1964time,Berkson1952SurvivalCF}.

\begin{restatable}[Scalability as expected failures before the first success]{proposition}{propredeemed}
\label{thm:redeemed-depth}
Suppose we draw a task randomly from a dataset and let $H$ be the index of its first successful rollout ($H=\infty$ if no rollout succeeds), so that $\Prob(H\le k)=\text{pass}@k$, taken over both the task and its rollouts. Then, given the test-time rollout budget $K$,
\begin{equation}
A(K)=\E\!\left[(H-1)\,\mathbf{1}\{H\le K\}\right],
\qquad
S(K)
=\E\!\left[
\tfrac{H-1}{K-1}\,\mathbf{1}\{H\le K\}
\;\middle|\; H>1
\right]\in[0,1],
\label{eq:redeemed-depth}
\end{equation}
where the second identity holds whenever $\text{pass}@1<1$.
\end{restatable}

Proposition~\ref{thm:redeemed-depth} states that a task solved on the very first try or a task never solved within $K$ attempts has zero contribution to $A(K)$. Only the task first solved on the $h$-th attempt adds $h-1$ to $A(K)$. Thus, $A(K)$ measures \textit{how much of a policy's success is obtained through retries}, with a late first success contributing more than an early one, and $S(K)$ is the same quantity restricted to the tasks that fail on the first try and normalized to $[0,1]$. In summary, a positive tax implies that the post-trained policy's successes rely less on retries, either because its failed attempts are less recovered within the budget (lost coverage) or because they are recovered after fewer retries (faster saturation)\footnote{Corollary~\ref{cor:two-metric} in Appendix~\ref{appendix:two_metric} explicitly characterizes two effects by decomposing the tax.}, thereby reflecting the value of additional rollouts regardless of outcomes, as we intended.

\noindent\textbf{Why sharpening charges a tax.}
We now describe each task $x$ with its per-rollout success probability $p_x \in [0,1]$ under a given policy and assume that repeated rollouts are conditionally independent given $x$. Then $\text{pass}@k = \E_x[1-(1-p_x)^k]$, $\text{pass}^k = \E_x[p_x^k]$, and Eq.~\eqref{eq:raw_scalability} becomes $A(K) = \E_x[a_K(p_x)]$ with the per-task raw scalability, $a_K(p) = \sum_{k=1}^{K-1}[(1-p)^{k}-(1-p)^{K}]$,
which is positive for every $p \in (0,1)$ and vanishes at both extremes, $a_K(0) = a_K(1) = 0$. Under the notation of Proposition~\ref{thm:redeemed-depth}, a task at either extreme has $H=1$ or $H=\infty$, so retries have no value. The bimodalization effect of post-training moves tasks toward exactly these two zeros. We model its limiting case, in which sharpened tasks collapse fully to the extremes, and identify how much test-time scalability it removes.

\begin{restatable}[Collapse to the extremes charges a proportional tax]{theorem}{thmtax}
\label{thm:proportional-tax}
Suppose that post-training sharpens each task with probability $\lambda \in [0,1]$, independently of the base model's success probability $p_x$, replacing $p_x$ by either $0$ or $1$ and leaving it unchanged otherwise, and let $\lambda_0 \le \lambda$ be the proportion of tasks sharpened to $0$. Then, for every budget $K \ge 2$,
\begin{equation}
\text{Tax}_A(K) = \lambda\, A_{\mathrm{Base}}(K) \;\ge\; 0,
\qquad
\frac{S_{\mathrm{Post}}(K)}{S_{\mathrm{Base}}(K)} = \frac{(1-\lambda)\,\epsilon}{(1-\lambda)\,\epsilon + \lambda_0},
\label{eq:proportional-tax}
\end{equation}
where $\epsilon = 1 - \text{pass}@1_{\mathrm{Base}}$, and the second identity holds whenever $A_{\mathrm{Base}}(K) > 0$ and $\text{pass}@1_{\mathrm{Post}} < 1$.
\end{restatable}

The first identity says that sharpening reduces scalability regardless of which extreme a task moves to. A task moved to $p=1$ succeeds on its first attempt, and a task moved to $p=0$ fails on every attempt; in both cases, retries can no longer turn failures into successes. The post-trained policy thus keeps a $(1-\lambda)$ proportion of the base model's scalability, and the raw tax, $\text{Tax}_{A}(K)$, is strictly positive whenever $\lambda > 0$ and $A_{\mathrm{Base}}(K) > 0$. Besides, the second identity says that a higher first-attempt accuracy does not remove the calibrated tax. The ratio is below one, so $\text{Tax}_{S}(K) > 0$ whenever $\lambda_0 > 0$, since post-training keeps only a $(1-\lambda)$ proportion of the base model's scalability but more than a $(1-\lambda)$ proportion of its failures (including tasks sharpened to $p=0$). Meanwhile, post-training improves pass@$1$ when enough of the sharpened tasks move to $p=1$. Therefore, a post-trained policy looks strictly better at $K{=}1$ and still pays the tax, consistent with the higher pass@$1$ (Figure~\ref{fig:observation-tts}) and positive $\text{Tax}_{S}$ (Figure~\ref{fig:sharpeningtax}) we observe for large backbones.

\noindent\textbf{How PTGS improves group-based RL.}
Our PTGS cools prompts that the policy has already mastered, locking in their drift toward $p=1$ for more consistent success, but heats the prompts that the policy keeps failing, raising zero-success prompts up to the intermediate zone, where retries can turn failures into successes. 
We argue that this adaptive heating mechanism enhances the learning signal during modern RL post-training.
Since modern RL pipelines usually sample a group of $n \ge 2$ rollouts per prompt, we analyze the signal coming from the rollout group per prompt. Group-based RL methods such as GRPO and its variants~\citep{shao2024deepseekmath,yu2025dapo,liu2025understanding,chu2026gpg} need multiple rollouts by design, and even PPO, which in principle needs only one rollout per prompt, often uses multiple rollouts in practice for cheaper and lower-variance advantage estimates~\citep{kazemnejad25vineppo,wang2025ragen,wang2026ragen,hu2025open}. How much a group contributes to learning depends on two events: (1) an actor-critic learner such as PPO still updates on an all-failure group via its critic, but can reinforce a success only if the group has one, and (2) a group-contrast learner such as GRPO gets a nonzero advantage only from a \emph{mixed} group, i.e., a group with both a success and a failure. Theorem~\ref{thm:ptgs-signal} shows how PTGS changes the frequency of both events.

\begin{restatable}[PTGS improves rollout groups that RL learns from]{theorem}{thmptgs}
\label{thm:ptgs-signal}
Consider a prompt whose true success probability changes from $p$ to $p'$ under the temperature selected by PTGS, and let each group consist of $n \ge 2$ conditionally independent rollouts. Denote $I_n(q) = 1-(1-q)^n$ as the probability that the group contains at least one success and $J_n(q) = 1-q^n-(1-q)^n$ as the probability that it is a mixed group of successes and failures. If $p' > p$, then
\begin{equation}
\begin{aligned}
\mathrm{(i)} \quad & I_n(p') > I_n(p) && \textnormal{for all } 0 \le p < p' \le 1,\\
\mathrm{(ii)} \quad & J_n(p') > J_n(p) && \textnormal{if } p < p' \le \tfrac12.
\end{aligned}
\label{eq:ptgs-signal}
\end{equation}
\end{restatable}

Parts (i) and (ii) cover the hard prompts that PTGS heats. If heating raises such a prompt's success probability, groups with at least one success become more frequent at every difficulty level, so an actor-critic learner more often has a success to reinforce (i), and mixed groups also become more frequent as long as the prompt remains hard ($p' \le \tfrac12$), so a group-contrast learner receives nonzero advantages more often (ii). Since $J_n$ peaks at $p = \tfrac12$, the theorem also motivates raising the target success rate $\tilde{p}$ toward $\tfrac12$ over training (Appendix~\ref{appendix:ptgs_setup}). For the mastered prompts that PTGS cools, Proposition~\ref{thm:ptgs-mastered} in Appendix~\ref{appendix:ptgs_signal} shows that raising the success probability turns mixed groups into all-success rather than all-failure groups.

Note that Theorem~\ref{thm:ptgs-signal} is assuming that temperature scheduling raises the success probability ($p' > p$); it does not claim that a higher temperature always helps a hard prompt. In Figure~\ref{fig:ptgs-mech-zero} in Appendix~\ref{appendix:ptgs_mech_1}, we show that PPO training with PTGS's adaptive temperature reduces the number of zero-success prompts compared to a fixed-temperature baseline, demonstrating that this assumption empirically holds.

\begin{tcolorbox}[mybox]
\noindent \textbf{Takeaways.} (1) Sharpening Tax counts the expected failed attempts until facing the first success; (2) Collapsing tasks to the two extremes through the post-training sharpening charges a tax proportional to the collapsed fraction; (3) PTGS cuts tax during training by raising a hard prompt's success probability to make groups with at least one success and groups with both success and failure more frequent, thereby inducing better learning signals that PPO and GRPO can extract from.
\end{tcolorbox}

\section{Related Work}
\label{section:literature}

\noindent\textbf{Does RL post-training expand or sharpen the reasoning boundary?}
Whether RL post-training extends the reasoning boundary of the base model is under active debate. One line of work observes that RL-tuned models lose to their base counterparts in pass@$K$ at large $K$, suggesting that post-training sharpens the policy around behaviors the base model already possesses~\citep{yue2025does,zhao2025echo,wu2025invisible}. Another line argues that the boundary can genuinely expand, e.g., through prolonged RL training~\citep{liu2025prorl} or grokking~\citep{sun2026rl}, while others attribute the outcome to properties of the training data and task distribution~\citep{zhang2025interplay,shen2026understanding,shao2026spurious}. However, the evidence on both sides comes almost entirely from mathematics and coding~\citep{yue2025does,he2025rewarding}, where the skills required for realistic agentic tasks are not assessed, and pass@$K$ can be inflated by lucky final answers~\citep{wen2026reinforcement,dragoi2025beyond}. We bring this debate to agentic environments and give it a quantitative diagnosis, Sharpening Tax, which turns manual inspection of test-time scaling visualizations into a measurable and predictable quantity.

\noindent\textbf{Preserving diversity in post-training.}
Alongside the investigation of the sharpening hypothesis of post-trained policy, preventing RL post-training from collapsing policy diversity has been a popular topic of research. Common objective-side solutions include exploration bonuses and mitigation of entropy-collapse~\citep{yu2025dapo,he2025rewarding,song2025outcome,gai2025differential} and pass@$k$-aware learning objectives~\citep{walder2025pass,tajwar2026maximum}. Another line is the sampling-side solutions adjusting the rollout distribution through temperature scheduling across training or within rollouts~\citep{yang2025let,liao2025enhancing,dang2026temperature}, or estimating power distribution~\citep{karan2026reasoning}. While belonging to this sampling-side solution, our PTGS distinguishes itself from others by pursuing difficulty-guided adaptation of temperature for each prompt via a Beta-Binomial posterior over its success rate; it requires no change to the existing RL algorithm, thereby applying to most of the modern RL pipelines on the fly.

\noindent\textbf{Test-time parallel scaling via repeated sampling.}
Repeated sampling has been established as a reliable axis of test-time compute where solution coverage grows smoothly with the number of rollouts~\citep{brown2024large}, following predictable inference-time scaling laws~\citep{wu2024inference,schaeffer2025how}, and can be gathered with verifiers or best-of-$N$ selection~\citep{christiano2017deep,stiennon2020learning,gao2023scaling,luo2024improve,snell2025scaling,kwok2026llm}, where the standard unbiased estimator of pass@$k$~\citep{chen2021evaluating} has been central to these analyses. Complementary approaches scale the reasoning budget within a single trajectory through budget forcing~\citep{muennighoff2025s1}. While this line of work treats repeated sampling as a way to gain performance, we instead repurpose it as a diagnostic tool, using the shape of the scaling curve itself rather than its endpoints to reveal what post-training did to the policy.

\section{Discussion and Conclusion}
\label{section:conclusion}
Large-scale post-training, especially via RL with verifiable rewards, has become a central step in building LLM reasoners and agents. Recently, \citet{lambert2025reinforcement} framed this step through the \textit{elicitation theory} of post-training, ``base models determine the vast majority of the potential of a final model, and post-training's job is to cultivate all of it,'' but he noted that this view may fit lighter recipes better than compute-heavy frontier RL. Agentic tasks, whose multi-turn tool use with complex environment interaction is rare in pre-training data, were a natural place to expect post-training to go beyond elicitation. The results we gain in this work are thus interesting---base models with only a light harness already solve many agentic tasks and, given enough rollouts, often reach more tasks than their post-trained versions. That is, post-training still mainly changes \textit{how reliably} a model solves tasks rather than \textit{which} tasks it can solve in the agentic domain. This cost grows with model scale, but fortunately, it is not unavoidable, as PTGS lowers the tax while improving accuracy.

We close with a simple lesson on the road to \textit{Super Intelligence}~\citep{whitehouse2026superintelligence}. Base models keep getting stronger~\citep{karan2026reasoning}, and more compute is moving to test time for search~\citep{yao2023tree}, repeated sampling~\citep{brown2024large}, and verification~\citep{kwok2026llm}. Both trends raise the price of sharpening, since a stronger base model holds more rare but high-value behaviors, and a larger test-time budget is exactly what turns them into solutions. Many challenges we hope future systems will solve, such as open scientific questions~\citep{gottweis2026accelerating} and Millennium Prize Problems~\citep{openai2026navierstokes}, may be cracked by an unusual shot that works once rather than by a routine that works every time. Post-training should thus be judged not only by its pass@$1$ but also by how much of the base model's potential it keeps, e.g., by reporting Sharpening Tax alongside accuracy. In conclusion, \textbf{reliability and reach should grow together}. Whether heavier RL truly moves beyond elicitation remains open, and base-anchored diagnostics like Sharpening Tax offer a direct way to check it as post-training continues to scale.

\section{Limitations and Future Work}
\label{section:limitation}
Note that the representative open-source models that we extensively analyze in \S\ref{section:observation} do not disclose the exact data used in pre-training and post-training, so our analysis remains observational rather than interventional. Replicating this in a controlled setup~\citep{zhao2025echo,shen2026understanding} where both data mixes are fully known would allow a causal account of when and how sharpening emerges exactly. Second, although the Sharpening Tax does not depend on the exact post-training recipe, we did not compare how different types of post-training sharpen the policy. Investigating how RL, classic supervised fine-tuning and on-policy distillation~\citep{hinton2015distilling,gu2023minillm,agarwal2024policy,lu2025onpolicy,heo2026policy} charge taxes differently is a natural next step. Third, our scope is text-based LLM agents. Whether the same accuracy-consistency-coverage trilemma holds and whether the benefits of PTGS carry over to post-training beyond pure LLMs, including multimodal LLMs and vision-language-action models, remains an open question~\citep{sun2024aligning,kwok2025robomonkey,shen2026rlvr,jeddi2026does,li2026simplevla}. 
Moreover, Sharpening Tax relies on a binary success signal; extending it to open-ended generation, where alignment likewise collapses output diversity~\citep{yuan2026forty}, would require replacing pass@$k$ with quality-gated diversity measures.
Finally, given the diverse applications of base and post-trained policy pairs spanning contrastive decoding~\citep{li2023contrastive,o2023contrastive}, model merging~\citep{wortsman2022robust,rame2024warp,oh2025dawin}, and process reward modeling~\citep{yuan2024free,oh2026neglected}, characterizing how sharpening reshapes these signals and quantities would be valuable beyond test-time scaling diagnosis.

\section*{Acknowledgments}
Changdae Oh thanks Max Khanov, Shawn Im, and Jimmy Di for their intensive proofreading and professional feedback; also thanks Leitian Tao, Samuel Yeh, and Yulin Chen, for their constructive discussion. Changdae Oh and Sharon Li are supported in part by the AFOSR Young Investigator Program under award number FA9550-23-1-0184, the Office of Naval Research (ONR) under award number N000142612508, the National Science Foundation under awards IIS-2237037 and IIS-2331669, the Alfred P. Sloan Fellowship, Open Philanthropy (now Coefficient Giving), and Schmidt Sciences Foundation.

\clearpage
\bibliographystyle{assets/plainnat}
\bibliography{paper}

\clearpage
\beginappendix
\noindent\rule{\linewidth}{0.4pt}
\par{\small\makeatletter
\renewcommand*\l@section[2]{\ifnum \c@tocdepth >\z@ \addpenalty\@secpenalty \addvspace{0.5em \@plus\p@}%
  \setlength\@tempdima{1.5em}\begingroup \parindent \z@ \rightskip \@pnumwidth \parfillskip -\@pnumwidth
  \leavevmode \bfseries \advance\leftskip\@tempdima \hskip -\leftskip #1\nobreak\hfil \nobreak\hb@xt@\@pnumwidth{\hss #2}\par
  \endgroup \fi}%
\@starttoc{toc}\makeatother\par}
\vspace{0.25em}
\noindent\rule{\linewidth}{0.4pt}
\vspace{0.5em}
\addtocontents{toc}{\protect\setcounter{tocdepth}{2}}

\clearpage

\FloatBarrier
\section{Experimental Details}
\label{appendix:setup}

In this section, we describe the settings needed to reproduce our experiments. For the sharpening analysis in \S\ref{section:observation} and \S\ref{section:diagnosis}, we describe the benchmarks (\S\ref{appendix:benchmarks}), the model pairs and serving stack (\S\ref{appendix:models}), the harness for base models (\S\ref{appendix:harness}),  metric estimation and uncertainty quantification (\S\ref{appendix:metrics}), and the pseudocode of the tax computation (\S\ref{appendix:pseudocode}). For PTGS (\S\ref{section:solution}), we describe the training and evaluation setup (\S\ref{appendix:ptgs_setup}) and provide its pseudocode (\S\ref{appendix:pseudocode}).

\subsection{Benchmarks and Environments}
\label{appendix:benchmarks}

Notably, all three benchmarks we consider are process-centric. That is, task success is determined by the environment state that the model reaches through its intermediate actions, not by matching a final answer string. Therefore, unlike some math and STEM benchmarks~\citep{wen2026reinforcement,shao2026spurious,rahman2026can}, a flawed trajectory cannot get credit from a lucky final answer.

\noindent\textbf{BFCL multi-turn.} We use the \texttt{multi\_turn\_base} category of the Berkeley Function-Calling Leaderboard v4~\citep{patil2025bfcl}. A rollout passes if and only if the native BFCL multi-turn checker returns \texttt{valid}. This checker compares the environment state produced by the tool calls of the model with the reference end state, so it also grades intermediate actions. Following BFCL, we allow at most 20 tool-call rounds per user turn. We report this single category rather than the full 800-task multi-turn suite with four categories, because the categories differ substantially in difficulty and the comparison between the base and post-trained models can flip from one category to another. Fixing one category avoids averaging over such different regimes.

\noindent\textbf{WebShop.} We evaluate on the first 500 goal instructions of WebShop~\citep{yao2022webshop} with a 1,000-product index, and each episode is limited to 30 environment actions. The model takes one action per turn, either \texttt{search[<keywords>]} or \texttt{click[<button>]}. Each observation is a text rendering of the page, with UI elements separated by \texttt{[SEP]} special token, followed by the list of available actions. Although WebShop environment produces a dense reward in $[0,1]$ that also credits partial attribute matches, we count an episode as a success only when the final reward is $1.0$, i.e., a fully matching purchase, following the common evaluation protocol~\citep{chen2026scaling}.

\noindent\textbf{ACEBench.} We use the official English/Chinese version of ACEBench~\citep{chen2025acebench} with its per-category system prompts and default verifier. We use 770-task subset that combines single-turn general tool calling and multi-step problem solving (750 tasks) with multi-turn agentic tasks (20 tasks). Due to budget constraints, we exclude the tasks that require an LLM-based user simulator.

BFCL and WebShop are widely used benchmarks of basic agentic capabilities such as multi-turn tool calling and environment interaction. ACEBench tests similar skills but is more recent and less commonly reported in model development, so it is less likely to have been optimized for during post-training of the considered model families. We also chose benchmarks on which both base and post-trained models reach non-trivial success within our budget, since pass@$k$ curves that stay near zero for both policies reveal little about sharpening.

\subsection{Model Checkpoints and Serving}
\label{appendix:models}

\noindent\textbf{Checkpoint pairs.} We use open-source base/post-trained checkpoint pairs (Table~\ref{table:model_pairs}) from four model families, Gemma-4\footnote{\url{https://huggingface.co/collections/google/gemma-4}}, Ministral-3\footnote{\url{https://huggingface.co/collections/mistralai/ministral-3}}, Qwen2.5\footnote{\url{https://huggingface.co/collections/Qwen/qwen25}}, and Qwen3.5\footnote{\url{https://huggingface.co/collections/Qwen/qwen35}}, available in huggingface collections.

\begin{table}[h]
\centering
\caption{\textbf{The 14 base/post-trained checkpoint pairs used in our analysis}, ranging from 3B to 35B parameters. \texttt{gemma-4-26B-A4B} and \texttt{Qwen3.5-35B-A3B} are mixture-of-experts models with about 4B and 3B active parameters, respectively.}
\label{table:model_pairs}
\resizebox{\textwidth}{!}{%
\begin{tabular}{@{}llll@{}}
\toprule
\# & Base checkpoint & Post-trained checkpoint & Family \\
\midrule
1 & \texttt{google/gemma-4-E4B} & \texttt{google/gemma-4-E4B-it} & Gemma-4 \\
2 & \texttt{google/gemma-4-12B} & \texttt{google/gemma-4-12B-it} & Gemma-4 \\
3 & \texttt{google/gemma-4-26B-A4B} & \texttt{google/gemma-4-26B-A4B-it} & Gemma-4 \\
4 & \texttt{google/gemma-4-31B} & \texttt{google/gemma-4-31B-it} & Gemma-4 \\
5 & \texttt{mistralai/Ministral-3-3B-Base-2512} & \texttt{mistralai/Ministral-3-3B-Instruct-2512} & Ministral-3 \\
6 & \texttt{mistralai/Ministral-3-8B-Base-2512} & \texttt{mistralai/Ministral-3-8B-Instruct-2512} & Ministral-3 \\
7 & \texttt{mistralai/Ministral-3-14B-Base-2512} & \texttt{mistralai/Ministral-3-14B-Instruct-2512} & Ministral-3 \\
8 & \texttt{Qwen/Qwen2.5-3B} & \texttt{Qwen/Qwen2.5-3B-Instruct} & Qwen2.5 \\
9 & \texttt{Qwen/Qwen2.5-7B} & \texttt{Qwen/Qwen2.5-7B-Instruct} & Qwen2.5 \\
10 & \texttt{Qwen/Qwen2.5-14B} & \texttt{Qwen/Qwen2.5-14B-Instruct} & Qwen2.5 \\
11 & \texttt{Qwen/Qwen2.5-32B} & \texttt{Qwen/Qwen2.5-32B-Instruct} & Qwen2.5 \\
12 & \texttt{Qwen/Qwen3.5-4B-Base} & \texttt{Qwen/Qwen3.5-4B} & Qwen3.5 \\
13 & \texttt{Qwen/Qwen3.5-9B-Base} & \texttt{Qwen/Qwen3.5-9B} & Qwen3.5 \\
14 & \texttt{Qwen/Qwen3.5-35B-A3B-Base} & \texttt{Qwen/Qwen3.5-35B-A3B} & Qwen3.5 \\
\bottomrule
\end{tabular}}
\vspace{-0.75em}
\end{table}

\noindent\textbf{On the term ``RL post-trained model''.} The post-training recipes of these open checkpoints are not fully disclosed, and the \texttt{-it}/\texttt{-Instruct} models may combine SFT, preference optimization, and RL with verifiable rewards in unknown proportions. However, our analysis does not depend on the exact recipe. SFT and DPO also sharpen the policy distribution toward preferred behaviors~\citep{huang2025self}; even knowledge distillation does the same~\citep{cha2025knowledge}. Here, Sharpening Tax measures the effect of post-training, i.e., the loss of test-time scalability relative to the base checkpoint, rather than its mechanism. Therefore, we use ``RL post-trained model'' as a shorthand for all post-trained models, reflecting that RL has become the central and most compute-intensive stage of modern post-training pipelines~\citep{guo2025deepseek,olmo2025olmo,blakeman2025nemotron,microsoft2026mai}.

\noindent\textbf{Inference configurations.} All models are served with vLLM 0.19.0~\citep{kwon2023efficient}, except \texttt{gemma-4-12B} for vLLM 0.23. Base models receive a raw text prompt without any chat template, whereas post-trained models are queried with the native chat template and tool-call parser of each family (e.g., \texttt{gemma4} for Gemma-4 and \texttt{mistral} for Ministral-3).
Meanwhile, sampling parameters are fixed for each benchmark and shared by base and post-trained models, without per-model tuning (Table~\ref{table:decoding}). We use temperature 0.7 for WebShop and ACEBench and 0.4 for BFCL, since the structured tool calls on BFCL were parsed less reliably at 0.7.
Some post-trained models support an optional thinking mode, while others do not. For a fair comparison across model families, we turn it off for all post-trained models that support it by passing \texttt{chat\_template\_kwargs = \{"enable\_thinking": false\}} in every chat-completion request. Another reason is that our budget counts rollouts, not tokens. With thinking mode, post-trained models would use extra test-time compute in each rollout, which would bias the coverage comparison against the base model. We also confirmed that this choice does not change our conclusion. As shown in Figure~\ref{fig:thinkingmode}, although turning on the thinking mode improves pass@$k$ of post-trained models, the overall trend stays the same, i.e., base models still catch up with the post-trained ones at large $k$, while the post-trained models still pay a positive tax within the considered budget range.

\begin{figure}[h]
    \centering
    \includegraphics[width=\linewidth]{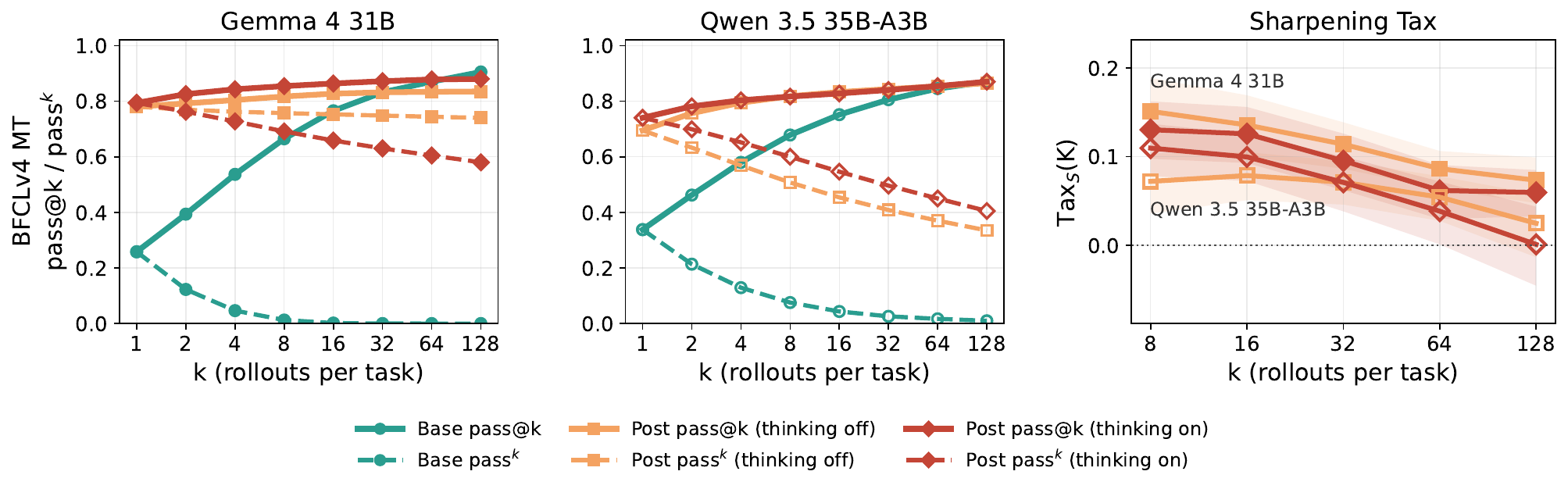}
    \caption{\textbf{Ablation on thinking mode.} Coverage (pass@$k$) and consistency ($\text{pass}^k$) curves of the base and post-trained models of \texttt{gemma-4-31B} (left) and \texttt{Qwen3.5-35B-A3B} (middle) on BFCLv4 MT, with the thinking mode of the post-trained model turned off (our default setup) and on, and the corresponding $\text{Tax}_{S}(k)$ (right), as a function of the rollout budget $k$ per task. Thinking mode improves the accuracy and coverage of the post-trained models, but it still doesn't change the bold conclusion---base models still catch up at large $k$, and the tax remains non-negative.}
    \label{fig:thinkingmode}
\end{figure}

\subsection{Harness Design for Base Models}
\label{appendix:harness}

\begin{tcolorbox}[enhanced, colback=harnessblue, colframe=harnessframe,
  boxrule=0.6pt, arc=1.5mm, left=2mm, right=2mm, top=1mm, bottom=1mm]
{\footnotesize
\begin{verbatim}
You are a helpful assistant that can call tools.

# Available tools
```python
def get_weather(
    city: str, unit: Literal["celsius", "fahrenheit"] = "celsius"
) -> dict:
    """Get the current weather for a city.

    Args:
        city (str, required): City name.
        unit (Literal["celsius", "fahrenheit"], optional):
            Temperature unit.
            choices: ['celsius', 'fahrenheit']
    """
    ...
```

# How to call tools
To call one or more tools, emit a fenced ```tool_call``` block containing a JSON array of calls. 
Each call is an object {"name": <tool_name>, "arguments": {<param>: <value>, ...}} 
whose argument keys match the function parameters above. Format:
```tool_call
[{"name": "<tool_name>", "arguments": {"<param>": "<value>"}}]
```
After the block, stop -- the result will be provided under "### Tool Output". 
You may place several calls in the array to call tools in parallel. 
When the task is complete, reply in plain text with no ```tool_call``` block.

### User
What's the weather in Palo Alto?

### Assistant
```tool_call
[{"name": "get_weather", "arguments": {"city": "Palo Alto"}}]
```

### Tool Output
[get_weather] -> {"temp_c": 21, "sky": "sunny"}

### Assistant
\end{verbatim}
}
\end{tcolorbox}

A pre-trained base checkpoint has no chat template and no tool-call format, but the default benchmark runners expect the OpenAI-style tooling interface with \texttt{messages} and \texttt{tools} arguments. Our harness serves as an adapter between the two. It renders the conversation into a plain-text prompt for the completion endpoint and parses the free-text completion back into structured tool calls. Since our goal is not state-of-the-art performance of base models, we use \textbf{a simple, lightweight universal harness for every base model and every benchmark}, without per-model presets, per-family calibration, or few-shot demonstrations. The prompt explains the tool-call format with placeholder names but never shows a solved task, so the harness does not leak any task information.

\noindent\textbf{Prompt construction.} The prompt consists of five parts in the following order. (1) The system prompt of the benchmark. (2) The tool catalog, where each tool is written as a Python function signature with a short docstring that describes each parameter, its type, and whether it is required, similar to \cite{yang2024swe} and \cite{wang24executable}. (3) An instruction on how to call tools, i.e., by writing a fenced \texttt{tool\_call} block that contains a JSON array of calls. (4) The conversation so far, under the section markers \texttt{\#\#\# User}, \texttt{\#\#\# Assistant}, and \texttt{\#\#\# Tool Output}. (5) A trailing \texttt{\#\#\# Assistant} header, so that the completion of the model starts inside its own turn. We mix the two formats (Python and JSON) on purpose. In preliminary experiments, we found that base models read tool definitions best as Python function signatures but write JSON most reliably. Therefore, we let the model read Python and write JSON, and keep the two consistent by using the parameter names of the function signature as the argument keys of the JSON call. The colored box above shows the exact prompt text that the harness produces for a simple episode with one tool, one user request, and one completed tool-call round with long lines wrapped for display.

\noindent\textbf{Stop sequences and episode termination.} Unlike an instruction-tuned model, a base model has no special token that marks the end of its turn. Therefore, unless equipped with proper guardrails, it would keep writing and eventually invent the environment's response on its own. We therefore stop generation with three stop sequences. The first two, \texttt{"\#\#\# Tool Output"} and \texttt{"\#\#\# User"}, cut the model off the moment it starts writing a section that belongs to the environment or the user. The third, \verb|"\n```"|, stops the model right after it closes its first tool-call block. Here, well-behaved models already stop there on their own, so this stop only guards against a rare failure mode in which a model repeats the same block until the token limit. For the three Qwen3.5 base models on BFCL, which often write a short prose preamble before the block, the third stop is \verb|"\n```\n"| instead, since \verb|"\n```"| would also match the \emph{opening} fence after the preamble and end the turn before any call is written. To determine episode termination, we simply take any reply without a \texttt{tool\_call} block as the model's final answer and end the episode there.

\noindent\textbf{Asymmetric scaffolding.} Throughout the paper, we evaluate base models with the harness and post-trained models with the default scaffolding of each benchmark, so that each checkpoint is used in the way it would be deployed. Table~\ref{table:harness_ablation_full} gives the full ablation. For base models, the harness is essential on BFCL, where six of the eight base checkpoints score zero without it. The two Qwen3.5 base models are the exception, since they already produce parseable native tool calls without the harness. For them, the harness helps the 35B-A3B model but hurts the 4B model on BFCL. On ACEBench, the harness leaves the base models largely unchanged. For post-trained models, the harness lowers performance on average, especially pass@$1$, which is why we keep their default scaffolding.

\noindent\textbf{Harness ablation protocol.} Since a base checkpoint has no chat template, we define the ``w/o harness'' setting of Tables~\ref{table:harness_ablation} and~\ref{table:harness_ablation_full} in the following way. We serve the base checkpoint with the chat template of its post-trained counterpart, enable the native tool-call parser, and send the requests through the same chat-completion endpoint as for the post-trained models. In this ablation, only the prompt and parsing interface change, while the weights, tasks, scoring, seed, rollout budget ($N=32$), and decoding parameters of each benchmark are kept fixed.

\begin{table}[t]
\centering
\caption{\textbf{Harness ablation for base and post-trained models.} pass@$1$ and pass@$32$ (\%) of eight backbones with (w/) and without (w/o) the harness, using $N=32$ rollouts per task and the same tasks, seed, temperature, and step budget in all four settings. $\Delta$p@32 is the pass@$32$ gain from the harness (w/ minus w/o), and the larger value of each w/o--w/ pair is in bold. For base models, ``w/o harness'' serves the checkpoint with the chat template and native tool-call parser of its post-trained counterpart. For post-trained models, ``w/ harness'' sends the plain-text harness prompt through the completion endpoint, without the chat template or the native parser.}
\label{table:harness_ablation_full}
\resizebox{\textwidth}{!}{%
\begin{tabular}{@{}ll cc cc c cc cc c@{}}
\toprule
& & \multicolumn{5}{c}{Base model} & \multicolumn{5}{c}{Post-trained model} \\
\cmidrule(lr){3-7}\cmidrule(l){8-12}
& & \multicolumn{2}{c}{w/o harness} & \multicolumn{2}{c}{w/ harness} & & \multicolumn{2}{c}{w/o harness} & \multicolumn{2}{c}{w/ harness} & \\
\cmidrule(lr){3-4}\cmidrule(lr){5-6}\cmidrule(lr){8-9}\cmidrule(lr){10-11}
Backbone & Benchmark & p@1 & p@32 & p@1 & p@32 & $\Delta$p@32 & p@1 & p@32 & p@1 & p@32 & $\Delta$p@32 \\
\midrule
\multirow{3}{*}{\texttt{gemma-4-E4B}} & BFCL & 0.0 & 0.0 & \textbf{2.2} & \textbf{21.0} & +21.0 & \textbf{23.7} & \textbf{38.0} & 12.0 & 21.0 & -17.0 \\
 & WebShop & \textbf{1.4} & \textbf{16.8} & 0.9 & 10.2 & -6.6 & \textbf{32.4} & \textbf{56.2} & 4.1 & 22.8 & -33.4 \\
 & ACEBench & 23.8 & 77.6 & \textbf{30.8} & \textbf{77.7} & +0.1 & \textbf{81.4} & \textbf{89.1} & 34.0 & 68.1 & -20.9 \\
\midrule
\multirow{3}{*}{\texttt{gemma-4-31B}} & BFCL & 0.0 & 0.0 & \textbf{32.8} & \textbf{86.5} & +86.5 & \textbf{78.0} & \textbf{83.0} & 7.5 & 21.5 & -61.5 \\
 & WebShop & 3.3 & 38.8 & \textbf{18.2} & \textbf{78.4} & +39.6 & \textbf{40.3} & 52.0 & 20.2 & \textbf{54.4} & +2.4 \\
 & ACEBench & 35.8 & \textbf{96.9} & \textbf{54.8} & 92.0 & -4.9 & \textbf{89.7} & \textbf{91.5} & 57.0 & 87.1 & -4.4 \\
\midrule
\multirow{3}{*}{\texttt{Ministral-3-8B}} & BFCL & 0.0 & 0.0 & \textbf{8.3} & \textbf{40.0} & +40.0 & \textbf{31.7} & \textbf{59.0} & 7.8 & 26.0 & -33.0 \\
 & WebShop & 1.5 & 27.6 & \textbf{6.1} & \textbf{51.4} & +23.8 & \textbf{14.7} & \textbf{44.8} & 11.9 & 31.4 & -13.4 \\
 & ACEBench & \textbf{47.0} & \textbf{93.9} & 38.7 & 88.0 & -5.9 & \textbf{74.9} & \textbf{90.0} & 56.8 & 89.1 & -0.9 \\
\midrule
\multirow{3}{*}{\texttt{Ministral-3-14B}} & BFCL & 0.0 & 0.0 & \textbf{13.2} & \textbf{47.0} & +47.0 & \textbf{33.6} & \textbf{61.5} & 17.2 & 40.5 & -21.0 \\
 & WebShop & 4.8 & 45.4 & \textbf{8.4} & \textbf{53.4} & +8.0 & 28.5 & 66.2 & \textbf{30.1} & \textbf{67.0} & +0.8 \\
 & ACEBench & \textbf{56.3} & \textbf{94.0} & 46.1 & 92.5 & -1.5 & \textbf{78.9} & 92.4 & 70.1 & \textbf{92.9} & +0.5 \\
\midrule
\multirow{3}{*}{\texttt{Qwen2.5-3B}} & BFCL & 0.0 & 0.0 & \textbf{0.8} & \textbf{6.3} & +6.3 & \textbf{11.5} & \textbf{27.5} & 2.6 & 8.5 & -19.0 \\
 & WebShop & 0.0 & 1.2 & \textbf{0.1} & \textbf{2.5} & +1.3 & 3.1 & 9.4 & \textbf{5.9} & \textbf{21.8} & +12.4 \\
 & ACEBench & 15.7 & 70.4 & \textbf{21.3} & \textbf{70.7} & +0.3 & \textbf{48.4} & \textbf{67.5} & 37.1 & 65.6 & -1.9 \\
\midrule
\multirow{3}{*}{\texttt{Qwen2.5-32B}} & BFCL & 0.0 & 0.0 & \textbf{19.1} & \textbf{59.6} & +59.6 & \textbf{35.5} & \textbf{59.0} & 13.4 & 22.0 & -37.0 \\
 & WebShop & 8.9 & 61.4 & \textbf{20.6} & \textbf{74.7} & +13.3 & 33.8 & 54.2 & \textbf{34.6} & \textbf{57.8} & +3.6 \\
 & ACEBench & 47.0 & \textbf{95.2} & \textbf{51.3} & 90.4 & -4.8 & 81.8 & 87.6 & \textbf{83.2} & \textbf{93.7} & +6.1 \\
\midrule
\multirow{3}{*}{\texttt{Qwen3.5-4B}} & BFCL & \textbf{25.9} & \textbf{60.5} & 14.9 & 52.0 & -8.5 & \textbf{58.9} & \textbf{83.5} & 21.7 & 69.0 & -14.5 \\
 & WebShop & \textbf{7.1} & 50.6 & 6.3 & \textbf{51.0} & +0.4 & \textbf{32.1} & 61.6 & 21.7 & \textbf{64.0} & +2.4 \\
 & ACEBench & \textbf{58.3} & 88.5 & 45.4 & \textbf{93.5} & +5.0 & \textbf{71.6} & 86.9 & 45.3 & \textbf{90.5} & +3.6 \\
\midrule
\multirow{3}{*}{\texttt{Qwen3.5-35B-A3B}} & BFCL & 18.9 & 61.0 & \textbf{33.8} & \textbf{80.5} & +19.5 & \textbf{69.8} & 84.0 & 48.5 & \textbf{88.0} & +4.0 \\
 & WebShop & \textbf{23.5} & \textbf{74.0} & 17.2 & 71.7 & -2.3 & \textbf{32.9} & \textbf{69.6} & 17.6 & 66.2 & -3.4 \\
 & ACEBench & \textbf{69.9} & \textbf{97.5} & 52.6 & 96.6 & -0.8 & \textbf{83.1} & \textbf{94.4} & 42.8 & 87.7 & -6.7 \\
\midrule\midrule
\multirow{3}{*}{Average (8)} & BFCL & 5.6 & 15.2 & \textbf{15.6} & \textbf{49.1} & +33.9 & \textbf{42.8} & \textbf{61.9} & 16.3 & 37.1 & -24.9 \\
 & WebShop & 6.3 & 39.5 & \textbf{9.7} & \textbf{49.2} & +9.7 & \textbf{27.2} & \textbf{51.8} & 18.3 & 48.2 & -3.6 \\
 & ACEBench & \textbf{44.2} & \textbf{89.2} & 42.6 & 87.7 & -1.6 & \textbf{76.2} & \textbf{87.4} & 53.3 & 84.4 & -3.1 \\
\bottomrule
\end{tabular}}
\end{table}

\begin{table}[h]
\centering
\caption{\textbf{Decoding configuration for each benchmark}, shared by the base and post-trained models of all 14 checkpoint pairs. Every model-benchmark combination is evaluated with $N=128$ rollouts per task.}
\label{table:decoding}
{\small
\begin{tabular}{@{}lccc@{}}
\toprule
 & BFCLv4 multi-turn & WebShop & ACEBench \\
\midrule
temperature & 0.4 & 0.7 & 0.7 \\
top\_$p$ & 0.95 & 0.95 & 0.95 \\
max tokens per generation & 1024 & 256 & 1200 \\
context length & 16384 / 32768 & 16384 & 16384 / 32768 \\
rollouts per task ($N$) & 128 & 128 & 128 \\
reported $K$ grid of pass@$K$ & \multicolumn{3}{c}{$1, 2, 4, 8, 16, 32, 64, 128$} \\
\bottomrule
\end{tabular}}
\end{table}

\subsection{Metrics and Uncertainty Estimation}
\label{appendix:metrics}

\noindent\textbf{Estimation.} We compute all metrics of \S\ref{section:prelim} for each task from its success count $c_i$ out of $n_i$ rollouts with the unbiased estimators of \cite{chen2021evaluating}, and then average them over tasks with equal weight. For numerical stability, we implement the estimators with cumulative products instead of binomial coefficients. Each model uses all of its own $N$ rollouts.

\noindent\textbf{Bootstrap confidence intervals.} The 95\% bands in Figure~\ref{fig:sharpeningtax} are percentile intervals from a paired task bootstrap with 1,000 resamples. In each resample, we draw tasks with replacement and use the same draw for both models, so that the interval of the paired difference accounts for the correlation of task difficulty between the two models. We resample tasks rather than rollouts because we want to know whether the tax would remain under a different set of evaluation tasks. Note that the estimate of pass@$K$ at $K=N$ reduces to a single per-task indicator of whether any rollout succeeded, without averaging over subsets, so the last point of each scaling curve is a noisier estimate.

\subsection{PTGS Training and Evaluation Setup}
\label{appendix:ptgs_setup}
\label{appendix:ptgs_setup_setup}

\noindent\textbf{Model and environments.} All PTGS experiments fine-tune \texttt{Qwen2.5-7B-Instruct} in the RAGEN framework~\citep{wang2025ragen,wang2026ragen} on two interactive environments. FrozenLake uses the slippery variant (each move goes in the intended direction with probability 0.8 and otherwise slips to one of the two perpendicular directions with probability 0.1 each) with a binary success reward. Sokoban uses a $6\times6$ grid with one box with a shaped reward, and an episode counts as a success once the box is on its target. Each episode allows at most 5 turns with up to 2 actions per turn.

\noindent\textbf{Training setup.} We train PPO and GRPO, each with and without PTGS, for 200 StarPO steps~\citep{wang2025ragen}. Each step samples 8 tasks from a fixed pool of 200 tasks and generates 16 rollouts per task, i.e., 128 trajectories per step. All methods use AdamW~\citep{loshchilov2017decoupled} with a constant actor learning rate of $10^{-6}$, DAPO-style~\citep{yu2025dapo} asymmetric clip ratios of $0.2$ (lower) and $0.28$ (upper), token-mean loss aggregation, no KL penalty, and a minibatch size of 32 with one epoch per step. The StarPO-S version 2 reward-variance filter~\citep{wang2026ragen} keeps the top 90\% of prompt groups ranked by reward variance. Both baselines sample rollouts ancestrally at a fixed temperature $T_{\mathrm{ref}} = 1.0$, and PTGS replaces this fixed temperature with a per-prompt temperature.

PPO uses generalized advantage estimation (GAE;~\cite{schulman2015high}) with both the discount factor and the GAE parameter set to $1$, a critic learning rate of $10^{-5}$, and an entropy coefficient of $0.001$. GRPO~\citep{shao2024deepseekmath} uses group-relative advantages without a critic and an entropy coefficient of $0$. Within each algorithm, the baseline and PTGS use the same recipe except for the rollout temperature and the matching temperature used to compute the policy log-probabilities.

\noindent\textbf{PTGS configuration.} For PPO, we use the same PTGS setting in both environments, with $\tau = 1.5$ (so $T_x \in [0.667, 1.5]$) and a target success rate $\tilde{p}$ that grows from $0.25$ to $0.5$ over the 200 steps by a constant factor per step. The final target success rate of $0.5$ is the success probability at which a rollout group is most likely to contain both successes and failures, i.e., where the within-group reward variance $p(1-p)$ peaks (Theorem~\ref{thm:ptgs-signal}). Starting from a lower target success rate heats fewer prompts early in training, when most prompts still have low success rates. Appendix~\ref{appendix:ptgs_tau_ablation} ablates $\tau \in \{1.2, 1.3, 1.4, 1.5\}$ for PPO. For GRPO, we select the PTGS parameter based on peak validation success rate for each environment, which gives $\tau = 1.4$ with a fixed target success rate of $0.5$ on Sokoban and $\tau = 1.2$ with a target success rate growing from $0.25$ to $0.5$ on FrozenLake. In all PTGS runs, the forgetting factor is $\gamma = 0.95$, the Beta prior has total mass 2 and is centered on the current target success rate, and $\hat{p}_x$ is drawn from the posterior of each task by Thompson sampling. A rollout counts as a success when its trajectory return exceeds $0.5$. We draw one temperature per prompt group and use it for every turn of every rollout in that group.

\noindent\textbf{Implementation details.} We denote $T_{\mathrm{ref}}$ as a reference temperature for an environment. It is $1.0$ by default, so we omit it in \S\ref{section:solution}. When the environment decoding temperature departs from $1.0$, one can set $T_{\mathrm{ref}}$ to that environment-default temperature (e.g., \S\ref{appendix:ptgs_inference}). The tempering rule in Eq.~\eqref{eq:ptgs} passes through three points, $T(0) = \tau T_{\mathrm{ref}}$, $T(\tilde{p}) = T_{\mathrm{ref}}$, and $T(1) = T_{\mathrm{ref}}/\tau$, and it reduces to the symmetric rule $T = T_{\mathrm{ref}}\,\tau^{\,1 - 2\hat{p}_x}$ when $\tilde{p} = 1/2$. 

\noindent\textbf{Checkpoint selection.} For PPO, Table~\ref{tab:ptgs-combined} reports the final checkpoint at step 200. For GRPO, whose validation success peaks before step 200 and then sometimes collapses (which is common in multi-turn RL setups~\cite{li2026turn,wang2025ragen,wang2026ragen}), both the baseline and PTGS use the checkpoint with the highest validation success during training (marked with $\dagger$ in the table). 

\noindent\textbf{Offline evaluation and repeated runs.} In all the evaluations in the main paper, we evaluate each checkpoint on 64 unseen tasks with 128 ancestral rollouts per task at $T = 0.5$, without PTGS at test time (except in \S\ref{appendix:ptgs_inference} where we introduce the inference-time PTGS variant). We compute pass@$k$ for $k \le 128$ with the estimators of \S\ref{section:prelim}, and $\text{Tax}_{S}(128)$ for all RL-tuned models against a single shared evaluation of the base model in each environment. Table~\ref{tab:ptgs-combined} reports the mean over five runs of each RL setting.

\subsection{Pseudocode}
\label{appendix:pseudocode}
\begin{lstlisting}[style=pseudopy, float=h, label={lst:tax},
  caption={Computing the Sharpening Tax metrics of \S\ref{section:diagnosis}.}]
def pass_at_k(c, n, k):
    return 1 - comb(n - c, k) / comb(n, k)

def scalability(counts, K):  # counts: [(c_i, n_i) for each task in D]
    curve = [mean([pass_at_k(c, n, k) for (c, n) in counts]) for k in range(1, K + 1)]
    A = sum(curve[-1] - curve[k] for k in range(K - 1))   
    S = A / ((K - 1) * (1 - curve[0]))                    
    return A, S

def sharpening_tax(counts_base, counts_rl, K):
    A_b, S_b = scalability(counts_base, K)
    A_r, S_r = scalability(counts_rl, K)
    return A_b - A_r, S_b - S_r    
\end{lstlisting}

Listing~\ref{lst:tax} shows how we compute the Sharpening Tax of \S\ref{section:diagnosis}. Given the per-task success counts of a policy, we first compute the dataset-level pass@$k$ curve for $k = 1, \dots, K$. We then sum the gaps between pass@$K$ and pass@$k$ to obtain $A(K)$ (Eq.~\eqref{eq:raw_scalability}) and normalize it by the budget and the single-attempt failure rate with total computation budget to obtain $S(K)$ (Eq.~\eqref{eq:norm_scalability}). Finally, the differences between the base and post-trained policies on the same task set give $\text{Tax}_A(K)$ and $\text{Tax}_{S}(K)$ (Eq.~\eqref{eq:tax}). For clarity, the listing omits the cumulative-product implementation of the estimator (\S\ref{appendix:metrics}) and the paired task bootstrap for confidence intervals.

\begin{lstlisting}[style=pseudopy, float=h, label={lst:ptgs},
  caption={PTGS within an RL training loop (Algorithm~\ref{alg:ptgs}).}]
s, f = zeros(len(D)), zeros(len(D))  

for step in range(num_rl_steps):
    batch = sample_tasks(D)
    for x in batch:
        a, b = 2 * p_target + s[x], 2 * (1 - p_target) + f[x]   
        p = Beta(a, b).sample()                           
        h = (p_target - p) / (p_target if p <= p_target else 1 - p_target)
        T[x] = tau ** h                     

        group = policy.generate(x, n=group_size, temperature=T[x])
        s[x] = gamma * s[x] + num_success(group)
        f[x] = gamma * f[x] + group_size - num_success(group)

    rl_update(batch, logprobs=policy.log_probs(batch, temperature=T))
\end{lstlisting}

Listing~\ref{lst:ptgs} gives the pseudocode of PTGS (\S\ref{section:solution}). For each prompt $x$, we keep discounted success and failure counts $(\tilde{s}_x, \tilde{f}_x)$. With a prior of total mass 2 centered on the target success rate $\tilde{p}$, these counts define a Beta posterior over the current success rate of the prompt. Before each rollout group, we draw $\hat{p}_x$ from this posterior by Thompson sampling, map it to a temperature $T_x \in [1/\tau, \tau]$ with Eq.~\eqref{eq:ptgs}, generate the group at $T_x$, and add the observed successes and failures to the counts with the forgetting factor $\gamma$. The log-probabilities in the policy-gradient ratio are computed at the same temperature, so the update remains on-policy. Algorithm~\ref{alg:ptgs} in the main text summarizes the same procedure.

\FloatBarrier
\section{Missing Proofs}
\label{appendix:theory}

This section restates each result of \S\ref{section:theory} and gives its proof, together with a decomposition of the tax (Corollary~\ref{cor:two-metric}) that we use in Appendix~\ref{appendix:tax_allpairs}. Throughout, we abbreviate the subscripts $\mathrm{Base}$ and $\mathrm{Post}$ of the main text as $\mathrm{B}$ and $\mathrm{P}$, respectively.

\subsection{What the Sharpening Tax Measures}
\label{appendix:redeemed_depth}

\propredeemed*

\begin{proof}[Proof of Proposition~\ref{thm:redeemed-depth}]
By definition of $H$, the event $\{H \le k\}$ is that at least one of the first $k$ rollouts succeeds, so we know $\Prob(H \le k) = \text{pass}@k$ for every $k \ge 1$. Hence, for each $1 \le k \le K-1$,
\[
\text{pass}@K-\text{pass}@k=\Prob(k < H \le K)=\sum_{h=k+1}^{K}\Prob(H=h).
\]
By summing over $k$ and exchanging the two finite sums, we have
\begin{align*}
A(K)
&=\sum_{k=1}^{K-1}\sum_{h=k+1}^{K}\Prob(H=h)
=\sum_{h=2}^{K}\Prob(H=h)\sum_{k=1}^{h-1}1 \\
&=\sum_{h=2}^{K}(h-1)\,\Prob(H=h)
=\E\!\left[(H-1)\,\mathbf{1}\{H\le K\}\right],
\end{align*}
where the $h=1$ term vanishes; this proves the first identity of Eq.~\eqref{eq:redeemed-depth}.

For the second identity, note that $1-\text{pass}@1 = \Prob(H>1) > 0$ by assumption, and that the random variable $(H-1)\,\mathbf{1}\{H\le K\}$ vanishes on the event $\{H=1\}$. Therefore,
\begin{align*}
S(K)
&=\frac{\E\!\left[(H-1)\,\mathbf{1}\{H\le K\}\right]}{(K-1)\,\Prob(H>1)}
=\frac{\E\!\left[(H-1)\,\mathbf{1}\{H\le K\}\,\mathbf{1}\{H>1\}\right]}{(K-1)\,\Prob(H>1)} \\
&=\E\!\left[\tfrac{H-1}{K-1}\,\mathbf{1}\{H\le K\}\;\middle|\;H>1\right].
\end{align*}
Finally, $S(K) \in [0,1]$ because $\tfrac{H-1}{K-1}\,\mathbf{1}\{H \le K\}$ takes values in $[0,1]$. It equals $(H-1)/(K-1) \in (0,1]$ on $\{1 < H \le K\}$ and $0$ otherwise.
\end{proof}

\noindent\textbf{Finite-sample estimates.} The proof uses only $\Prob(H \le k) = \text{pass}@k$, so it also applies to the estimates we report. For a task with $n$ observed rollouts of which $c$ succeed, let $H$ be the index of the first success in a uniformly random ordering of these rollouts. The probability that none of the first $k \le n$ rollouts succeeds is $\binom{n-c}{k}\big/\binom{n}{k}$, so $\Prob(H \le k)$ equals the unbiased pass@$k$ estimate of \S\ref{section:prelim}, and averaging over tasks gives the dataset-level estimate. Proposition~\ref{thm:redeemed-depth} therefore holds exactly for the reported $A(K)$ and $S(K)$ whenever $K \le n$.

\subsection{Ceiling and Saturation Decomposition of the Tax}
\label{appendix:two_metric}

The raw scalability $A(K)$ and the tax $\text{Tax}_{A}(K)$ can also be decomposed by using the coverage ceiling and the coverage accumulated below it.

\begin{corollary}[Ceiling and saturation decomposition]
\label{cor:two-metric}
Denote $C(K) = \text{pass}@K$ as the coverage ceiling and $\overline{C}(K) = \frac{1}{K}\sum_{k=1}^{K}\text{pass}@k$ for the budget-averaged coverage of a policy. Then, for each policy and every $K \ge 2$ test-time rollout budget,
\begin{equation}
A(K) = K\,\big[C(K) - \overline{C}(K)\big],
\qquad
\text{Tax}_A(K) = K\,\big[C_{\mathrm{B}}(K) - C_{\mathrm{P}}(K)\big] \;-\; K\,\big[\overline{C}_{\mathrm{B}}(K) - \overline{C}_{\mathrm{P}}(K)\big].
\label{eq:two-metric}
\end{equation}
\end{corollary}

In other words, the raw scalability of a policy is $K$ times the gap between its ceiling and its budget-averaged coverage, which we call the \emph{saturation gap}. The tax is then $K$ times the difference in the saturation gaps of the two policies, which splits into a ceiling gap and a gap in budget-averaged coverage as in Eq.~\eqref{eq:two-metric}.

\begin{proof}[Proof of Corollary~\ref{cor:two-metric}]
For the first identity,
\begin{align*}
A(K) &= \sum_{k=1}^{K-1}\big[C(K) - \text{pass}@k\big]
= (K-1)\,C(K) - \Big(\sum_{k=1}^{K}\text{pass}@k - C(K)\Big) \\
&= K\, C(K) - \sum_{k=1}^{K}\text{pass}@k
= K\,\big[C(K) - \overline{C}(K)\big].
\end{align*}
The second follows by applying the first identity to each policy in $\text{Tax}_A(K) = A_{\mathrm{B}}(K) - A_{\mathrm{P}}(K)$.
\end{proof}

\subsection{Why Sharpening Charges a Tax}
\label{appendix:proportional_tax}

\thmtax*

\begin{proof}[Proof of Theorem~\ref{thm:proportional-tax}]
Each task $x$ has a per-rollout success probability $p_x$ (population-level exact probability), and rollouts are conditionally independent given the task, so $\text{pass}@k$ for task $x$ equals $1-(1-p_x)^k$ and dataset-level metrics are expectations over $x$. Define the per-task raw scalability
\begin{align*}
a_K(p)
&=\sum_{k=1}^{K-1}\Big[\big(1-(1-p)^K\big)-\big(1-(1-p)^k\big)\Big]
=\sum_{k=1}^{K-1}\Big[(1-p)^{k}-(1-p)^{K}\Big], \\
A(K)&=\E_x\big[a_K(p_x)\big].
\end{align*}
Each term $(1-p)^k-(1-p)^K$ vanishes at $p=0$ and $p=1$, so $a_K(0)=a_K(1)=0$. A task that always fails or always succeeds gains nothing from retries. Under the sharpening model, an independently chosen $\lambda$-fraction of tasks has $p_x$ replaced by a value in $\{0,1\}$ while the remaining $(1-\lambda)$-fraction is unchanged. Taking expectations,
\begin{align*}
A_{\mathrm{P}}(K)&=(1-\lambda)\,\E_x\big[a_K(p_x)\big]+\lambda\cdot 0=(1-\lambda)A_{\mathrm{B}}(K), \\
\text{Tax}_A(K)&=A_{\mathrm{B}}(K)-A_{\mathrm{P}}(K)=\lambda A_{\mathrm{B}}(K),
\end{align*}
which is the first identity of Eq.~\eqref{eq:proportional-tax}. Since pass@$k$ is nondecreasing in $k$, we have $A_{\mathrm{B}}(K)\ge0$; hence $\text{Tax}_A(K)=\lambda A_{\mathrm{B}}(K)\ge0$, with strict inequality whenever $\lambda>0$ and $A_{\mathrm{B}}(K)>0$, regardless of which extreme the sharpened tasks reach. If the sharpened tasks are instead chosen depending on $p_x$, the same computation gives $\text{Tax}_A(K)=\E_x\big[\mathbf{1}\{x\text{ sharpened}\}\,a_K(p_x)\big]=\lambda\,\E\big[a_K(p_x)\mid x\text{ sharpened}\big]\ge0$, since each sharpened task's contribution drops from $a_K(p_x)\ge0$ to $0$; independence is needed only for the proportionality.

For the second identity, write $\epsilon_{\mathrm{B}}=1-\text{pass}@1_{\mathrm{B}}$ ($=\epsilon$ in the statement) and $\epsilon_{\mathrm{P}}=1-\text{pass}@1_{\mathrm{P}}$ for the single-attempt failure rates, and recall that $\lambda_0\le\lambda$ is the proportion of tasks sharpened to $p=0$. Unsharpened tasks keep their failure probability and tasks sharpened to $p=0$ always fail, so
\[
\epsilon_{\mathrm{P}}=(1-\lambda)\,\epsilon_{\mathrm{B}}+\lambda_0 .
\]
Assume $A_{\mathrm{B}}(K)>0$, which implies $\epsilon_{\mathrm{B}}>0$, and $\epsilon_{\mathrm{P}}>0$. Since $S(K)=A(K)/\big((K-1)\,\epsilon\big)$ and $A_{\mathrm{P}}(K)=(1-\lambda)A_{\mathrm{B}}(K)$,
\[
\frac{S_{\mathrm{P}}(K)}{S_{\mathrm{B}}(K)}=\frac{(1-\lambda)\,\epsilon_{\mathrm{B}}}{(1-\lambda)\,\epsilon_{\mathrm{B}}+\lambda_0},
\]
which is below $1$ when $\lambda_0>0$. This implies that post-training keeps only a $(1-\lambda)$ proportion of the base model's scalability but more than a $(1-\lambda)$ proportion of its failures, because tasks sharpened to $p=0$ still fail. Meanwhile, first-attempt accuracy improves exactly when $\epsilon_{\mathrm{P}}<\epsilon_{\mathrm{B}}$, i.e., $\lambda_0<\lambda\,\epsilon_{\mathrm{B}}$. Thus, when $0<\lambda_0<\lambda\,\epsilon_{\mathrm{B}}$, first-attempt accuracy improves while the calibrated tax remains strictly positive, which proves the second identity and the consequences discussed after Theorem~\ref{thm:proportional-tax} in \S\ref{section:theory}. Neither identity depends on which extreme each sharpened task moves to, so that choice may itself depend on $p_x$.
\end{proof}

\subsection{How PTGS Improves Group-Based RL}
\label{appendix:ptgs_signal}

\thmptgs*

\begin{proof}[Proof of Theorem~\ref{thm:ptgs-signal}]
Within a group, the $n$ rollouts are conditionally independent given the prompt, each succeeding with the prompt's true probability. Hence, for a prompt with success probability $q$, the probability that at least one rollout succeeds is $1-(1-q)^n = I_n(q)$, the probability that all succeed is $q^n$, and the probability of a mixed group, at least one success and at least one failure, is $1-q^n-(1-q)^n = J_n(q)$. Averaging the first quantity over prompts gives $\text{pass}@n=\E_x[I_n(p_x)]$, and subtracting the all-success probability gives $\text{pass}@n-\text{pass}^n=\E_x[J_n(p_x)]$, which we use after Proposition~\ref{thm:ptgs-mastered} below.

(i) $I_n'(q)=n(1-q)^{n-1}>0$ on $(0,1)$, so $I_n$ is strictly increasing and $p'>p$ implies $I_n(p')>I_n(p)$.

(ii) $J_n'(q)=n\big[(1-q)^{n-1}-q^{n-1}\big]$, which is positive if and only if $q<\tfrac12$ for $n\ge2$. Thus $J_n$ is strictly increasing on $[0,\tfrac12]$ and maximized at $q=\tfrac12$, so $p<p'\le\tfrac12$ implies $J_n(p')>J_n(p)$.
\end{proof}

\noindent\textbf{Extension to mastered prompts.} Theorem~\ref{thm:ptgs-signal} concerns the hard prompts that PTGS heats. The following proposition covers the mastered prompts that PTGS cools, for which raising the success probability makes mixed groups less frequent.

\begin{proposition}[Cooling mastered prompts]
\label{thm:ptgs-mastered}
In the setting of Theorem~\ref{thm:ptgs-signal}, if $\tfrac12 \le p < p'$, then
\begin{equation}
J_n(p') < J_n(p)
\qquad\text{and}\qquad
J_n(p) - J_n(p') \le (p')^{n} - p^{n}.
\label{eq:ptgs-mastered}
\end{equation}
\end{proposition}

Mixed groups become more sparse as a mastered prompt's success probability rises, but so do all-failure groups, and the bound in Eq.~\eqref{eq:ptgs-mastered} ensures that the lost mixed groups turn into all-success groups rather than all-failure groups. Since averaging over prompts gives $\text{pass}@n=\E_x[I_n(p_x)]$ and $\text{pass}@n-\text{pass}^n=\E_x[J_n(p_x)]$, Theorem~\ref{thm:ptgs-signal} and Proposition~\ref{thm:ptgs-mastered} say that PTGS gains coverage while a prompt is still being explored and converts that coverage into consistency once the prompt is mastered, which is consistent with the joint gains in Table~\ref{tab:ptgs-combined}.

\begin{proof}[Proof of Proposition~\ref{thm:ptgs-mastered}]
By the derivative computation in the proof of Theorem~\ref{thm:ptgs-signal}~(ii), $J_n$ is strictly decreasing on $[\tfrac12,1]$, so $\tfrac12\le p<p'$ implies $J_n(p')<J_n(p)$. Moreover, from $J_n(q)=1-q^n-(1-q)^n$,
\[
J_n(p)-J_n(p')
=\big[(p')^{n}-p^{n}\big]-\big[(1-p)^{n}-(1-p')^{n}\big]
\le (p')^{n}-p^{n},
\]
since $q\mapsto(1-q)^n$ is strictly decreasing on $[0,1]$. That is, the decrease in the probability of a mixed group is fully absorbed by the increase in the all-success probability $q^n$, which is itself strictly increasing in $q$.
\end{proof}

\FloatBarrier
\section{Additional Results on the Accuracy--Coverage Tension}
\label{appendix:results_observation}

This section gives the full results behind \S\ref{section:observation} for each backbone scale, model family, and checkpoint pair. The following three main findings of \S\ref{section:observation} hold in every family. First, post-training bimodalizes the per-task success rates (Figures~\ref{fig:taskpassrate_dist_gemma4}--\ref{fig:taskpassrate_dist_ministral3}). Second, the coverage and consistency curves of the base model diverge much more than those of the post-trained model (Figures~\ref{fig:appendix_ttscurve_gemma4}--\ref{fig:appendix_ttscurve_ministral3}). Third, the post-trained model stands out in both sampling efficiency and coverage at a small budget, but at a large budget the base model frequently reaches higher coverage, mostly for the larger backbones (Figure~\ref{fig:coveff_lostgain}). Below, we first compare the largest and smallest backbones (\S\ref{appendix:agg_scopes}), then show the results for each family and checkpoint pair (\S\ref{appendix:perfamily}), and look at the individual task curves on WebShop (\S\ref{appendix:webshop_density}). Finally, we show that a global temperature cannot resolve the trade-off (\S\ref{appendix:global_temperature}).

\FloatBarrier
\subsection{Aggregate Results by Backbone Scale}
\label{appendix:agg_scopes}

We group the model-benchmark combinations by backbone scale. The \emph{largest-backbone} set contains the largest backbone of each family (\texttt{gemma-4-31B}, \texttt{Ministral-3-14B}, \texttt{Qwen2.5-32B}, and \texttt{Qwen3.5-35B-A3B}) on the three benchmarks, and the \emph{smallest-backbone} set contains the smallest backbone of each family (\texttt{gemma-4-E4B}, \texttt{Ministral-3-3B}, \texttt{Qwen2.5-3B}, and \texttt{Qwen3.5-4B}), giving 12 combinations per set. The aggregate curves are a simple average over the 12 combinations, and the shaded bands show one standard deviation across combinations, i.e., the spread between models and benchmarks (not sampling error).

Figures~\ref{fig:agg12_ttscurve}--\ref{fig:agg12_diffhist} show that the two sets behave in qualitatively different ways. For the smallest backbones (3--4B), post-training mainly supplies basic skills that the base model lacks, e.g., well-formed tool calls, and thereby expands coverage. For the largest backbones, post-training mainly polarizes the per-task success rates and thereby loses coverage. Accordingly, the mean base curve overtakes the mean post-trained curve at around $k{=}22$ for the largest backbones but never does so for the smallest ones. At the task level, post-training loses more tasks than it gains for the largest backbones, while the opposite holds for the smallest ones. The same contrast appears in Table~\ref{table:bimodal_freq_full}, where post-training reduces the proportion of always-fail tasks on BFCL for every smallest backbone but increases it for three of the four largest backbones.

\begin{figure}[h]
    \centering
    \begin{minipage}[t]{0.48\linewidth}
    \includegraphics[width=\linewidth]{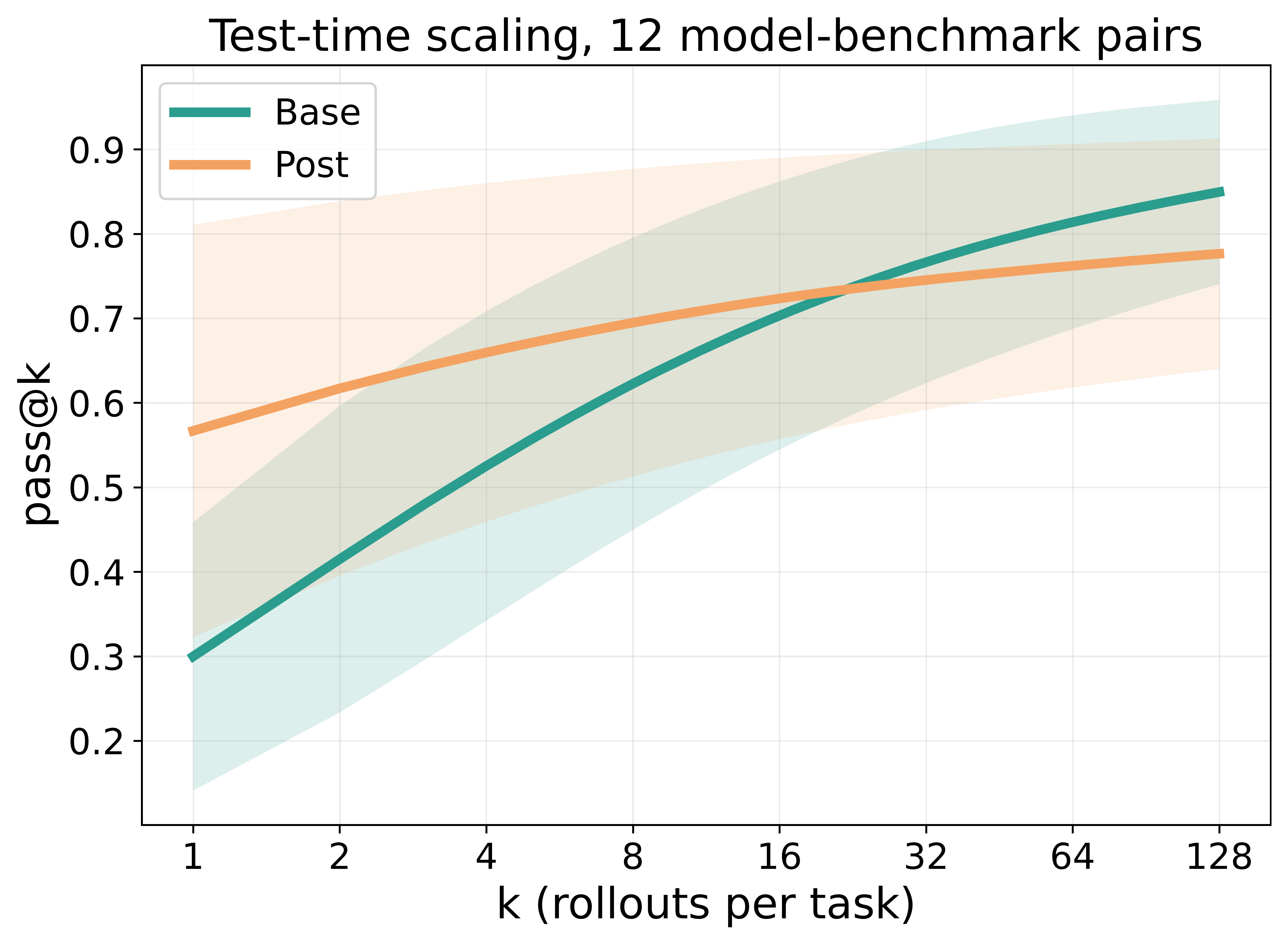}
    \end{minipage}\hfill
    \begin{minipage}[t]{0.48\linewidth}
    \includegraphics[width=\linewidth]{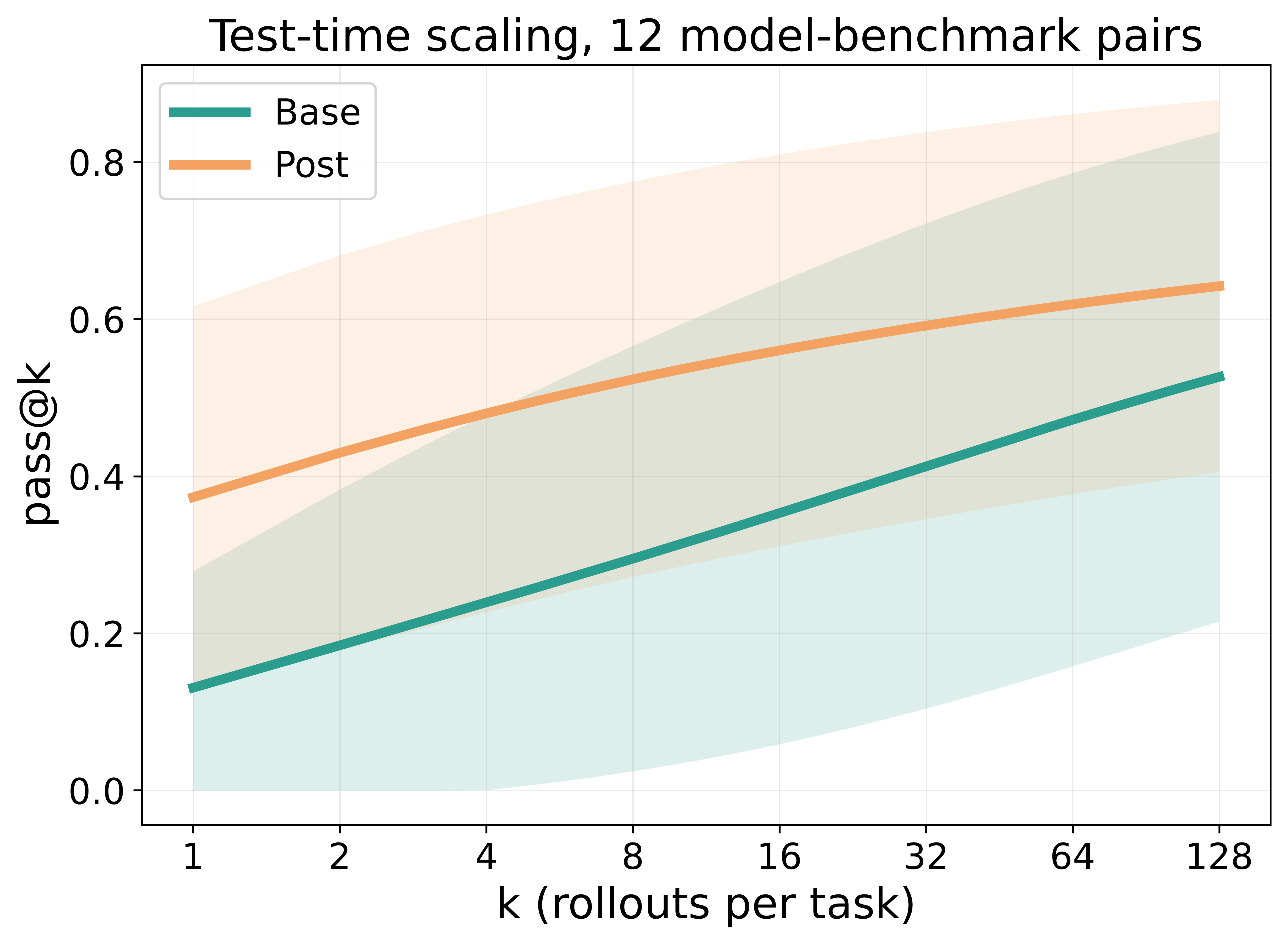}
    \end{minipage}
    \caption{\textbf{Aggregate test-time scaling curves of the base and post-trained models for the largest (left) and smallest (right) backbones.} Each curve is the mean pass@$k$ over 12 model-benchmark combinations (one backbone per family on each of the three benchmarks), and the bands show one standard deviation across combinations. For the largest backbones, the mean base curve overtakes the mean post-trained curve at $k \approx 22$ and ends clearly ahead at $k=128$. For the smallest backbones, the base curve never catches up within the budget.}
    \label{fig:agg12_ttscurve}
\end{figure}

\begin{figure}[h]
    \centering
    \begin{minipage}[t]{0.48\linewidth}
    \includegraphics[width=\linewidth]{figures/agg12L_covcons_noshade.pdf}
    \end{minipage}\hfill
    \begin{minipage}[t]{0.48\linewidth}
    \includegraphics[width=\linewidth]{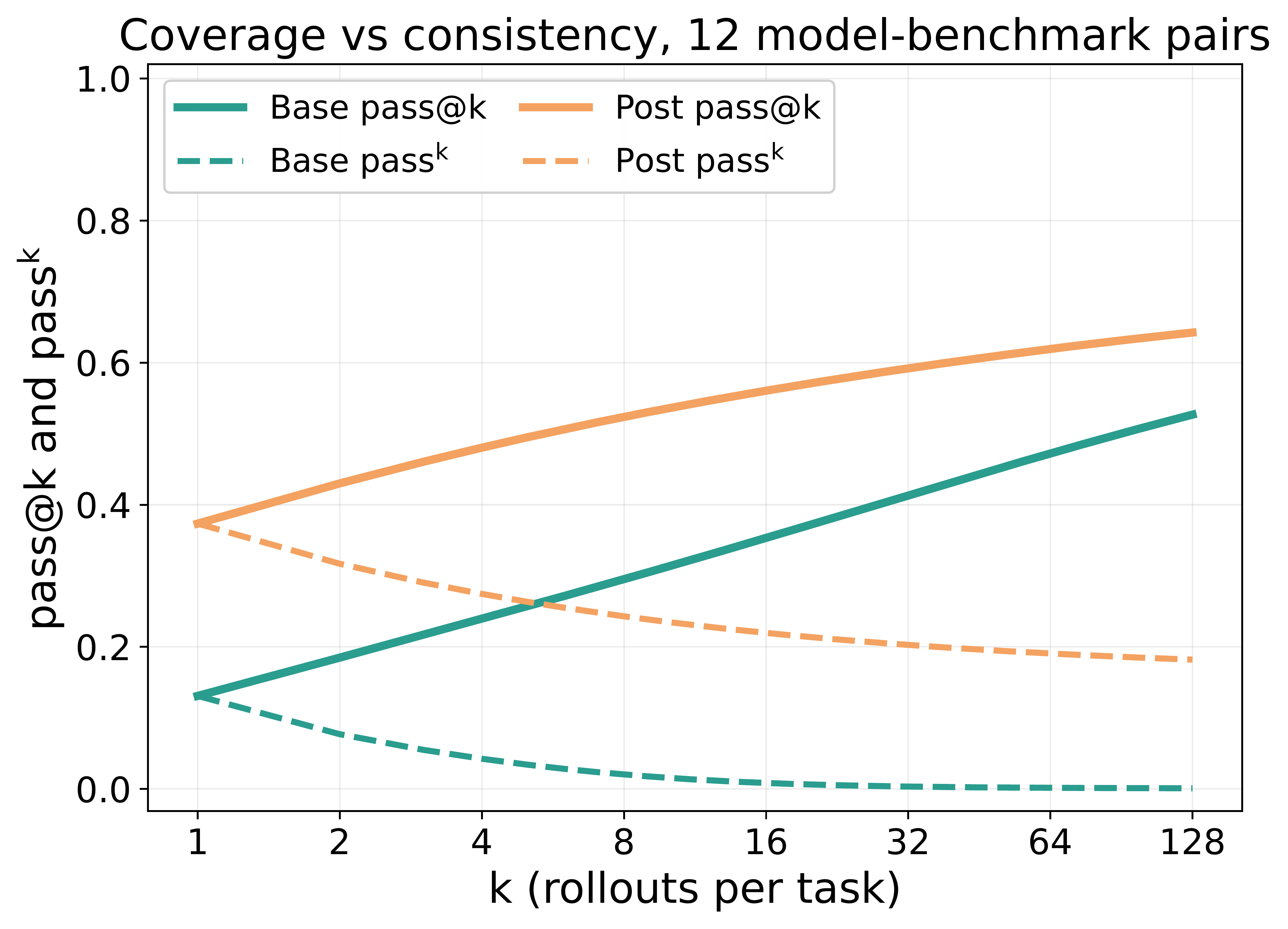}
    \end{minipage}
    \caption{\textbf{Aggregate coverage (pass@$k$) and consistency ($\text{pass}^k$) curves of the base and post-trained models for the largest (left) and smallest (right) backbones}, averaged over 12 model-benchmark combinations (one backbone per family on each of the three benchmarks). The consistency of the base model drops to nearly zero as $k$ grows, but that of the post-trained model stays much higher at both scales. For the largest backbones, more samples let the base model overtake the post-trained model in coverage, but not in consistency.}
    \label{fig:agg12_covcons}
\end{figure}

\begin{figure}[h]
    \centering
    \begin{minipage}[t]{0.48\linewidth}
    \includegraphics[width=\linewidth]{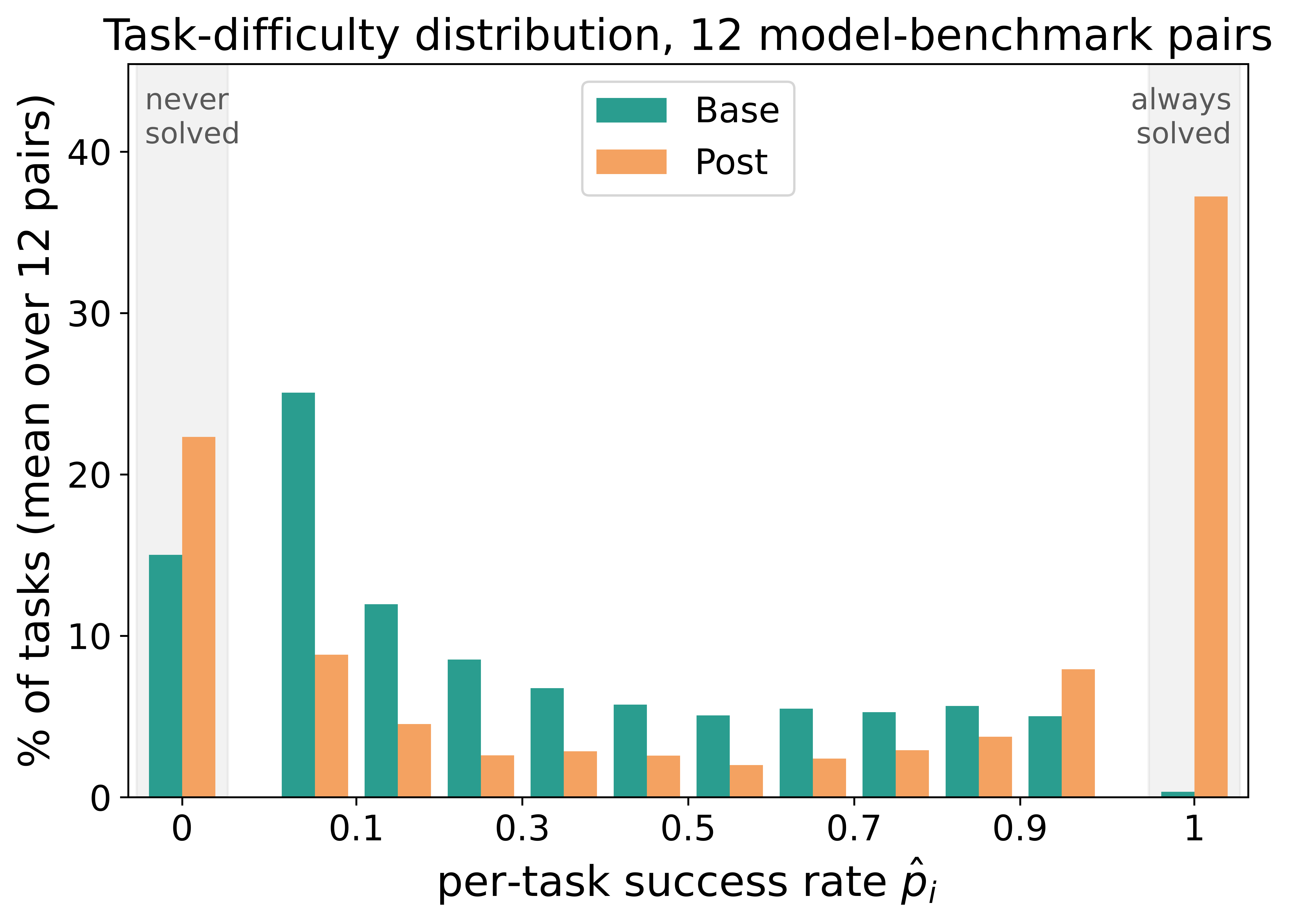}
    \end{minipage}\hfill
    \begin{minipage}[t]{0.48\linewidth}
    \includegraphics[width=\linewidth]{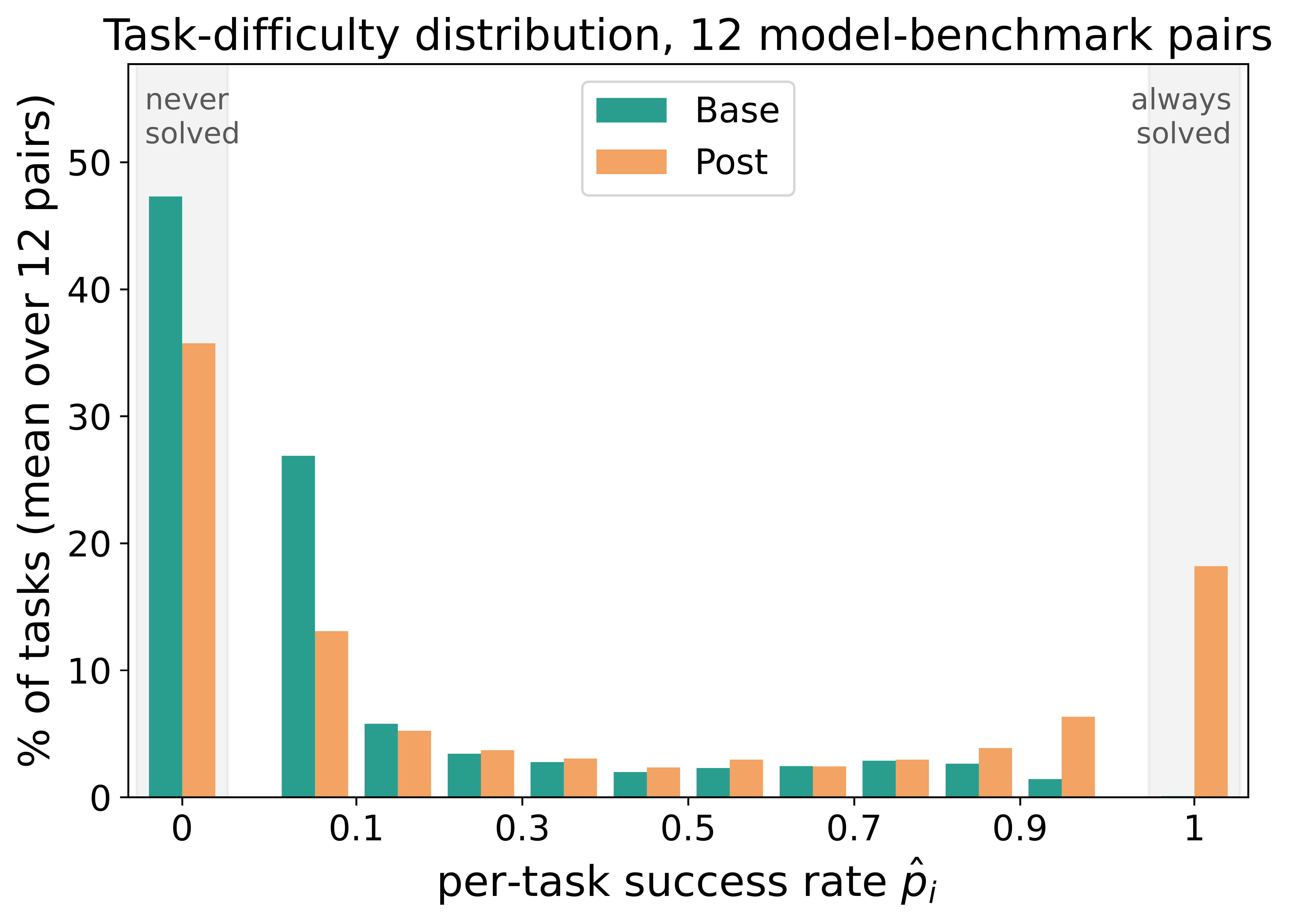}
    \end{minipage}
    \caption{\textbf{Aggregate per-task success-rate distributions of the base and post-trained models for the largest (left) and smallest (right) backbones}, averaged over 12 model-benchmark combinations (one backbone per family on each of the three benchmarks). For the largest backbones, post-training increases both the never-solved share and the always-solved share (the latter from nearly 0\% to 37\%), i.e., it loses some solvable tasks and fully solves others. For the smallest backbones, post-training instead reduces the never-solved share by repairing the base model.}
    \label{fig:agg12_diffhist}
\end{figure}

\FloatBarrier
\subsection{Results by Model Family and Checkpoint Pair}
\label{appendix:perfamily}

Table~\ref{table:bimodal_freq_full} and Figures~\ref{fig:taskpassrate_dist_gemma4}--\ref{fig:coveff_lostgain} show the results for every family and checkpoint pair, including the task categorization per benchmark at $K=128$, the per-task success-rate distributions, the coverage and consistency curves, and the comparison between the coverage lost and the sampling efficiency gained by post-training.

\begin{table}[t]
\centering
\caption{\textbf{Task categories of all 14 base/post-trained pairs at $K=128$.} Each task is classified by its outcomes over 128 rollouts as \emph{always pass} (every rollout succeeds), \emph{pass given compute} (some but not all rollouts succeed), or \emph{always fail} (no rollout succeeds). Numbers are percentages of tasks in each benchmark. Teal ($\uparrow$) and purple ($\downarrow$) mark an increase and a decrease after post-training. For the larger backbones of every family, the middle category shrinks and its tasks move to the two extremes. For the smallest backbones on BFCL, post-training can instead enlarge the middle category by moving tasks out of always fail, as it improves the tool-call formatting capability of the base model. This is the setting in which the tax can become negative.}
\label{table:bimodal_freq_full}
\resizebox{0.875\linewidth}{!}{%
\begin{NiceTabular}{llcccccc}
\CodeBefore
\Body
\toprule
 & & \multicolumn{2}{c}{BFCLv4 MT} & \multicolumn{2}{c}{WebShop} & \multicolumn{2}{c}{ACEBench} \\
\cmidrule(lr){3-4}\cmidrule(lr){5-6}\cmidrule(lr){7-8}
Backbone & Task category & Base & Post & Base & Post & Base & Post \\
\midrule
\Block{3-1}{\texttt{gemma-4-E4B}} & Always pass (any $k$) & 0.0 & \textcolor{teal}{12.5 ($\uparrow$)} & 0.0 & \textcolor{teal}{11.4 ($\uparrow$)} & 0.0 & \textcolor{teal}{66.4 ($\uparrow$)} \\
 & Pass given compute & 27.0 & \textcolor{purple}{26.0 ($\downarrow$)} & 48.6 & \textcolor{teal}{53.8 ($\uparrow$)} & 85.6 & \textcolor{purple}{22.9 ($\downarrow$)} \\
 & Always fail (any $k$) & 73.0 & \textcolor{purple}{61.5 ($\downarrow$)} & 51.4 & \textcolor{purple}{34.8 ($\downarrow$)} & 14.4 & \textcolor{purple}{10.8 ($\downarrow$)} \\
\midrule
\Block{3-1}{\texttt{gemma-4-12B}} & Always pass (any $k$) & 0.0 & \textcolor{teal}{60.5 ($\uparrow$)} & 0.0 & \textcolor{teal}{14.0 ($\uparrow$)} & 0.5 & \textcolor{teal}{83.2 ($\uparrow$)} \\
 & Pass given compute & 47.0 & \textcolor{purple}{13.5 ($\downarrow$)} & 78.0 & \textcolor{purple}{40.6 ($\downarrow$)} & 90.0 & \textcolor{purple}{7.0 ($\downarrow$)} \\
 & Always fail (any $k$) & 53.0 & \textcolor{purple}{26.0 ($\downarrow$)} & 22.0 & \textcolor{teal}{45.4 ($\uparrow$)} & 9.5 & \textcolor{teal}{9.7 ($\uparrow$)} \\
\midrule
\Block{3-1}{\texttt{gemma-4-26B-A4B}} & Always pass (any $k$) & 0.0 & \textcolor{teal}{60.5 ($\uparrow$)} & 0.0 & \textcolor{teal}{17.2 ($\uparrow$)} & 0.0 & \textcolor{teal}{82.5 ($\uparrow$)} \\
 & Pass given compute & 62.5 & \textcolor{purple}{20.5 ($\downarrow$)} & 76.4 & \textcolor{purple}{29.6 ($\downarrow$)} & 93.8 & \textcolor{purple}{7.8 ($\downarrow$)} \\
 & Always fail (any $k$) & 37.5 & \textcolor{purple}{19.0 ($\downarrow$)} & 23.6 & \textcolor{teal}{53.2 ($\uparrow$)} & 6.2 & \textcolor{teal}{9.7 ($\uparrow$)} \\
\midrule
\Block{3-1}{\texttt{gemma-4-31B}} & Always pass (any $k$) & 0.0 & \textcolor{teal}{74.0 ($\uparrow$)} & 0.0 & \textcolor{teal}{26.0 ($\uparrow$)} & 0.4 & \textcolor{teal}{84.9 ($\uparrow$)} \\
 & Pass given compute & 90.5 & \textcolor{purple}{9.5 ($\downarrow$)} & 87.6 & \textcolor{purple}{30.0 ($\downarrow$)} & 93.8 & \textcolor{purple}{6.5 ($\downarrow$)} \\
 & Always fail (any $k$) & 9.5 & \textcolor{teal}{16.5 ($\uparrow$)} & 12.4 & \textcolor{teal}{44.0 ($\uparrow$)} & 5.8 & \textcolor{teal}{8.6 ($\uparrow$)} \\
\midrule
\Block{3-1}{\texttt{Ministral-3-3B}} & Always pass (any $k$) & 0.0 & \textcolor{teal}{4.5 ($\uparrow$)} & 0.0 & \textcolor{teal}{0.2 ($\uparrow$)} & 0.0 & \textcolor{teal}{16.4 ($\uparrow$)} \\
 & Pass given compute & 24.0 & \textcolor{teal}{57.0 ($\uparrow$)} & 32.6 & \textcolor{teal}{68.6 ($\uparrow$)} & 85.7 & \textcolor{purple}{69.5 ($\downarrow$)} \\
 & Always fail (any $k$) & 76.0 & \textcolor{purple}{38.5 ($\downarrow$)} & 67.4 & \textcolor{purple}{31.2 ($\downarrow$)} & 14.3 & \textcolor{purple}{14.2 ($\downarrow$)} \\
\midrule
\Block{3-1}{\texttt{Ministral-3-8B}} & Always pass (any $k$) & 0.0 & \textcolor{teal}{9.0 ($\uparrow$)} & 0.0 & \textcolor{teal}{0.6 ($\uparrow$)} & 0.0 & \textcolor{teal}{40.6 ($\uparrow$)} \\
 & Pass given compute & 54.0 & \textcolor{teal}{55.5 ($\uparrow$)} & 78.8 & \textcolor{purple}{54.0 ($\downarrow$)} & 93.2 & \textcolor{purple}{51.2 ($\downarrow$)} \\
 & Always fail (any $k$) & 46.0 & \textcolor{purple}{35.5 ($\downarrow$)} & 21.2 & \textcolor{teal}{45.4 ($\uparrow$)} & 6.8 & \textcolor{teal}{8.2 ($\uparrow$)} \\
\midrule
\Block{3-1}{\texttt{Ministral-3-14B}} & Always pass (any $k$) & 0.0 & \textcolor{teal}{11.5 ($\uparrow$)} & 0.0 & \textcolor{teal}{0.8 ($\uparrow$)} & 0.0 & \textcolor{teal}{50.4 ($\uparrow$)} \\
 & Pass given compute & 59.0 & \textcolor{purple}{53.5 ($\downarrow$)} & 78.8 & \textcolor{purple}{72.2 ($\downarrow$)} & 93.9 & \textcolor{purple}{42.9 ($\downarrow$)} \\
 & Always fail (any $k$) & 41.0 & \textcolor{purple}{35.0 ($\downarrow$)} & 21.2 & \textcolor{teal}{27.0 ($\uparrow$)} & 6.1 & \textcolor{teal}{6.8 ($\uparrow$)} \\
\midrule
\Block{3-1}{\texttt{Qwen2.5-3B}} & Always pass (any $k$) & 0.0 & \textcolor{teal}{5.5 ($\uparrow$)} & 0.0 & \textcolor{teal}{1.2 ($\uparrow$)} & 0.0 & \textcolor{teal}{26.8 ($\uparrow$)} \\
 & Pass given compute & 12.0 & \textcolor{teal}{31.0 ($\uparrow$)} & 7.6 & \textcolor{teal}{12.0 ($\uparrow$)} & 80.0 & \textcolor{purple}{42.6 ($\downarrow$)} \\
 & Always fail (any $k$) & 88.0 & \textcolor{purple}{63.5 ($\downarrow$)} & 92.4 & \textcolor{purple}{86.8 ($\downarrow$)} & 20.0 & \textcolor{teal}{30.6 ($\uparrow$)} \\
\midrule
\Block{3-1}{\texttt{Qwen2.5-7B}} & Always pass (any $k$) & 0.5 & \textcolor{teal}{8.0 ($\uparrow$)} & 0.0 & \textcolor{teal}{10.6 ($\uparrow$)} & 0.0 & \textcolor{teal}{51.0 ($\uparrow$)} \\
 & Pass given compute & 21.5 & \textcolor{teal}{52.0 ($\uparrow$)} & 71.2 & \textcolor{purple}{46.2 ($\downarrow$)} & 89.2 & \textcolor{purple}{24.0 ($\downarrow$)} \\
 & Always fail (any $k$) & 78.0 & \textcolor{purple}{40.0 ($\downarrow$)} & 28.8 & \textcolor{teal}{43.2 ($\uparrow$)} & 10.8 & \textcolor{teal}{24.9 ($\uparrow$)} \\
\midrule
\Block{3-1}{\texttt{Qwen2.5-14B}} & Always pass (any $k$) & 0.5 & \textcolor{teal}{7.5 ($\uparrow$)} & 0.0 & \textcolor{teal}{1.4 ($\uparrow$)} & 1.6 & \textcolor{teal}{59.2 ($\uparrow$)} \\
 & Pass given compute & 58.5 & \textcolor{teal}{67.5 ($\uparrow$)} & 74.0 & \textcolor{purple}{53.8 ($\downarrow$)} & 89.2 & \textcolor{purple}{21.3 ($\downarrow$)} \\
 & Always fail (any $k$) & 41.0 & \textcolor{purple}{25.0 ($\downarrow$)} & 26.0 & \textcolor{teal}{44.8 ($\uparrow$)} & 9.2 & \textcolor{teal}{19.5 ($\uparrow$)} \\
\midrule
\Block{3-1}{\texttt{Qwen2.5-32B}} & Always pass (any $k$) & 1.0 & \textcolor{teal}{12.5 ($\uparrow$)} & 0.0 & \textcolor{teal}{11.8 ($\uparrow$)} & 1.7 & \textcolor{teal}{74.0 ($\uparrow$)} \\
 & Pass given compute & 70.0 & \textcolor{purple}{51.5 ($\downarrow$)} & 83.8 & \textcolor{purple}{50.0 ($\downarrow$)} & 92.7 & \textcolor{purple}{13.9 ($\downarrow$)} \\
 & Always fail (any $k$) & 29.0 & \textcolor{teal}{36.0 ($\uparrow$)} & 16.2 & \textcolor{teal}{38.2 ($\uparrow$)} & 5.6 & \textcolor{teal}{12.1 ($\uparrow$)} \\
\midrule
\Block{3-1}{\texttt{Qwen3.5-4B}} & Always pass (any $k$) & 1.0 & \textcolor{teal}{19.5 ($\uparrow$)} & 0.0 & \textcolor{teal}{5.6 ($\uparrow$)} & 0.1 & \textcolor{teal}{48.6 ($\uparrow$)} \\
 & Pass given compute & 62.5 & \textcolor{teal}{65.0 ($\uparrow$)} & 71.6 & \textcolor{purple}{65.0 ($\downarrow$)} & 93.9 & \textcolor{purple}{39.1 ($\downarrow$)} \\
 & Always fail (any $k$) & 36.5 & \textcolor{purple}{15.5 ($\downarrow$)} & 28.4 & \textcolor{teal}{29.4 ($\uparrow$)} & 6.0 & \textcolor{teal}{12.3 ($\uparrow$)} \\
\midrule
\Block{3-1}{\texttt{Qwen3.5-9B}} & Always pass (any $k$) & 0.0 & \textcolor{teal}{19.5 ($\uparrow$)} & 0.0 & \textcolor{teal}{1.0 ($\uparrow$)} & 0.0 & \textcolor{teal}{45.6 ($\uparrow$)} \\
 & Pass given compute & 84.0 & \textcolor{purple}{65.0 ($\downarrow$)} & 73.4 & \textcolor{purple}{70.6 ($\downarrow$)} & 95.1 & \textcolor{purple}{46.6 ($\downarrow$)} \\
 & Always fail (any $k$) & 16.0 & \textcolor{purple}{15.5 ($\downarrow$)} & 26.6 & \textcolor{teal}{28.4 ($\uparrow$)} & 4.9 & \textcolor{teal}{7.8 ($\uparrow$)} \\
\midrule
\Block{3-1}{\texttt{Qwen3.5-35B-A3B}} & Always pass (any $k$) & 1.0 & \textcolor{teal}{33.5 ($\uparrow$)} & 0.0 & \textcolor{teal}{3.2 ($\uparrow$)} & 0.0 & \textcolor{teal}{64.2 ($\uparrow$)} \\
 & Pass given compute & 86.0 & \textcolor{purple}{53.0 ($\downarrow$)} & 84.4 & \textcolor{purple}{71.4 ($\downarrow$)} & 95.1 & \textcolor{purple}{30.8 ($\downarrow$)} \\
 & Always fail (any $k$) & 13.0 & \textcolor{teal}{13.5 ($\uparrow$)} & 15.6 & \textcolor{teal}{25.4 ($\uparrow$)} & 4.9 & \textcolor{teal}{5.1 ($\uparrow$)} \\
\bottomrule
\end{NiceTabular}}
\end{table}

\begin{figure}[htb]
    \centering
    \includegraphics[width=0.9\linewidth]{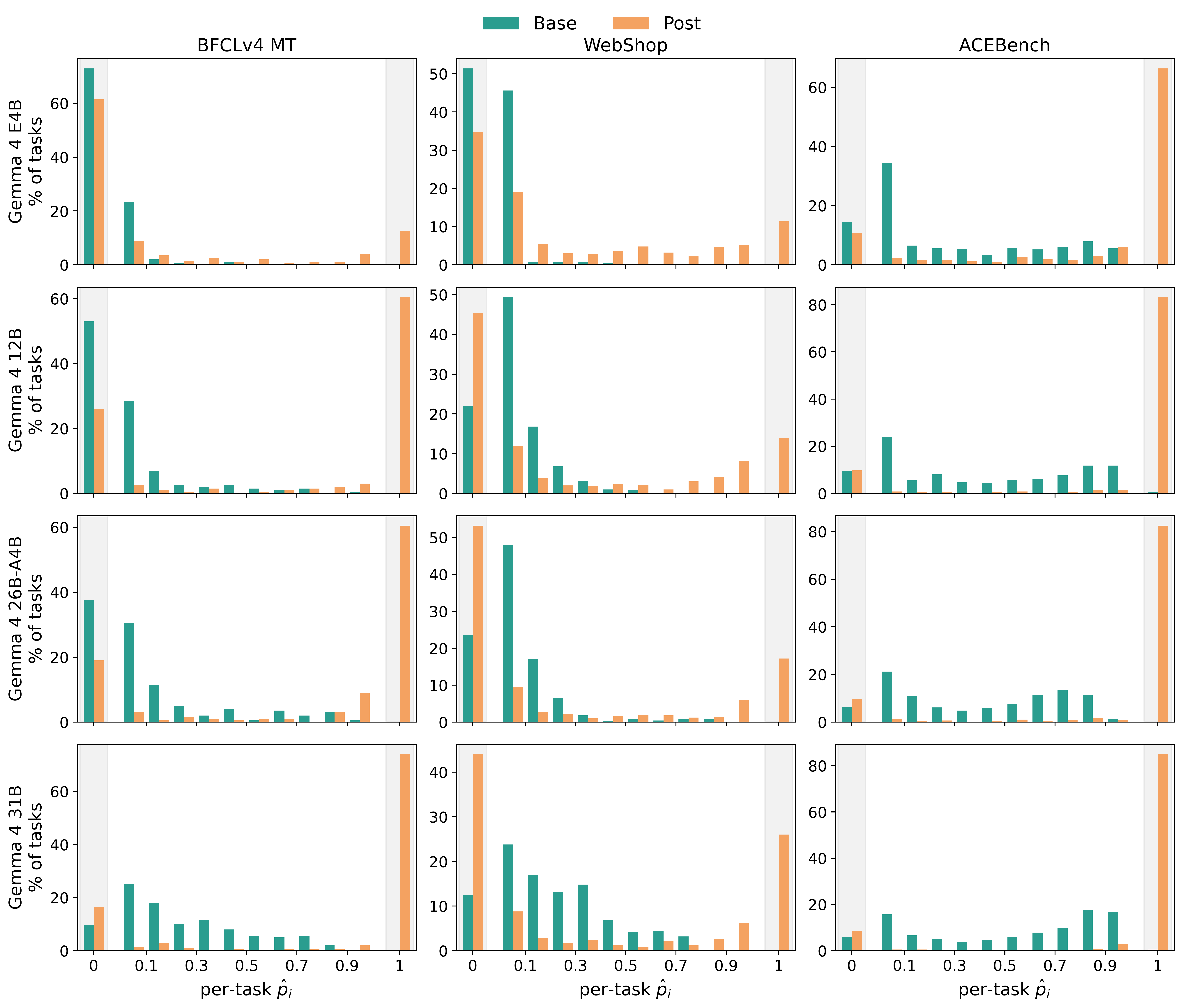}
    \caption{\textbf{Per-task success-rate distributions of the Gemma-4 base and post-trained models} (empirical success rate over 128 rollouts per task) for every backbone and benchmark. Post-training moves most of the intermediate mass to the two extremes at every model size and on every benchmark. Among the four families, this polarization is the strongest for Gemma-4, which also has the largest mean tax.}
    \label{fig:taskpassrate_dist_gemma4}
\end{figure}

\begin{figure}
    \centering
    \includegraphics[width=0.9\linewidth]{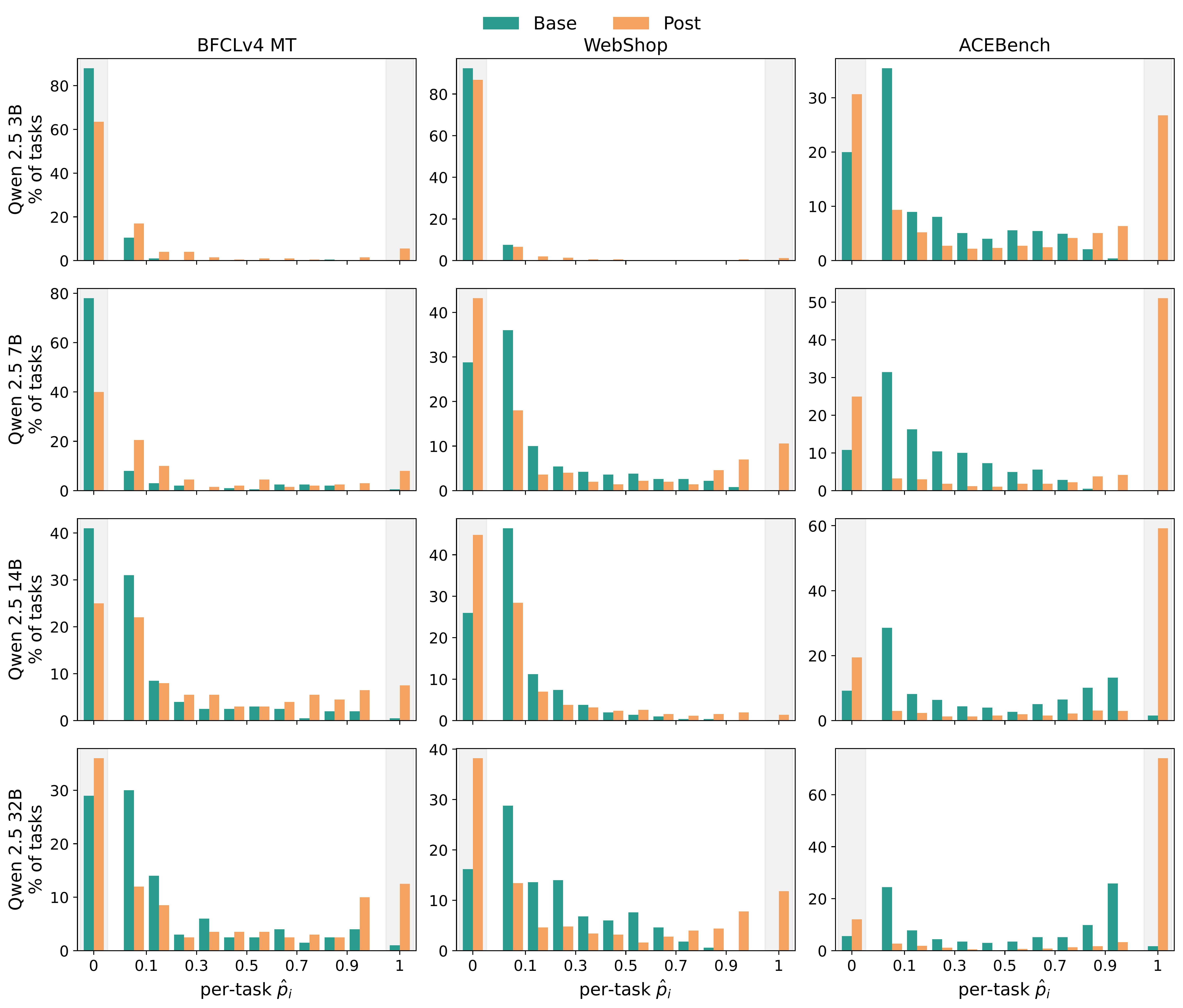}
    \caption{\textbf{Per-task success-rate distributions of the Qwen2.5 base and post-trained models} (empirical success rate over 128 rollouts per task) for every backbone and benchmark. Post-training moves the intermediate mass to the two extremes, as in the other families. 
    }
    \label{fig:taskpassrate_dist_qwen25}
\end{figure}

\begin{figure}
    \centering
    \includegraphics[width=0.9\linewidth]{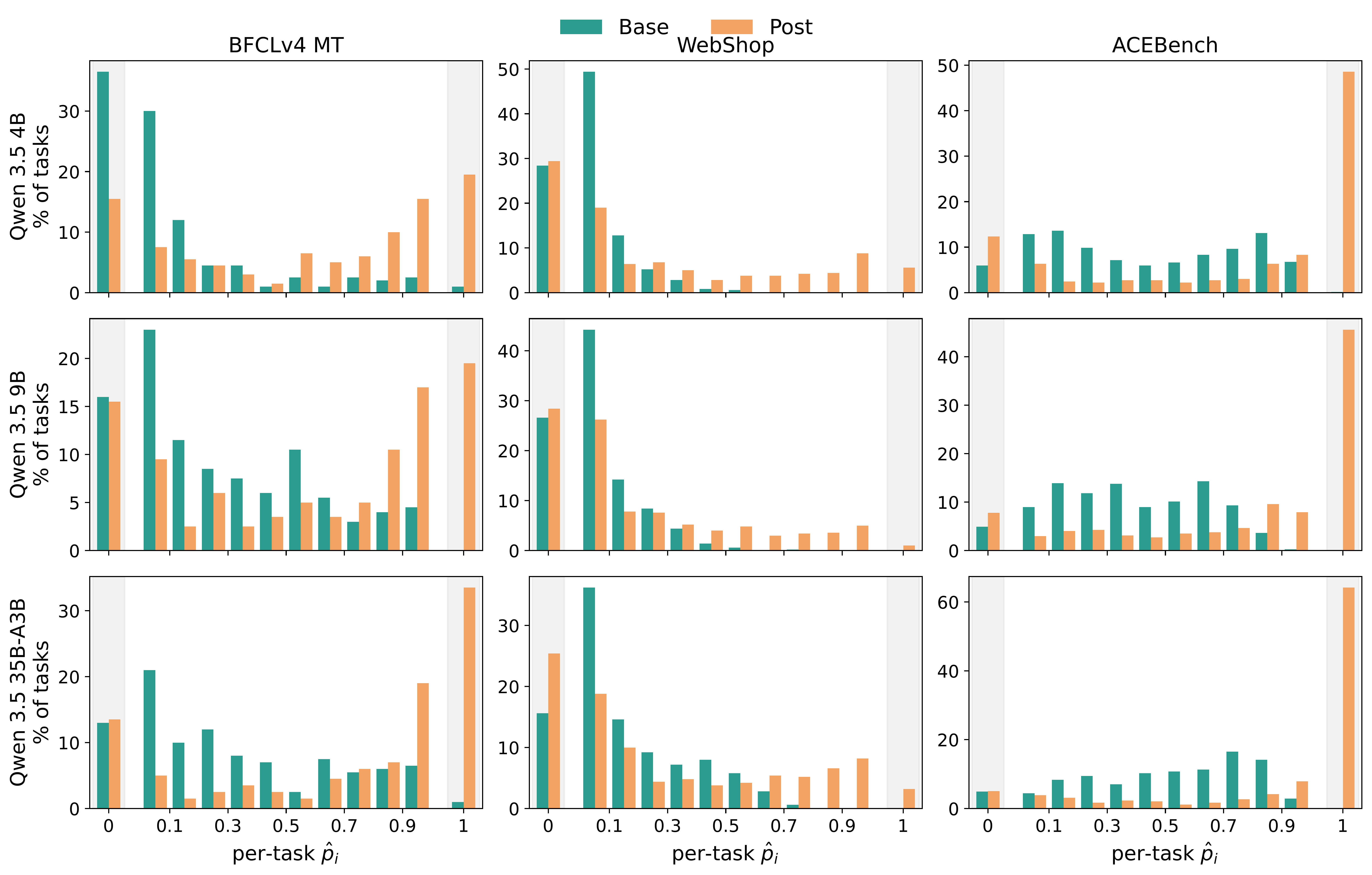}
    \caption{\textbf{Per-task success-rate distributions of the Qwen3.5 base and post-trained models} (empirical success rate over 128 rollouts per task) for every backbone and benchmark. Qwen3.5 shows the softest sharpening among the four families. 
    }
    \label{fig:taskpassrate_dist_qwen35}
\end{figure}

\begin{figure}
    \centering
    \includegraphics[width=0.9\linewidth]{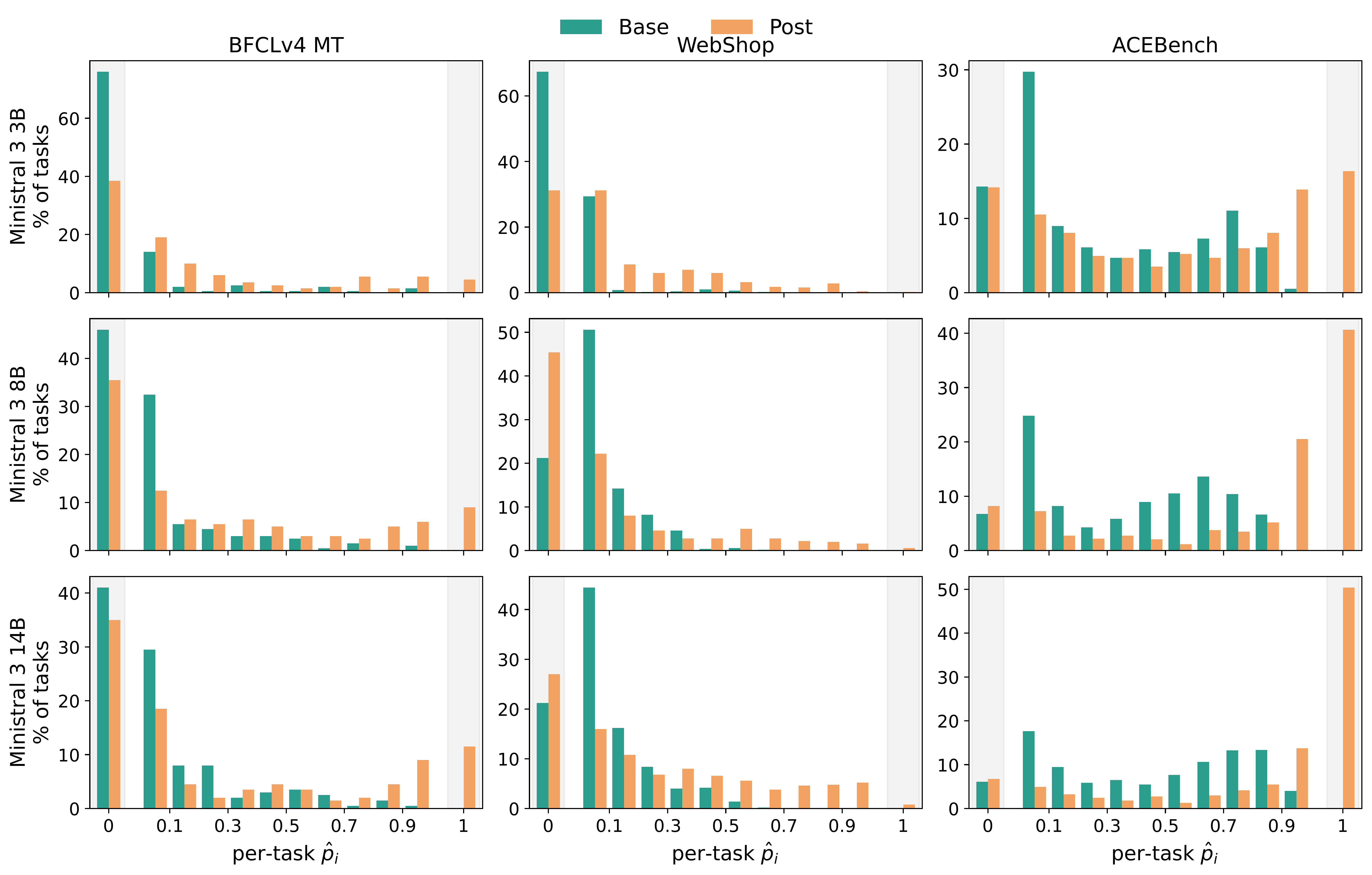}
    \caption{\textbf{Per-task success-rate distributions of the Ministral-3 base and post-trained models} (empirical success rate over 128 rollouts per task) for every backbone and benchmark. The polarization grows with model scale. For the 3B backbone, the base model rarely succeeds on BFCL and WebShop and post-training largely improves it, while the 8B and 14B backbones show the shift of mass to the two extremes as we observed in other backbones.}
    \label{fig:taskpassrate_dist_ministral3}
\end{figure}

\begin{figure}
    \centering
    \includegraphics[width=0.9\linewidth]{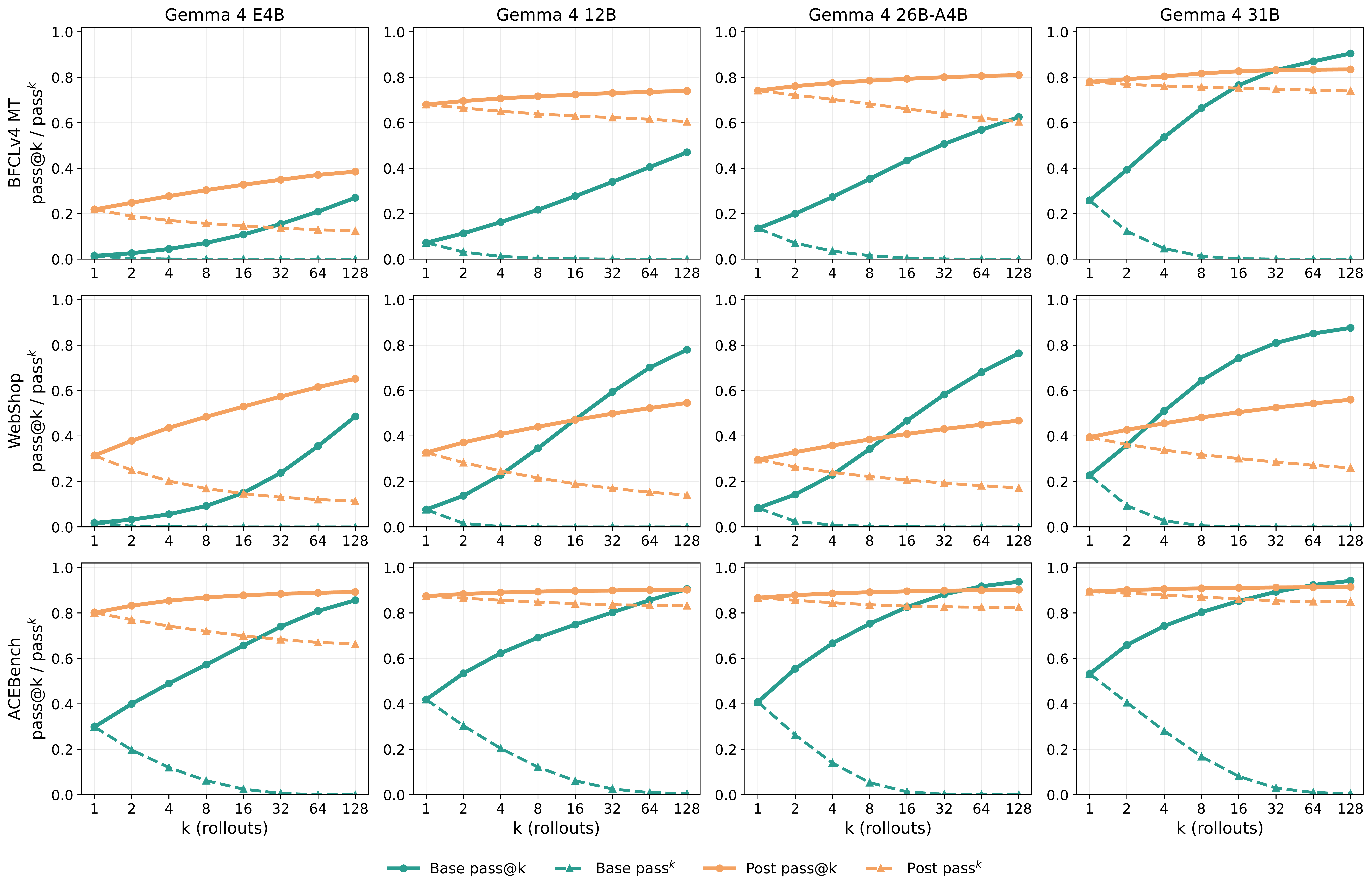}
    \caption{\textbf{Coverage (pass@$k$) and consistency ($\text{pass}^k$) curves of the Gemma-4 base and post-trained models} for every backbone and benchmark. The two curves of the post-trained model stay close together, as expected from a bimodalized policy. In contrast, the consistency of the base model drops toward zero while its coverage keeps rising. The base model overtakes the post-trained model in coverage at a smaller budget for larger backbones.}
    \label{fig:appendix_ttscurve_gemma4}
\end{figure}

\begin{figure}
    \centering
    \includegraphics[width=0.9\linewidth]{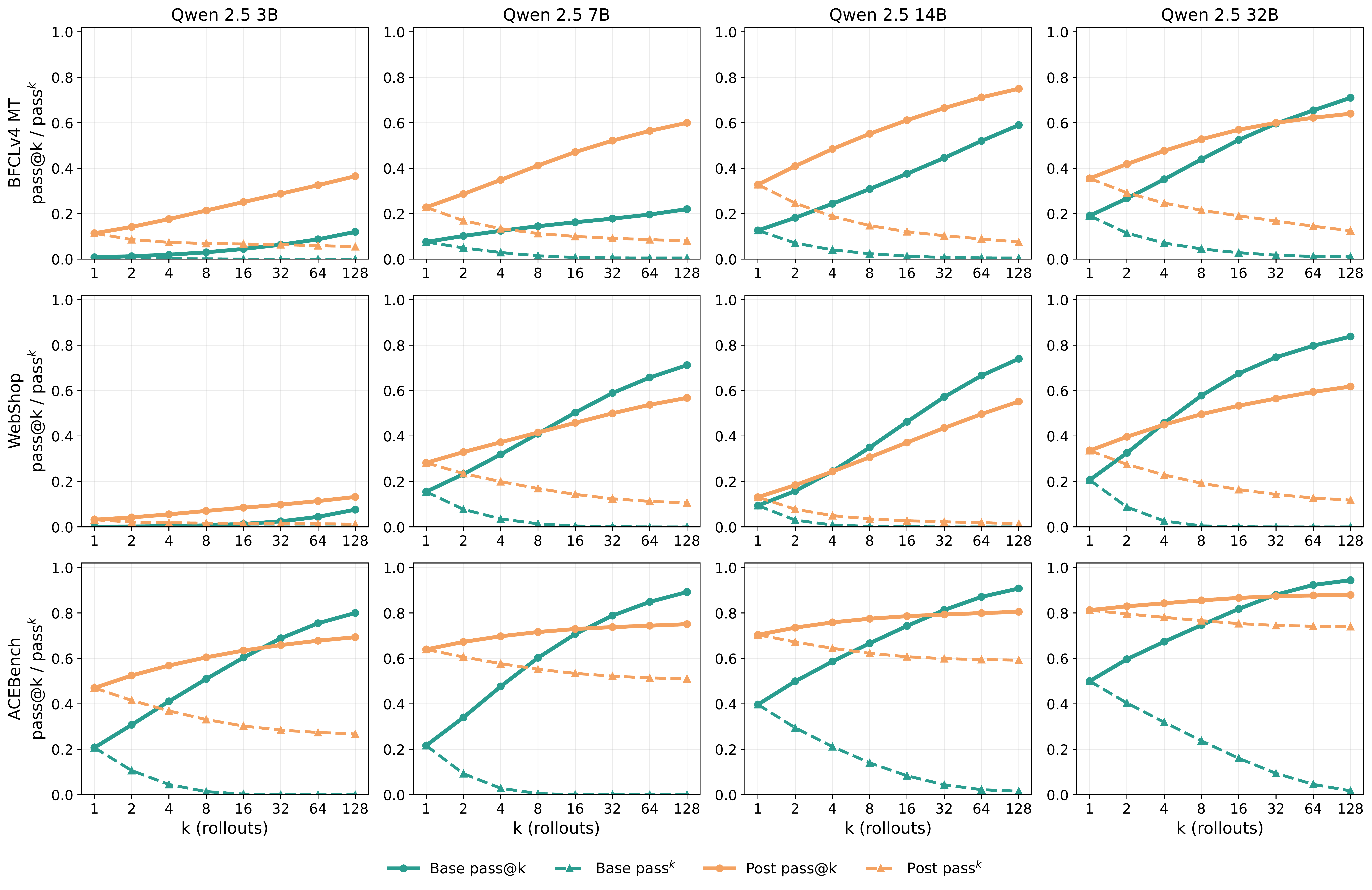}
    \caption{\textbf{Coverage (pass@$k$) and consistency ($\text{pass}^k$) curves of the Qwen2.5 base and post-trained models} for every backbone and benchmark. The consistency of the base model drops toward zero while its coverage keeps rising, and the two curves of the post-trained model stay much closer together.}
    \label{fig:appendix_ttscurve_qwen25}
\end{figure}

\begin{figure}
    \centering
    \includegraphics[width=0.9\linewidth]{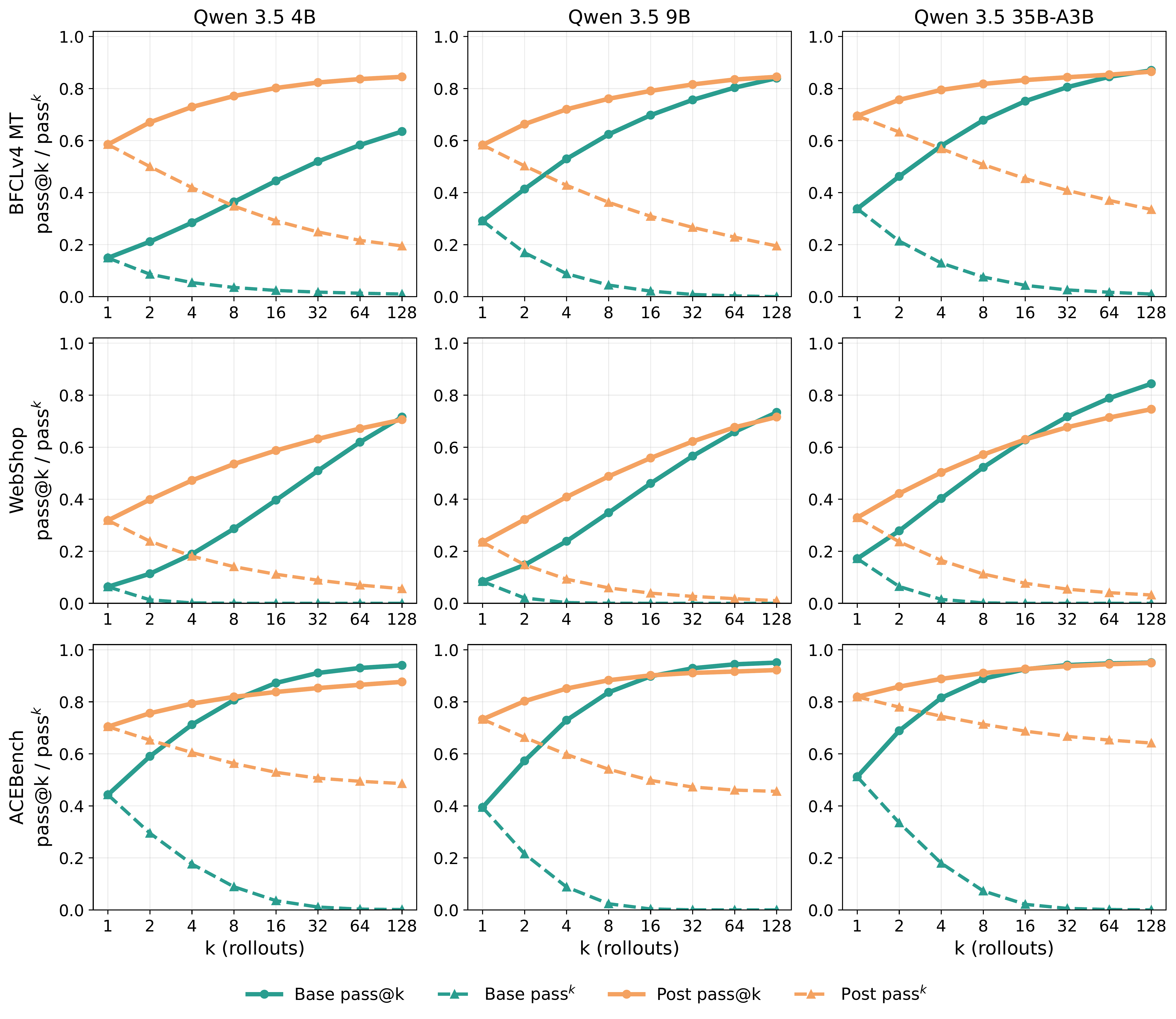}
    \caption{\textbf{Coverage (pass@$k$) and consistency ($\text{pass}^k$) curves of the Qwen3.5 base and post-trained models} for every backbone and benchmark. Since Qwen3.5 has the softest sharpening among the four families, the consistency curves of its post-trained models decay slightly instead of staying flat, and the coverage gaps between the base and post-trained models at $k=128$ are the smallest among the four families.}
    \label{fig:appendix_ttscurve_qwen35}
\end{figure}

\begin{figure}
    \centering
    \includegraphics[width=0.9\linewidth]{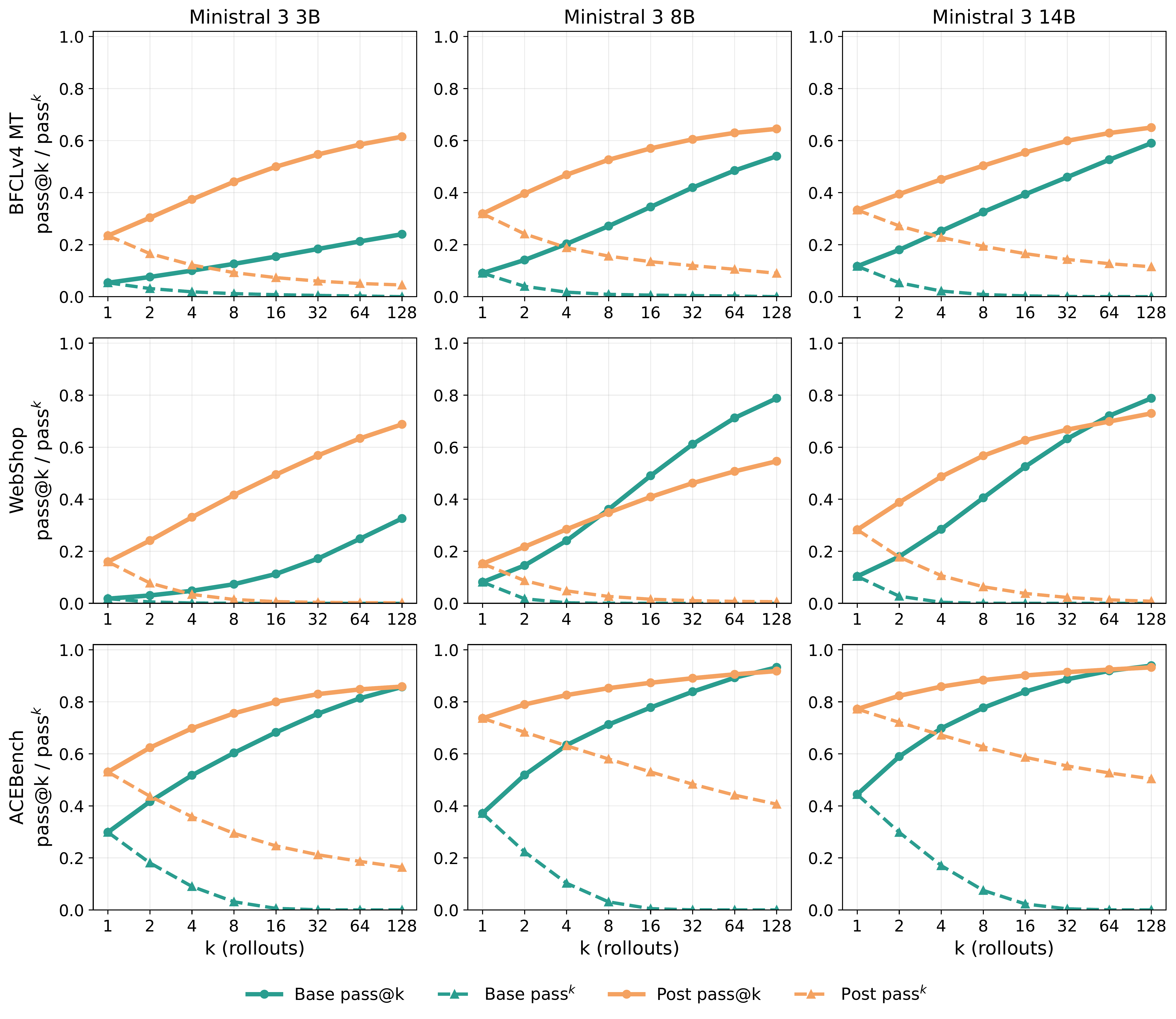}
    \caption{\textbf{Coverage (pass@$k$) and consistency ($\text{pass}^k$) curves of the Ministral-3 base and post-trained models} for every backbone and benchmark. The consistency of the base model drops toward zero while its coverage keeps rising, and the two curves of the post-trained model stay much closer together.}
    \label{fig:appendix_ttscurve_ministral3}
\end{figure}

\begin{figure}
    \centering
    \includegraphics[width=0.9\linewidth]{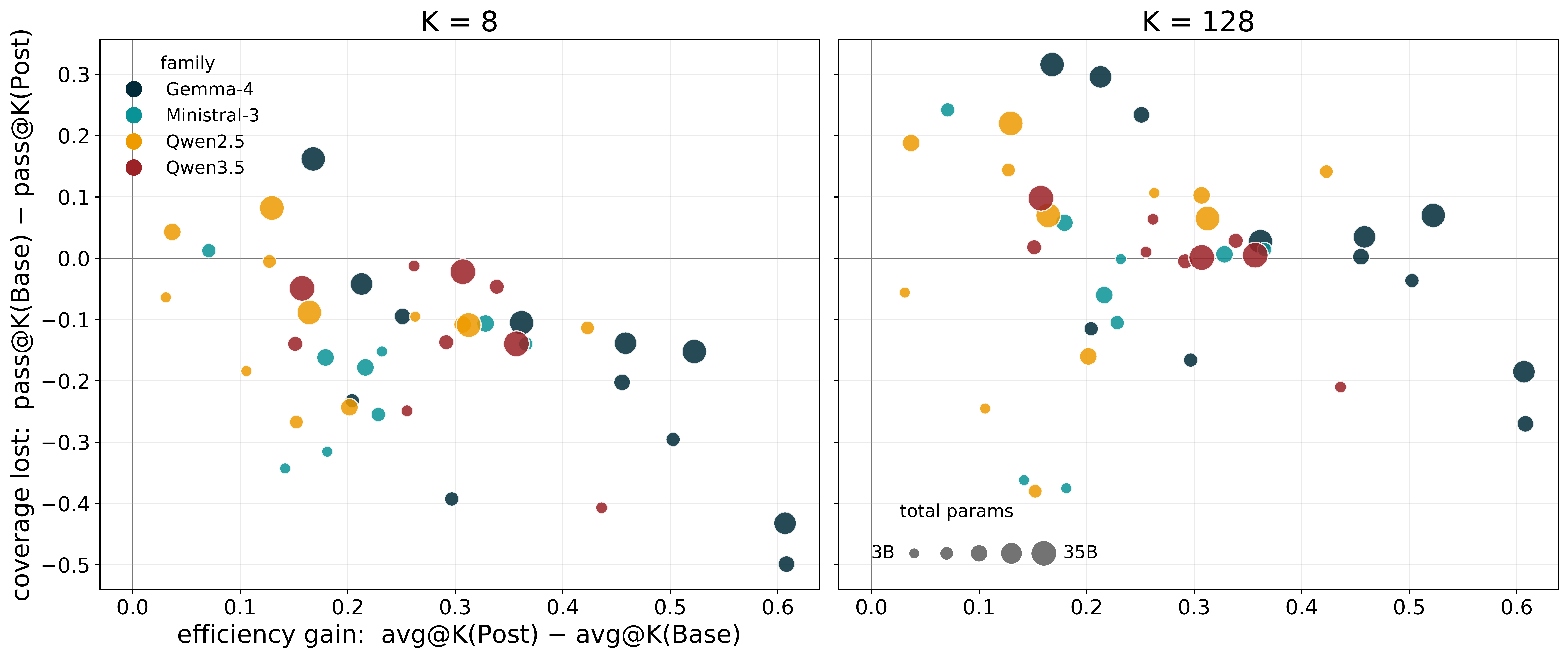}
    \caption{\textbf{Coverage lost versus sampling efficiency gained by post-training} at $K{=}8$ (left) and $K=128$ (right). The $y$-axis is the coverage lost, $\text{pass}@K_{\text{base}} - \text{pass}@K_{\text{RL}}$, and the $x$-axis is the efficiency gained, $\text{avg}@K_{\text{RL}} - \text{avg}@K_{\text{base}}$, where avg@$K$ is the mean success rate over $K$ rollouts. Each point is one model--benchmark combination, colored by family, and the marker size grows with the total number of parameters. At $K{=}8$, most points lie below zero, i.e., post-training is beneficial in both coverage and efficiency when compute is scarce. At $K=128$, many points, mostly from the larger backbones, cross above zero, implying that post-training still gains efficiency, but now pays for it with lost coverage.}
    \label{fig:coveff_lostgain}
\end{figure}

\FloatBarrier
\subsection{Per-Task Scaling Curves and Their Density}
\label{appendix:webshop_density}

Figure~\ref{fig:webshop_density_gemma} and Figure~\ref{fig:bfcl_density_gemma} show the scaling curves of individual WebShop tasks and BFCLv4 tasks (128 rollouts per policy) for the smallest and largest backbones of Gemma-4, the family with the steepest sharpening. The top row of each block shows example tasks of each policy at the harder, intermediate, and easier difficulty ranks, and the bottom row shows the density of all 500 per-task curves together with their mean.

\begin{figure}[htb]
    \centering
    \includegraphics[width=\textwidth]{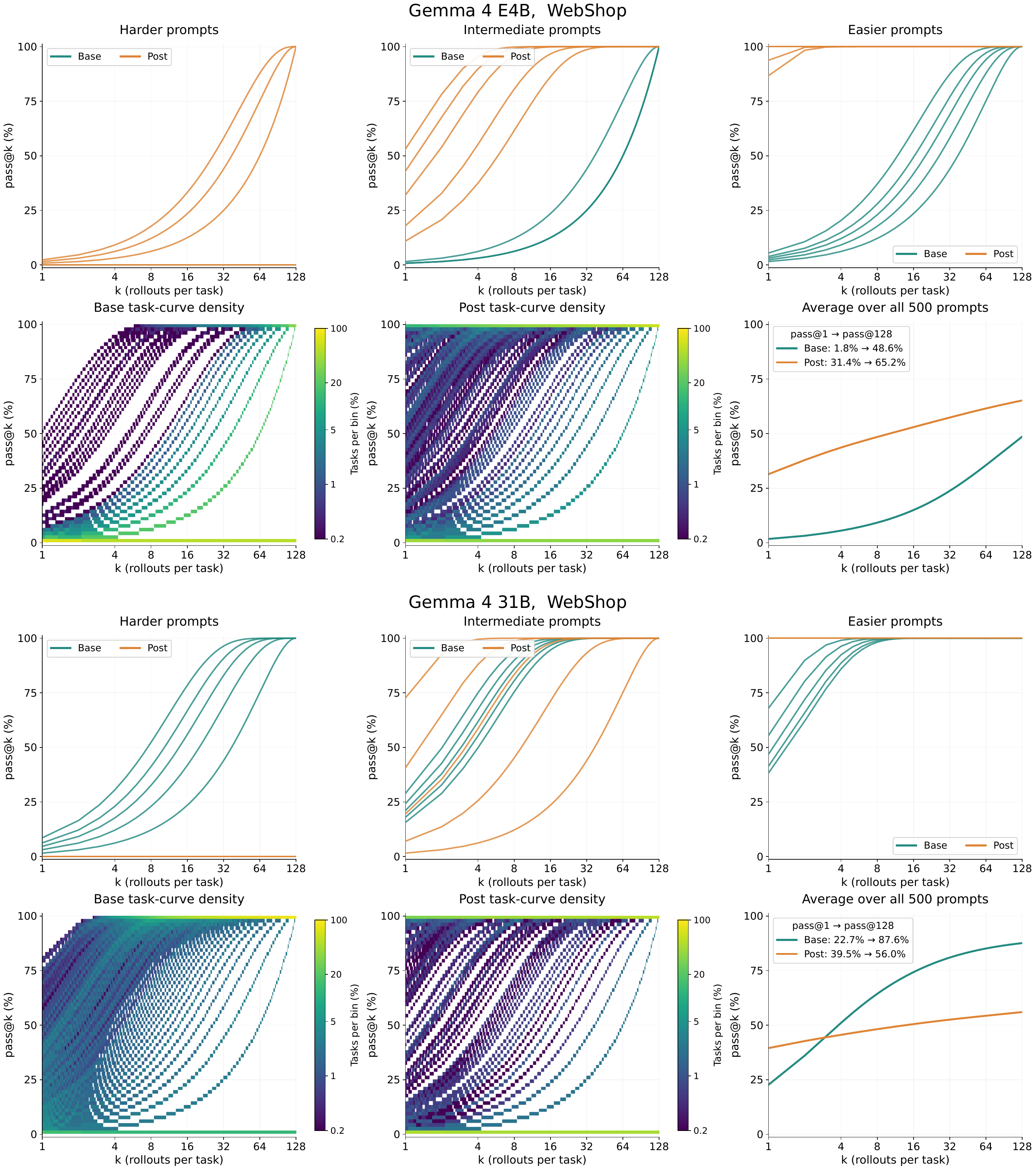}
    \caption{\textbf{Per-task scaling curves on WebShop for the smallest and largest Gemma-4 backbones} (\texttt{gemma-4-E4B}, top two rows; \texttt{gemma-4-31B}, bottom two rows). In each block, the top row shows five base (teal) and five post-trained (orange) task curves from the harder (5--25\%), intermediate (40--60\%), and easier (75--95\%) rank quantiles. Tasks are ranked separately for each policy by their success counts, among the tasks that at least one policy solves at least once. The bottom row shows the density of all 500 per-task curves of each policy and their mean pass@$k$, with pass@$1$ and pass@$128$ given in the legend. All curves use the unbiased estimator from 128 rollouts, and a flat curve at zero means no observed success.}
    \label{fig:webshop_density_gemma}
\end{figure}

\begin{figure}[htb]
    \centering
    \includegraphics[width=\textwidth]{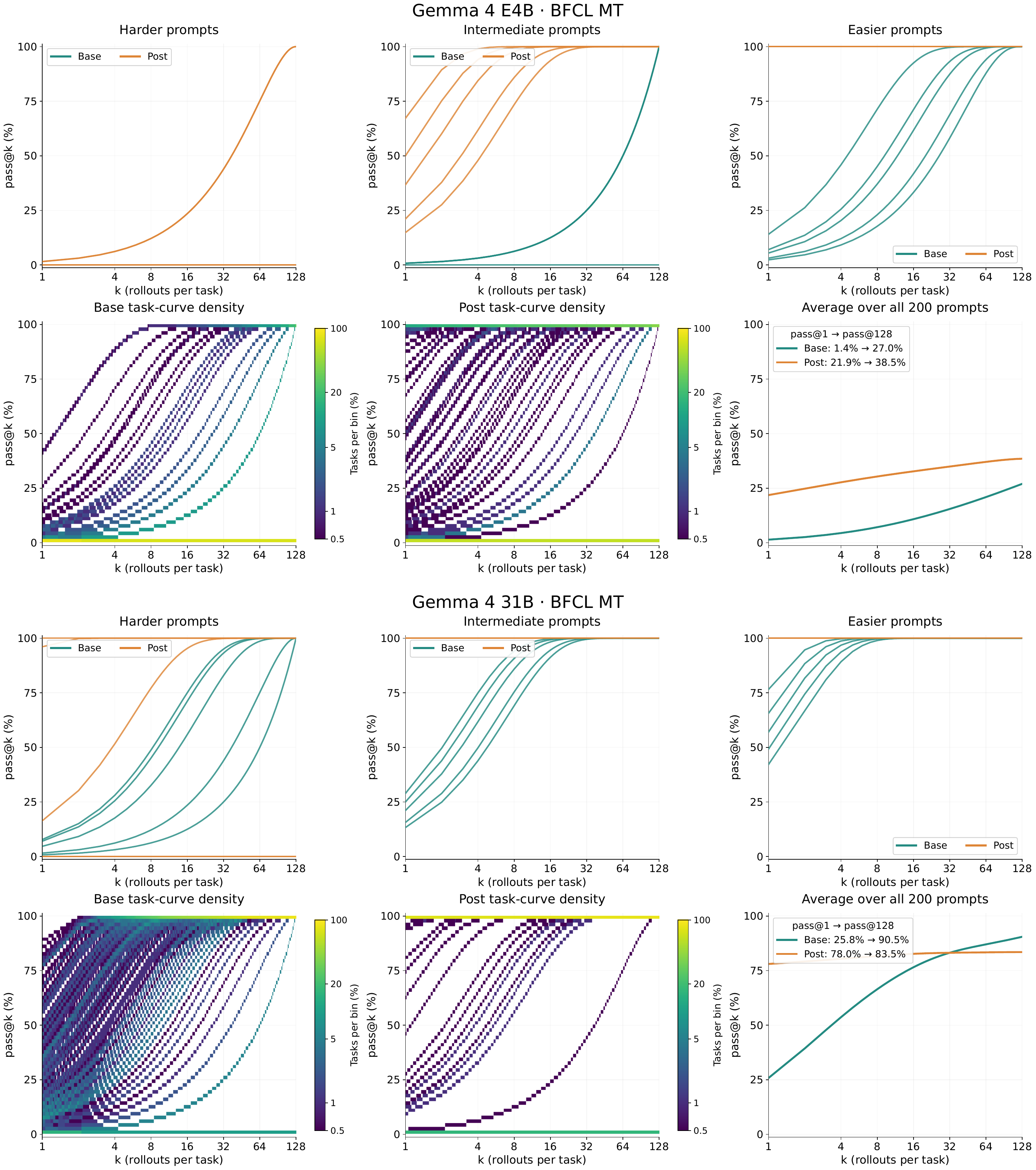}
    \caption{\textbf{Per-task scaling curves on BFCLv4 for the smallest and largest Gemma-4 backbones} (\texttt{gemma-4-E4B}, top two rows; \texttt{gemma-4-31B}, bottom two rows). In each block, the top row shows five base (teal) and five post-trained (orange) task curves from the harder (5--25\%), intermediate (40--60\%), and easier (75--95\%) rank quantiles. Tasks are ranked separately for each policy by their success counts, among the tasks that at least one policy solves at least once. The bottom row shows the density of all 200 per-task curves of each policy and their mean pass@$k$, with pass@$1$ and pass@$128$ given in the legend. All curves use the unbiased estimator from 128 rollouts, and a flat curve at zero means no observed success.}
    \label{fig:bfcl_density_gemma}
\end{figure}

\noindent\textbf{Bimodalization across the budget axis.} We first look at the largest backbone, \texttt{gemma-4-31B}, for which post-training loses coverage. The example curves in the top row already show the bimodalization of \S\ref{section:observation} (presented as an aggregate histogram) at the level of individual tasks. In the easier quantile, the post-trained model solves all five example tasks at $k{=}1$, whereas the base model needs up to about 16 rollouts to saturate. In the harder quantile, the order is reversed. The base model starts from a few percent but reaches 100\% on all five example tasks within the budget, while the post-trained model mostly stays at zero pass@$k$ for all 128 rollouts (on all five example tasks on WebShop). The intermediate quantile lies in between. On BFCLv4, the post-trained model already saturates at $k{=}1$ here as well, while on WebShop its curves in this quantile still climb with $k$.
The density panels in the bottom row confirm that this pattern holds for the whole task set, not only for the examples. The density of the post-trained model concentrates on the two edges at every $k$, i.e., always-solved tasks along the top (100\% pass@$k$) and never-solved tasks along the bottom (zero pass@$k$), with a much sparser band in between. The density of the base model instead spreads over the intermediate region and moves upward as $k$ grows, so that most tasks reach the top only at large $k$. These are the tasks that retries turn into successes, and they explain why the mean base curve overtakes the mean post-trained curve, after only a few rollouts on WebShop and at around $k{=}30$ on BFCLv4, and ends well above it at $k{=}128$.

\noindent\textbf{Contrast with the smallest backbone.} The smallest backbone, \texttt{gemma-4-E4B}, behaves differently because its base model rarely succeeds at all (pass@$1$ of 1.8\% on WebShop and 1.4\% on BFCLv4). Most of the density of the base model lies on the bottom edge, and most of the tasks it does solve rise only late along the lowest curves, i.e., it succeeds on them in only a few of the 128 rollouts. The example curves show the consequence. In the harder quantile, the base model stays at zero pass@$k$, whereas the post-trained model still reaches 100\% on some example tasks by $k{=}128$. Even in the easier quantile, the base model needs tens of rollouts to saturate, whereas the post-trained model saturates within one or two rollouts. Here, post-training mainly supplies skills that the base model lacks, e.g., well-formed tool calls, and moves tasks off the never-solved edge rather than polarizing tasks that the base model can already solve. The density of the post-trained model still has an always-solved band at the top, but, especially on BFCLv4, it keeps more curves in the intermediate region than that of the post-trained \texttt{gemma-4-31B}. As a result, the mean post-trained curve stays above the mean base curve at every budget (65.2\% vs.\ 48.6\% on WebShop and 38.5\% vs.\ 27.0\% on BFCLv4 at $k{=}128$), consistent with the comparison of the smallest and largest backbones in \S\ref{appendix:agg_scopes}.
The WebShop curves of the other three families (not shown) follow the same contrast in a milder form. The base model of the largest backbone overtakes its post-trained counterpart within the budget, while that of the smallest backbone stays behind or only catches up at $k{=}128$.

\subsection{Global Temperature Scaling Cannot Resolve the Trade-off}
\label{appendix:global_temperature}

\begin{figure*}[thb]
     \centering
     \includegraphics[width=0.94\textwidth]{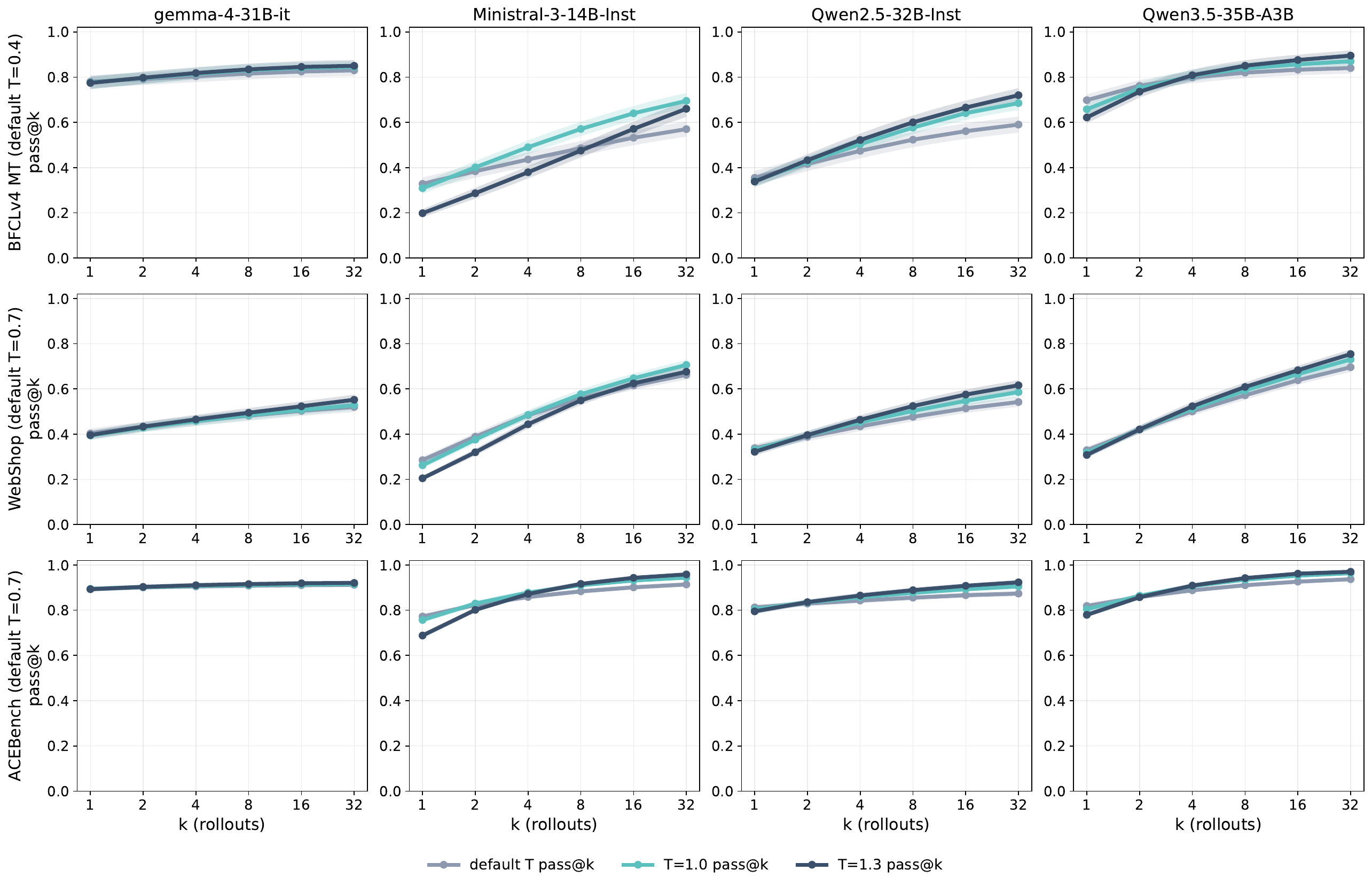}
     \vspace{-0.75em}
     \caption{\textbf{Global temperature scaling trades single-sample accuracy for coverage rather than improving both.} pass@$k$ of the post-trained model of the largest backbone in each family (columns) on BFCLv4 MT, WebShop, and ACEBench (rows), at the default temperature of each benchmark (0.4 for BFCL and 0.7 for the others) and at $T \in \{1.0, 1.3\}$. Shaded bands show $\pm 1$ bootstrap standard error over tasks (200 resamples). A higher temperature usually improves pass@$k$ at larger budgets, but pass@$1$ drops in some cases.}
     \vspace{-1em}
     \label{fig:temperaturescaling-rl}
\end{figure*}

A simple way to recover the lost coverage is to flatten the post-trained policy with a higher sampling temperature at inference, i.e., to heat every prompt with higher temperature. Figure~\ref{fig:temperaturescaling-rl} increases the temperature of the post-trained model of the largest backbone in each family from the default of each benchmark up to $T{=}1.3$. A higher temperature usually improves pass@$k$ at larger budgets, which confirms that part of the lost coverage can be recovered by widening the sampling distribution~\citep{dang2025weight}. 
However, pass@$1$ does not improve in any of the twelve panels and even drops clearly in some cases, e.g., \texttt{Ministral-3-14B-Instruct-2512}. 
The coverage gains are also neither uniform nor monotone in $T$. \texttt{gemma-4-31B-it}, the most sharply post-trained policy, hardly responds to heating, and for \texttt{Ministral-3-14B-Instruct-2512}, $T{=}1.0$ gives a higher pass@$32$ than $T{=}1.3$ on BFCL and WebShop. 
In summary, a single global temperature therefore moves the policy along the accuracy--coverage trade-off and cannot improve both. This suggests that \textbf{a prompt-adaptive temperature is needed to push the accuracy--coverage frontier upward}, which motivates PTGS in \S\ref{section:solution}.

\FloatBarrier
\section{Additional Results on the Sharpening Tax}
\label{appendix:results_tax}

This section gives additional results for \S\ref{section:diagnosis}. We report both tax variants for all 14 checkpoint pairs together with the full numerical results (\S\ref{appendix:tax_allpairs}), decompose the tax by how tasks move after post-training (\S\ref{appendix:tax_decomposition}), show the scaling curves of the tasks that contribute most to the tax (\S\ref{appendix:tax_extremes}), and use a cheap tax estimate to route each task to the base or the post-trained policy (\S\ref{appendix:routing}).

\subsection{Sharpening Tax for All Fourteen Checkpoint Pairs}
\label{appendix:tax_allpairs}

Figure~\ref{fig:sharpeningtax_all14} shows $\text{Tax}_A(k)$ and $\text{Tax}_{S}(k)$ as a function of the rollout budget for all 14 base/post-trained pairs, extending Figure~\ref{fig:sharpeningtax} of the main text. Table~\ref{table:two_metric_summary_allfam} summarizes all 42 model-benchmark evaluations. Every evaluation uses $N=128$ rollouts per task for both models and the same task set for each benchmark (200, 500, and 770 tasks for BFCL, WebShop, and ACEBench). For each combination, we report the coverage $C = \text{pass}@128$ of both models, the coverage gap $\Delta C = C_B - C_P$ with a 95\% paired task-bootstrap confidence interval, the budget-averaged coverage $\overline{C}(128)$ of both models, and the two tax metrics. By Corollary~\ref{cor:two-metric}, the raw scalability of each model is $A(128) = 128\,(C - \overline{C}(128))$.

\begin{figure}[htb]
     \centering
     \includegraphics[width=\linewidth]{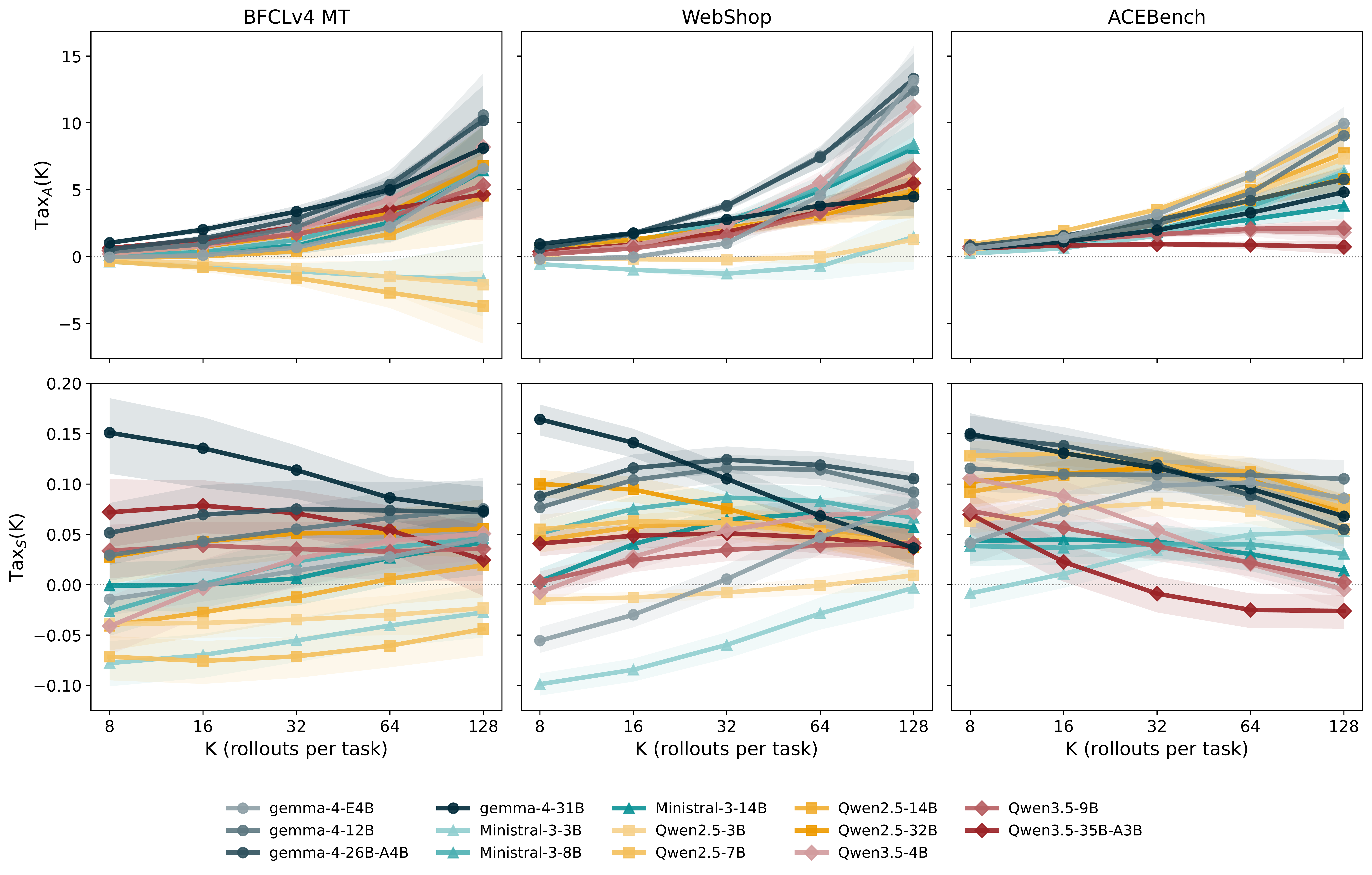}
     \caption{\textbf{Sharpening Tax as a function of the rollout budget for all 14 base/RL pairs.} $\text{Tax}_A(k)$ (top) and $\text{Tax}_{S}(k)$ (bottom) on BFCLv4 MT, WebShop, and ACEBench, with one curve per checkpoint pair and shaded 95\% bootstrap confidence intervals. Positive values mean that post-training reduces test-time scalability relative to the base model. The raw tax starts near zero and grows with the budget for almost every pair, while the calibrated tax can either rise or fall with the budget. Both variants stay negative at $k=128$ only for \texttt{Ministral-3-3B}, \texttt{Qwen2.5-3B}, and \texttt{Qwen2.5-7B} on BFCL, where post-training repairs the tool-call formatting of the base model.}
     \label{fig:sharpeningtax_all14}
\end{figure}

\begin{table}[h]
\centering
\caption{\textbf{Summary of the 42 model--benchmark evaluations at $k=128$.} $C_B$ and $C_P$ are the coverage ($\text{pass}@128$) of the base and post-trained (RL) models, $\Delta C = C_B - C_P$ is the coverage gap with a 95\% paired task-bootstrap confidence interval, and $\overline{C}_B$ and $\overline{C}_P$ are the budget-averaged coverage $\overline{C}(128) = \frac{1}{128}\sum_{k=1}^{128}\text{pass}@k$ of the two models. The raw scalability of each model is $A(128) = 128\,(C - \overline{C}(128))$, and the last two columns are the raw and calibrated Sharpening Tax at $k=128$. In every combination, the post-trained model reaches a larger fraction of its ceiling ($\overline{C}/C$) than the base model.}
\label{table:two_metric_summary_allfam}
\resizebox{\linewidth}{!}{%
\begin{tabular}{llrrrrrrr}
\toprule
Benchmark & Pair & $C_B$ & $C_P$ & $\Delta C$ [95\% CI] & $\overline{C}_B$ & $\overline{C}_P$ & Tax$_A$ & Tax$_{S}$ \\
\midrule
BFCLv4 MT & Gemma 4 E4B & 0.270 & 0.385 & -0.115 [-0.190,-0.050] & 0.194 & 0.360 & +6.59 & +0.046 \\
BFCLv4 MT & Gemma 4 12B & 0.470 & 0.740 & -0.270 [-0.345,-0.200] & 0.380 & 0.733 & +10.59 & +0.075 \\
BFCLv4 MT & Gemma 4 26B-A4B & 0.625 & 0.810 & -0.185 [-0.255,-0.120] & 0.538 & 0.803 & +10.18 & +0.072 \\
BFCLv4 MT & Gemma 4 31B & 0.905 & 0.835 & +0.070 [+0.015,+0.120] & 0.838 & 0.831 & +8.11 & +0.073 \\
BFCLv4 MT & Ministral 3 3B & 0.240 & 0.615 & -0.375 [-0.445,-0.300] & 0.201 & 0.562 & -1.71 & -0.027 \\
BFCLv4 MT & Ministral 3 8B & 0.540 & 0.645 & -0.105 [-0.190,-0.025] & 0.454 & 0.611 & +6.76 & +0.046 \\
BFCLv4 MT & Ministral 3 14B & 0.590 & 0.650 & -0.060 [-0.140,+0.030] & 0.499 & 0.609 & +6.44 & +0.042 \\
BFCLv4 MT & Qwen 2.5 3B & 0.120 & 0.365 & -0.245 [-0.310,-0.180] & 0.082 & 0.311 & -2.09 & -0.023 \\
BFCLv4 MT & Qwen 2.5 7B & 0.220 & 0.600 & -0.380 [-0.460,-0.305] & 0.190 & 0.542 & -3.69 & -0.044 \\
BFCLv4 MT & Qwen 2.5 14B & 0.590 & 0.750 & -0.160 [-0.230,-0.085] & 0.491 & 0.687 & +4.56 & +0.019 \\
BFCLv4 MT & Qwen 2.5 32B & 0.710 & 0.640 & +0.070 [-0.005,+0.140] & 0.624 & 0.607 & +6.81 & +0.056 \\
BFCLv4 MT & Qwen 3.5 4B & 0.635 & 0.845 & -0.210 [-0.270,-0.155] & 0.551 & 0.825 & +8.21 & +0.051 \\
BFCLv4 MT & Qwen 3.5 9B & 0.840 & 0.845 & -0.005 [-0.055,+0.045] & 0.774 & 0.821 & +5.36 & +0.036 \\
BFCLv4 MT & Qwen 3.5 35B-A3B & 0.870 & 0.865 & +0.005 [-0.035,+0.045] & 0.816 & 0.848 & +4.67 & +0.025 \\
\midrule
WebShop & Gemma 4 E4B & 0.486 & 0.652 & -0.166 [-0.212,-0.116] & 0.327 & 0.596 & +13.19 & +0.081 \\
WebShop & Gemma 4 12B & 0.780 & 0.546 & +0.234 [+0.188,+0.280] & 0.648 & 0.511 & +12.45 & +0.092 \\
WebShop & Gemma 4 26B-A4B & 0.764 & 0.468 & +0.296 [+0.246,+0.346] & 0.633 & 0.441 & +13.32 & +0.105 \\
WebShop & Gemma 4 31B & 0.876 & 0.560 & +0.316 [+0.272,+0.360] & 0.816 & 0.534 & +4.49 & +0.036 \\
WebShop & Ministral 3 3B & 0.326 & 0.688 & -0.362 [-0.410,-0.316] & 0.228 & 0.601 & +1.49 & -0.003 \\
WebShop & Ministral 3 8B & 0.788 & 0.546 & +0.242 [+0.194,+0.290] & 0.660 & 0.484 & +8.42 & +0.066 \\
WebShop & Ministral 3 14B & 0.788 & 0.730 & +0.058 [+0.022,+0.096] & 0.673 & 0.678 & +8.11 & +0.057 \\
WebShop & Qwen 2.5 3B & 0.076 & 0.132 & -0.056 [-0.090,-0.022] & 0.043 & 0.109 & +1.28 & +0.009 \\
WebShop & Qwen 2.5 7B & 0.712 & 0.568 & +0.144 [+0.100,+0.192] & 0.620 & 0.519 & +5.50 & +0.041 \\
WebShop & Qwen 2.5 14B & 0.740 & 0.552 & +0.188 [+0.138,+0.236] & 0.619 & 0.470 & +4.97 & +0.039 \\
WebShop & Qwen 2.5 32B & 0.838 & 0.618 & +0.220 [+0.184,+0.260] & 0.763 & 0.579 & +4.69 & +0.037 \\
WebShop & Qwen 3.5 4B & 0.716 & 0.706 & +0.010 [-0.030,+0.050] & 0.574 & 0.651 & +11.21 & +0.072 \\
WebShop & Qwen 3.5 9B & 0.734 & 0.716 & +0.018 [-0.012,+0.048] & 0.614 & 0.647 & +6.56 & +0.041 \\
WebShop & Qwen 3.5 35B-A3B & 0.844 & 0.746 & +0.098 [+0.064,+0.130] & 0.747 & 0.692 & +5.51 & +0.037 \\
\midrule
ACEBench & Gemma 4 E4B & 0.856 & 0.892 & -0.036 [-0.062,-0.012] & 0.771 & 0.885 & +9.96 & +0.086 \\
ACEBench & Gemma 4 12B & 0.905 & 0.903 & +0.003 [-0.023,+0.030] & 0.832 & 0.900 & +9.05 & +0.105 \\
ACEBench & Gemma 4 26B-A4B & 0.938 & 0.903 & +0.035 [+0.010,+0.060] & 0.889 & 0.899 & +5.78 & +0.055 \\
ACEBench & Gemma 4 31B & 0.942 & 0.914 & +0.027 [+0.005,+0.048] & 0.902 & 0.912 & +4.83 & +0.068 \\
ACEBench & Ministral 3 3B & 0.857 & 0.858 & -0.001 [-0.027,+0.025] & 0.780 & 0.831 & +6.42 & +0.053 \\
ACEBench & Ministral 3 8B & 0.932 & 0.918 & +0.014 [-0.006,+0.035] & 0.862 & 0.897 & +6.21 & +0.031 \\
ACEBench & Ministral 3 14B & 0.939 & 0.932 & +0.006 [-0.010,+0.026] & 0.894 & 0.917 & +3.80 & +0.014 \\
ACEBench & Qwen 2.5 3B & 0.800 & 0.694 & +0.106 [+0.071,+0.140] & 0.715 & 0.667 & +7.32 & +0.055 \\
ACEBench & Qwen 2.5 7B & 0.892 & 0.751 & +0.142 [+0.110,+0.169] & 0.808 & 0.740 & +9.26 & +0.076 \\
ACEBench & Qwen 2.5 14B & 0.908 & 0.805 & +0.103 [+0.074,+0.131] & 0.837 & 0.795 & +7.74 & +0.084 \\
ACEBench & Qwen 2.5 32B & 0.944 & 0.879 & +0.065 [+0.039,+0.088] & 0.893 & 0.873 & +5.83 & +0.072 \\
ACEBench & Qwen 3.5 4B & 0.940 & 0.877 & +0.064 [+0.043,+0.084] & 0.907 & 0.858 & +1.81 & -0.005 \\
ACEBench & Qwen 3.5 9B & 0.951 & 0.922 & +0.029 [+0.012,+0.045] & 0.922 & 0.910 & +2.14 & +0.003 \\
ACEBench & Qwen 3.5 35B-A3B & 0.951 & 0.949 & +0.001 [-0.016,+0.018] & 0.935 & 0.939 & +0.74 & -0.026 \\
\bottomrule
\end{tabular}}
\end{table}

\FloatBarrier
\subsection{Decomposition of the Tax by Task Movement}
\label{appendix:tax_decomposition}

Table~\ref{table:taxA_decomposition_allfam} splits $\text{Tax}_A(128)$ into three parts, the contribution of tasks solved only by the base model (lost), that of tasks solved only by the post-trained model (gained), and the difference in curve shape on tasks solved by both (shape). This is an empirical, finite-$K$ counterpart of the sharpening model in Theorem~\ref{thm:proportional-tax}. The lost term is positive in every combination, whereas a large gained term appears mainly for the small backbones on BFCL and WebShop, where post-training repairs the base model. Figure~\ref{fig:task_movement_scatter} shows how individual tasks move for the largest backbone of each family.

\begin{table}[h]
\centering
\caption{\textbf{Decomposition of $\text{Tax}_A(128)$ by task movement for all 42 model-benchmark combinations.} Each task is assigned to a category by which model solves it at least once in 128 rollouts, i.e., \emph{Lost} (only the base model), \emph{Gained} (only the post-trained model), \emph{Both}, or \emph{Neither}. The left block reports the proportion of tasks in each category. The right block splits $\text{Tax}_A(128)$ into the retry value lost on Lost tasks (positive), the retry value recovered on Gained tasks (negative), and the difference in curve shape on tasks solved by both models. The three contributions sum to the total in the last column.}
\label{table:taxA_decomposition_allfam}
\resizebox{0.88\linewidth}{!}{%
\begin{tabular}{llrrrrrrrr}
\toprule
 & & \multicolumn{4}{c}{Task categories (\% of tasks)} & \multicolumn{3}{c}{Contributions to Tax$_A$(128)} & \\
\cmidrule(lr){3-6}\cmidrule(lr){7-9}
Benchmark & Pair & Lost & Gained & Both & Neither & Lost ($+$) & Gained ($-$) & Shape & Tax$_A$(128) \\
\midrule
BFCLv4 MT & Gemma 4 E4B & 8.0 & 19.5 & 19.0 & 53.5 & +3.88 & -2.25 & +4.97 & +6.59 \\
BFCLv4 MT & Gemma 4 12B & 4.0 & 31.0 & 43.0 & 22.0 & +1.73 & -0.21 & +9.07 & +10.59 \\
BFCLv4 MT & Gemma 4 26B-A4B & 4.5 & 23.0 & 58.0 & 14.5 & +0.82 & -0.81 & +10.17 & +10.18 \\
BFCLv4 MT & Gemma 4 31B & 10.5 & 3.5 & 80.0 & 6.0 & +2.43 & -0.05 & +5.73 & +8.11 \\
BFCLv4 MT & Ministral 3 3B & 5.0 & 42.5 & 19.0 & 33.5 & +1.73 & -5.41 & +1.96 & -1.71 \\
BFCLv4 MT & Ministral 3 8B & 12.5 & 23.0 & 41.5 & 23.0 & +3.44 & -1.78 & +5.09 & +6.76 \\
BFCLv4 MT & Ministral 3 14B & 16.0 & 22.0 & 43.0 & 19.0 & +4.26 & -2.37 & +4.54 & +6.44 \\
BFCLv4 MT & Qwen 2.5 3B & 1.5 & 26.0 & 10.5 & 62.0 & +0.74 & -6.18 & +3.35 & -2.09 \\
BFCLv4 MT & Qwen 2.5 7B & 3.0 & 41.0 & 19.0 & 37.0 & +1.06 & -6.49 & +1.74 & -3.69 \\
BFCLv4 MT & Qwen 2.5 14B & 6.5 & 22.5 & 52.5 & 18.5 & +1.34 & -2.91 & +6.13 & +4.56 \\
BFCLv4 MT & Qwen 2.5 32B & 17.5 & 10.5 & 53.5 & 18.5 & +3.08 & -0.98 & +4.71 & +6.81 \\
BFCLv4 MT & Qwen 3.5 4B & 1.5 & 22.5 & 62.0 & 14.0 & +0.52 & -1.38 & +9.08 & +8.21 \\
BFCLv4 MT & Qwen 3.5 9B & 6.0 & 6.5 & 78.0 & 9.5 & +2.00 & -0.85 & +4.21 & +5.36 \\
BFCLv4 MT & Qwen 3.5 35B-A3B & 4.0 & 3.5 & 83.0 & 9.5 & +0.81 & -0.04 & +3.90 & +4.67 \\
\midrule
WebShop & Gemma 4 E4B & 10.0 & 26.6 & 38.6 & 24.8 & +4.73 & -4.00 & +12.46 & +13.19 \\
WebShop & Gemma 4 12B & 29.2 & 5.8 & 48.8 & 16.2 & +8.32 & -0.51 & +4.64 & +12.45 \\
WebShop & Gemma 4 26B-A4B & 34.2 & 4.6 & 42.2 & 19.0 & +9.20 & -0.42 & +4.54 & +13.32 \\
WebShop & Gemma 4 31B & 33.0 & 1.4 & 54.6 & 11.0 & +4.38 & -0.38 & +0.48 & +4.49 \\
WebShop & Ministral 3 3B & 3.0 & 39.2 & 29.6 & 28.2 & +1.52 & -8.76 & +8.73 & +1.49 \\
WebShop & Ministral 3 8B & 28.8 & 4.6 & 50.0 & 16.6 & +9.04 & -0.58 & -0.04 & +8.42 \\
WebShop & Ministral 3 14B & 11.8 & 6.0 & 67.0 & 15.2 & +4.73 & -1.12 & +4.51 & +8.11 \\
WebShop & Qwen 2.5 3B & 5.4 & 11.0 & 2.2 & 81.4 & +3.16 & -2.48 & +0.59 & +1.28 \\
WebShop & Qwen 2.5 7B & 21.4 & 7.0 & 49.8 & 21.8 & +5.61 & -1.78 & +1.67 & +5.50 \\
WebShop & Qwen 2.5 14B & 26.2 & 7.4 & 47.8 & 18.6 & +6.80 & -1.67 & -0.17 & +4.97 \\
WebShop & Qwen 2.5 32B & 23.0 & 1.0 & 60.8 & 15.2 & +5.62 & -0.25 & -0.69 & +4.69 \\
WebShop & Qwen 3.5 4B & 11.0 & 10.0 & 60.6 & 18.4 & +5.13 & -2.31 & +8.40 & +11.21 \\
WebShop & Qwen 3.5 9B & 7.0 & 5.2 & 66.4 & 21.4 & +2.91 & -1.42 & +5.07 & +6.56 \\
WebShop & Qwen 3.5 35B-A3B & 12.2 & 2.4 & 72.2 & 13.2 & +4.43 & -0.90 & +1.98 & +5.51 \\
\midrule
ACEBench & Gemma 4 E4B & 5.1 & 8.7 & 80.5 & 5.7 & +1.10 & -0.12 & +8.98 & +9.96 \\
ACEBench & Gemma 4 12B & 7.3 & 7.0 & 83.2 & 2.5 & +1.33 & -0.11 & +7.82 & +9.05 \\
ACEBench & Gemma 4 26B-A4B & 8.1 & 4.5 & 85.7 & 1.7 & +1.34 & -0.03 & +4.47 & +5.78 \\
ACEBench & Gemma 4 31B & 6.5 & 3.8 & 87.7 & 2.1 & +0.72 & -0.11 & +4.22 & +4.83 \\
ACEBench & Ministral 3 3B & 7.1 & 7.3 & 78.6 & 7.0 & +1.83 & -0.80 & +5.39 & +6.42 \\
ACEBench & Ministral 3 8B & 5.3 & 3.9 & 87.9 & 2.9 & +1.60 & -0.42 & +5.03 & +6.21 \\
ACEBench & Ministral 3 14B & 3.6 & 3.0 & 90.3 & 3.1 & +0.81 & -0.37 & +3.37 & +3.80 \\
ACEBench & Qwen 2.5 3B & 16.9 & 6.2 & 63.1 & 13.8 & +4.42 & -0.88 & +3.78 & +7.32 \\
ACEBench & Qwen 2.5 7B & 17.3 & 3.1 & 71.9 & 7.7 & +4.25 & -0.23 & +5.24 & +9.26 \\
ACEBench & Qwen 2.5 14B & 13.0 & 2.7 & 77.8 & 6.5 & +3.01 & -0.22 & +4.95 & +7.74 \\
ACEBench & Qwen 2.5 32B & 9.6 & 3.1 & 84.8 & 2.5 & +1.66 & -0.11 & +4.28 & +5.83 \\
ACEBench & Qwen 3.5 4B & 8.4 & 2.1 & 85.6 & 3.9 & +1.25 & -0.33 & +0.89 & +1.81 \\
ACEBench & Qwen 3.5 9B & 4.5 & 1.7 & 90.5 & 3.2 & +0.62 & -0.35 & +1.87 & +2.14 \\
ACEBench & Qwen 3.5 35B-A3B & 2.6 & 2.5 & 92.5 & 2.5 & +0.39 & -0.25 & +0.59 & +0.74 \\
\bottomrule
\end{tabular}}
\end{table}

\FloatBarrier
\begin{figure}[p]
\centering
\rotatebox{90}{\makebox[0.263\linewidth][c]{\scriptsize Four models avg.}}%
\includegraphics[width=0.79\linewidth]{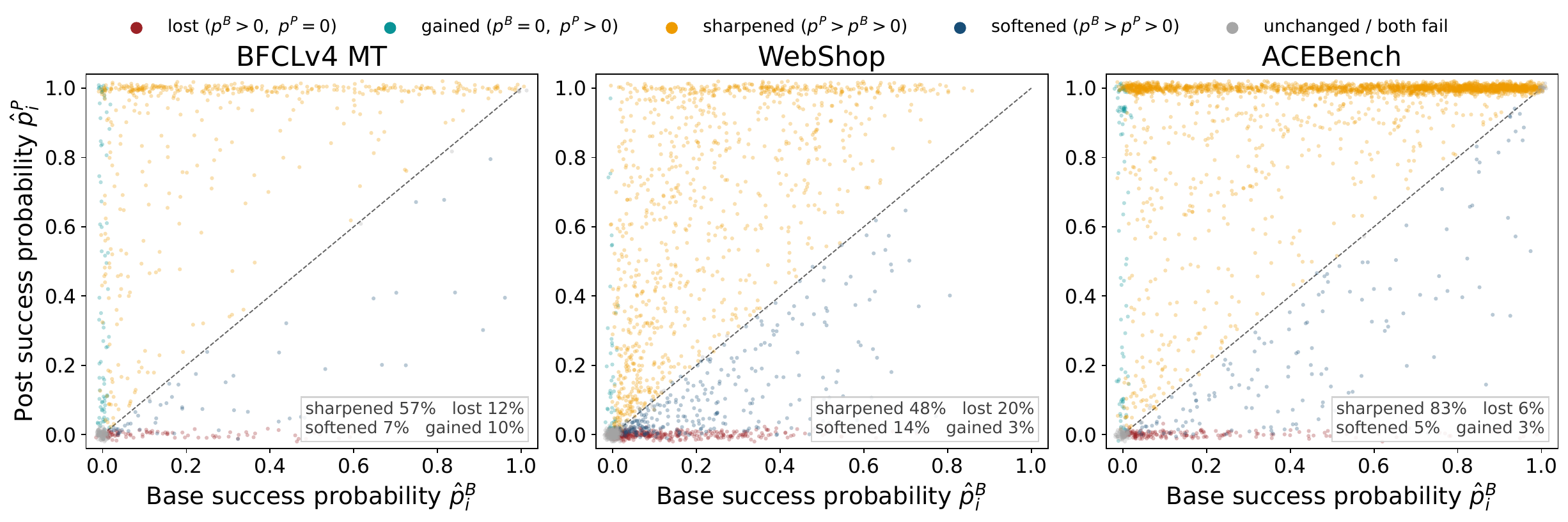}\\[2pt]
\rotatebox{90}{\makebox[0.227\linewidth][c]{\scriptsize\texttt{gemma-4-31B}}}%
\includegraphics[width=0.79\linewidth,trim=0 0 0 54bp,clip]{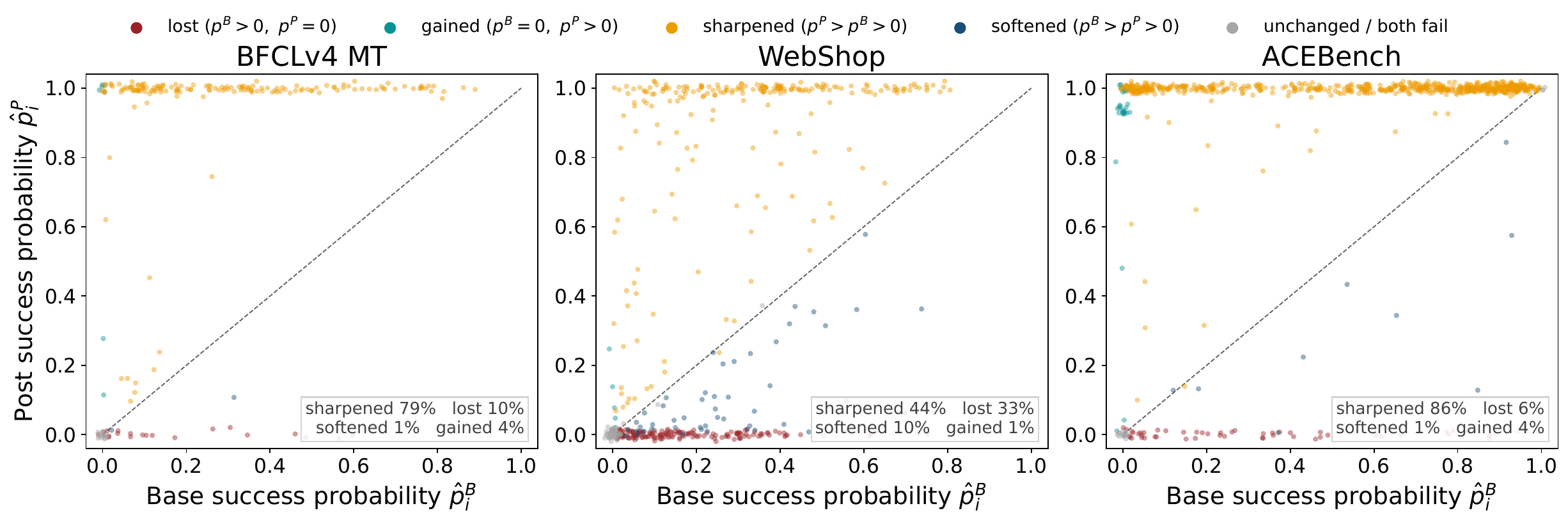}\\[2pt]
\rotatebox{90}{\makebox[0.227\linewidth][c]{\scriptsize\texttt{Ministral-3-14B}}}%
\includegraphics[width=0.79\linewidth,trim=0 0 0 54bp,clip]{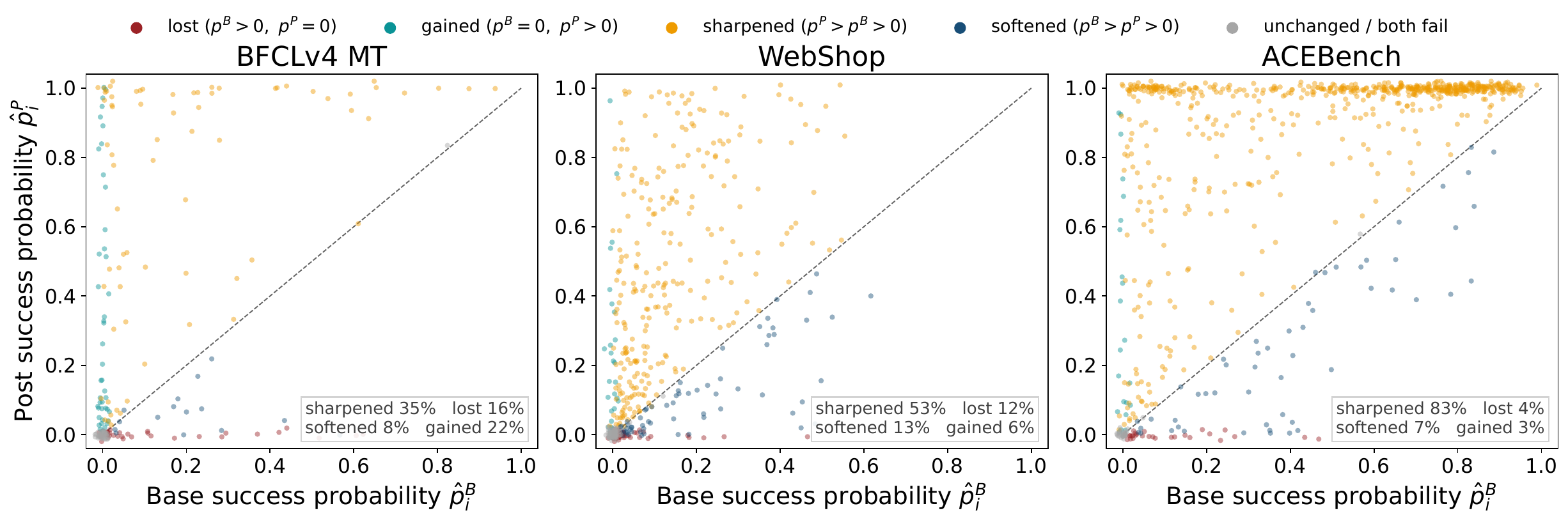}\\[2pt]
\rotatebox{90}{\makebox[0.227\linewidth][c]{\scriptsize\texttt{Qwen2.5-32B}}}%
\includegraphics[width=0.79\linewidth,trim=0 0 0 54bp,clip]{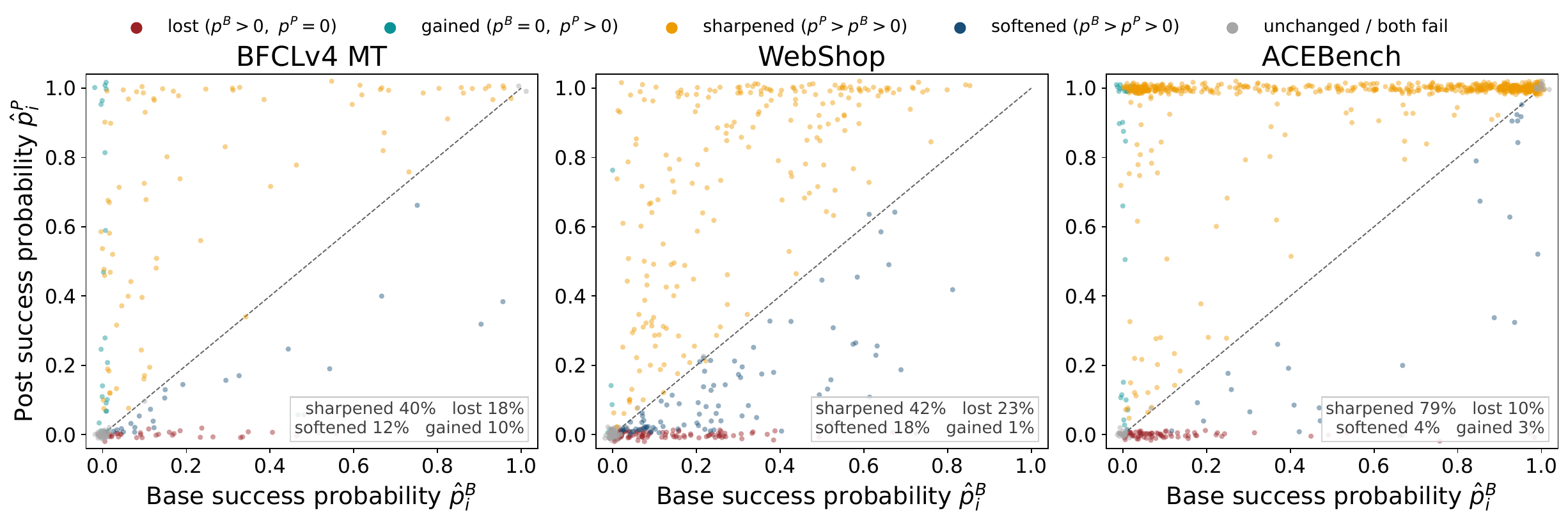}\\[2pt]
\rotatebox{90}{\makebox[0.227\linewidth][c]{\scriptsize\texttt{Qwen3.5-35B-A3B}}}%
\includegraphics[width=0.79\linewidth,trim=0 0 0 54bp,clip]{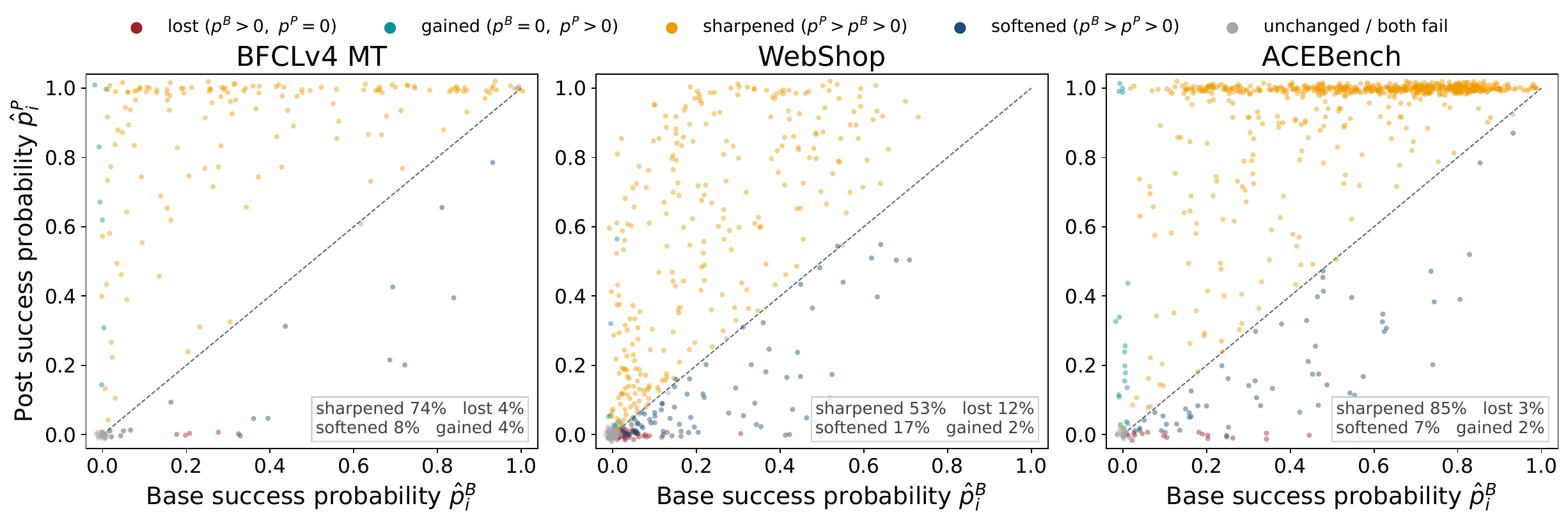}
\caption{\textbf{Per-task success rate movement from the base to the post-trained model} for the largest backbone of each family on BFCLv4 MT, WebShop, and ACEBench. Each point is one task, placed by its empirical success rate over 128 rollouts under the base model ($\hat{p}^B_i$, $x$-axis) and the post-trained model ($\hat{p}^P_i$, $y$-axis). Colors mark how each task moves, i.e., lost ($p^B > 0$, $p^P = 0$), gained ($p^B = 0$, $p^P > 0$), sharpened ($p^P > p^B > 0$), softened ($p^B > p^P > 0$), or unchanged, and the inset of each panel reports the proportion of each type among all tasks. The top row pools the four models, and each remaining row shows one model. Most tasks move far from the diagonal. Many sharpened tasks reach the top edge, while the lost tasks along the bottom edge show the coverage lost by post-training.}
\label{fig:task_movement_scatter}
\end{figure}

\FloatBarrier
\subsection{Which Tasks Pay the Tax?}
\label{appendix:tax_extremes}

\begin{figure}[ht]
    \centering
    \includegraphics[width=0.85\linewidth]{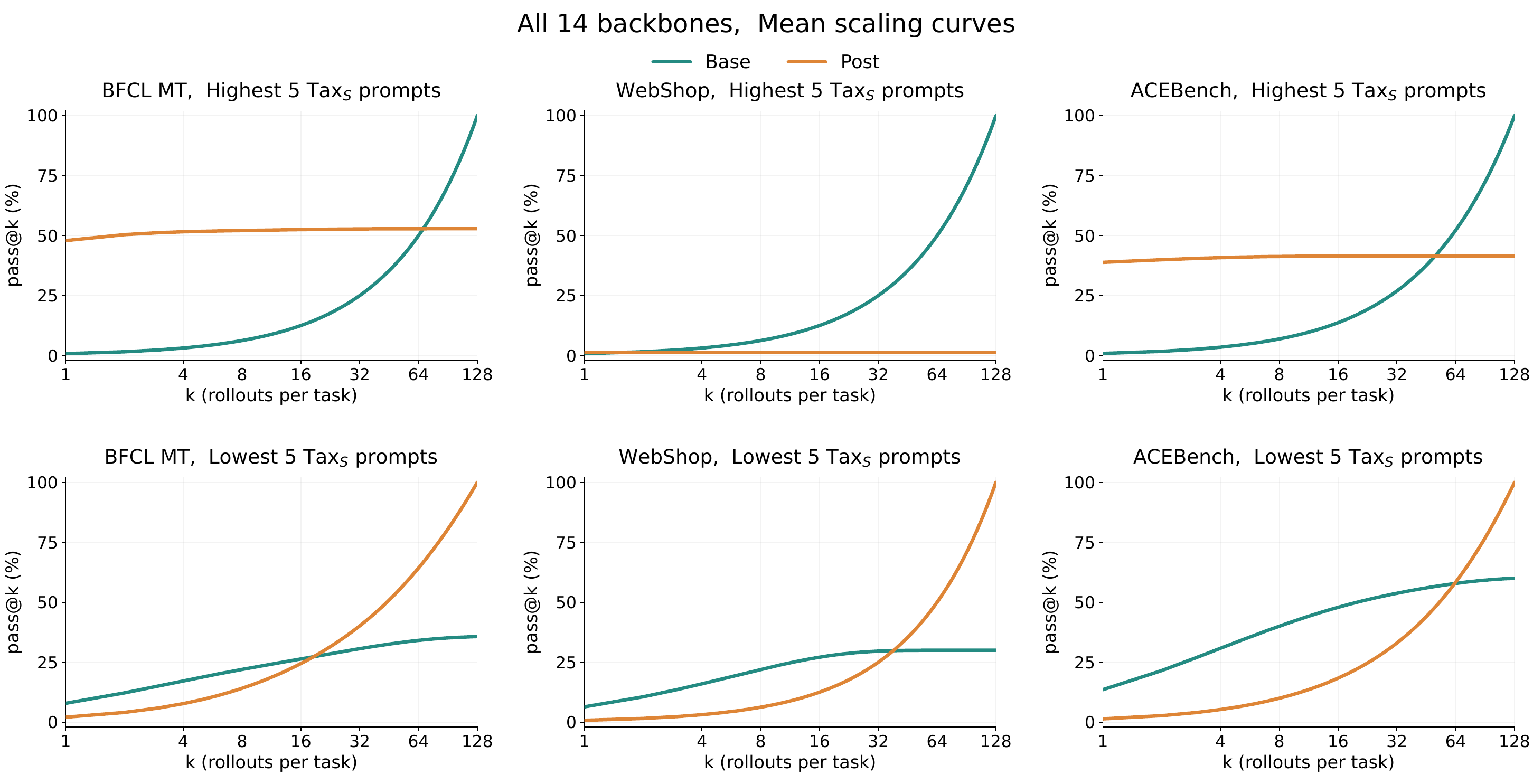}
    \caption{\textbf{Mean scaling curves of the tasks with the highest and lowest tax contributions, averaged over all 14 backbones.} For each benchmark and backbone, we select the five tasks with the highest (top row) and lowest (bottom row) signed contributions to $\text{Tax}_{S}(128)$ in Eq.~\eqref{eq:tax_contribution}. Each panel shows the mean pass@$k$ curve of the base (teal) and post-trained (orange) models over the selected tasks of all 14 backbones (70 tasks per panel).}
    \label{fig:tax_extremes_mean}
\end{figure}

The calibrated tax of a benchmark-backbone pair is an average of per-task contributions. Let $a_i(K) = \sum_{k=1}^{K-1}\big[\text{pass}@K_i - \text{pass}@k_i\big]$ be the raw scalability of task $i$. Then, Eq.~\eqref{eq:norm_scalability} gives
\begin{equation}
\text{Tax}_{S}(K) = \frac{1}{N}\sum_{i=1}^{N}\left[\frac{a^{B}_i(K)}{(K-1)\,(1-\text{pass}@1_{B})} - \frac{a^{P}_i(K)}{(K-1)\,(1-\text{pass}@1_{P})}\right],
\label{eq:tax_contribution}
\end{equation}
where each policy is normalized by its benchmark-level accuracy headroom. To see on which tasks the tax is charged, we rank the tasks of each benchmark-backbone pair by their signed contribution in Eq.~\eqref{eq:tax_contribution} at $k=128$ and plot the scaling curves of the five highest and five lowest tasks under both policies. Figure~\ref{fig:tax_extremes_mean} averages these curves over all 14 backbones, and Figures~\ref{fig:tax_extremes_gemma}--\ref{fig:tax_extremes_qwen35} show the individual tasks for each family.

The two extremes look alike across families and scales. On the highest-tax tasks, the post-trained policy has become deterministic, so its curve is flat at 100\% or at 0\% for every budget and gains nothing from retries. The base model instead succeeds only rarely, often once in 128 rollouts, so its curve stays low for small $k$ and rises to 100\% only near the full budget (Figure~\ref{fig:tax_extremes_mean}, top). This is Proposition~\ref{thm:redeemed-depth} at the level of a single task, since the tax is charged where post-training removes a rare success that retries could recover. The lowest-tax tasks show the opposite pattern (Figure~\ref{fig:tax_extremes_mean}, bottom). There, the post-trained policy holds the rare success, and the base model either never succeeds or saturates within a few rollouts. The per-family figures show how this depends on scale to some extent.

\begin{figure}[p]
    \centering
    \includegraphics[width=\textwidth,height=.72\textheight,keepaspectratio]{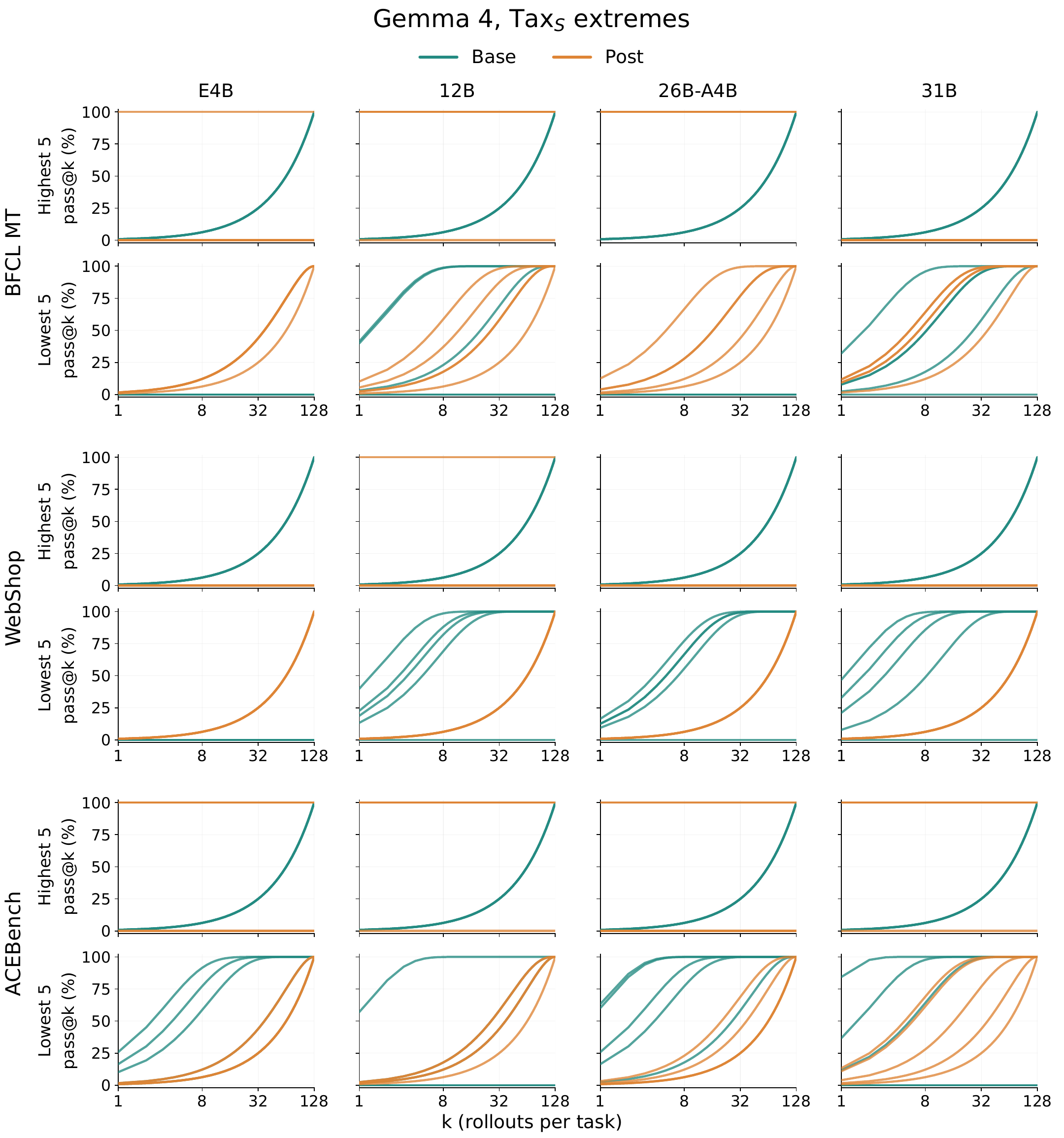}
    \caption{\textbf{Scaling curves of the tasks with the highest and lowest tax contributions for Gemma-4} (E4B, 12B, 26B-A4B, and 31B from left to right). For each benchmark (BFCL MT, WebShop, and ACEBench from top to bottom), the upper row shows the five tasks with the highest signed contributions to $\text{Tax}_{S}(128)$ (Eq.~\eqref{eq:tax_contribution}) and the lower row the five with the lowest, selected separately for each backbone. Each panel shows the pass@$k$ curves of the base (teal) and post-trained (orange) models on the same five tasks, computed from 128 rollouts per task without averaging.}
    \label{fig:tax_extremes_gemma}
\end{figure}

\begin{figure}[p]
    \centering
    \includegraphics[width=\textwidth,height=.72\textheight,keepaspectratio]{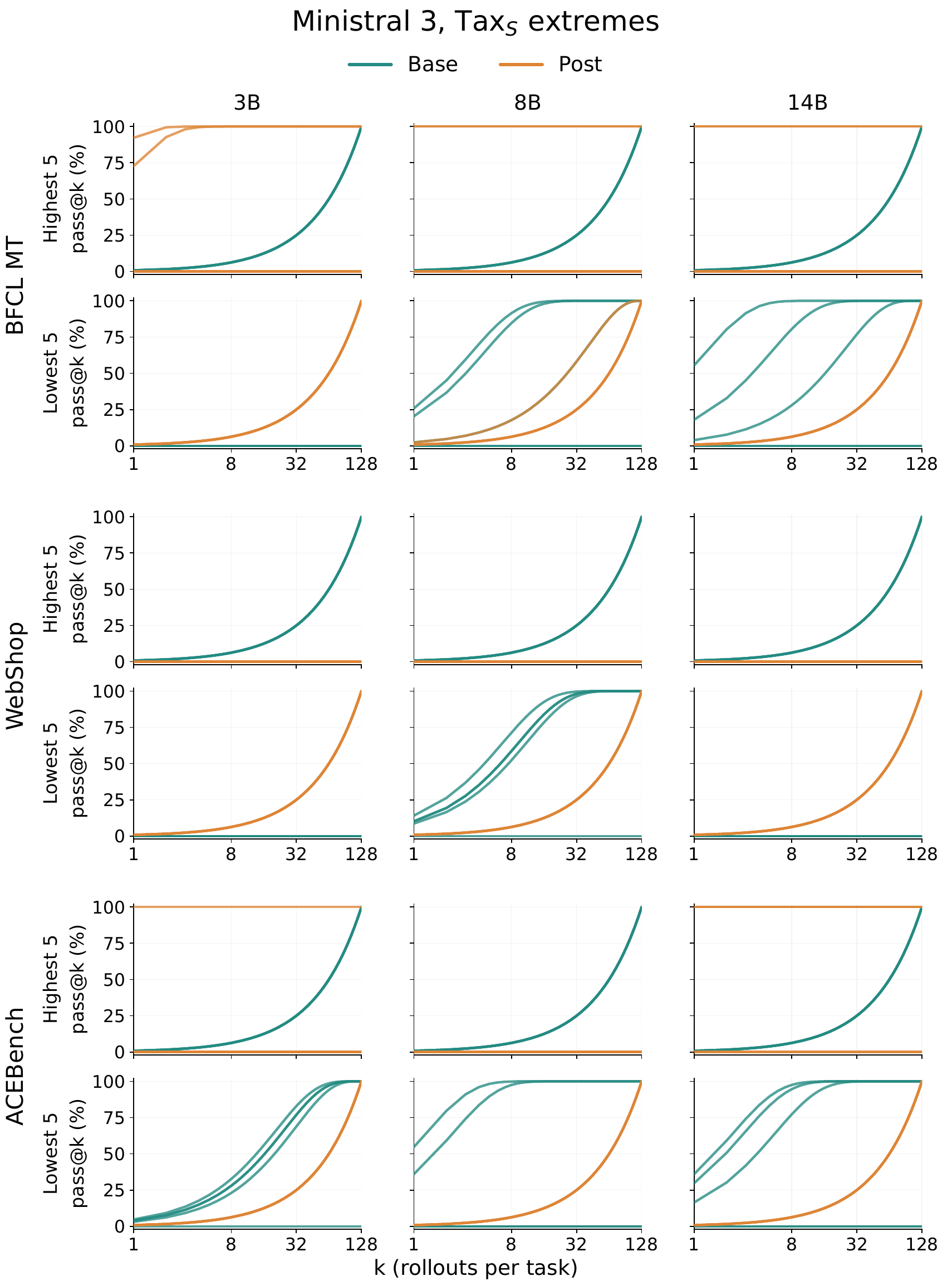}
    \caption{\textbf{Scaling curves of the tasks with the highest and lowest tax contributions for Ministral-3} (3B, 8B, and 14B from left to right). For each benchmark (BFCL MT, WebShop, and ACEBench from top to bottom), the upper row shows the five tasks with the highest signed contributions to $\text{Tax}_{S}(128)$ (Eq.~\eqref{eq:tax_contribution}) and the lower row the five with the lowest, selected separately for each backbone. Each panel shows the pass@$k$ curves of the base (teal) and post-trained (orange) models on the same five tasks, computed from 128 rollouts per task.}
    \label{fig:tax_extremes_ministral}
\end{figure}

\begin{figure}[p]
    \centering
    \includegraphics[width=\textwidth,height=.72\textheight,keepaspectratio]{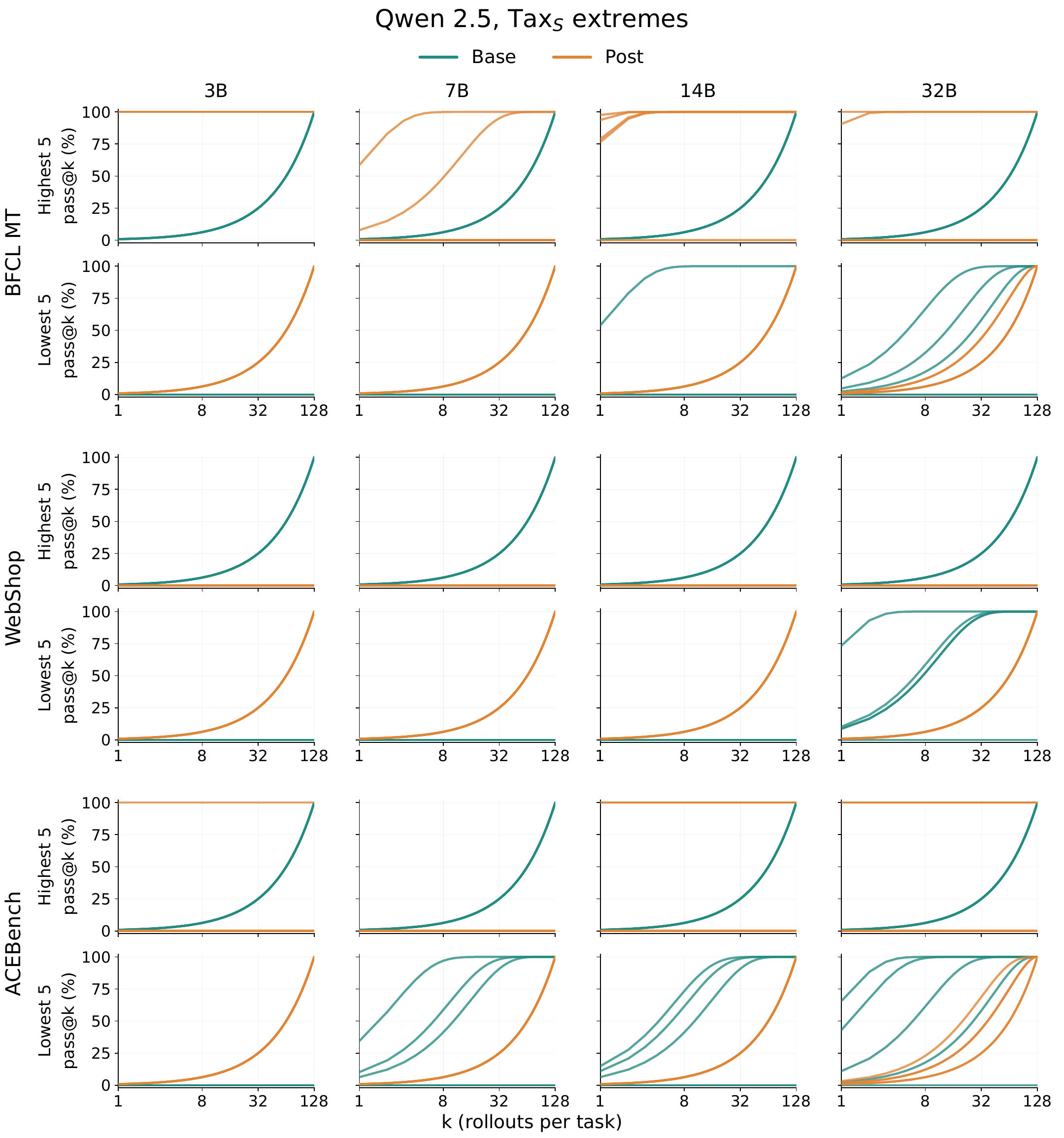}
    \caption{\textbf{Scaling curves of the tasks with the highest and lowest tax contributions for Qwen2.5} (3B, 7B, 14B, and 32B from left to right). For each benchmark (BFCL MT, WebShop, and ACEBench from top to bottom), the upper row shows the five tasks with the highest signed contributions to $\text{Tax}_{S}(128)$ (Eq.~\eqref{eq:tax_contribution}) and the lower row the five with the lowest, selected separately for each backbone. Each panel shows the pass@$k$ curves of the base (teal) and post-trained (orange) models on the same five tasks, computed from 128 rollouts per task.}
    \label{fig:tax_extremes_qwen25}
\end{figure}

\begin{figure}[p]
    \centering
    \includegraphics[width=\textwidth,height=.72\textheight,keepaspectratio]{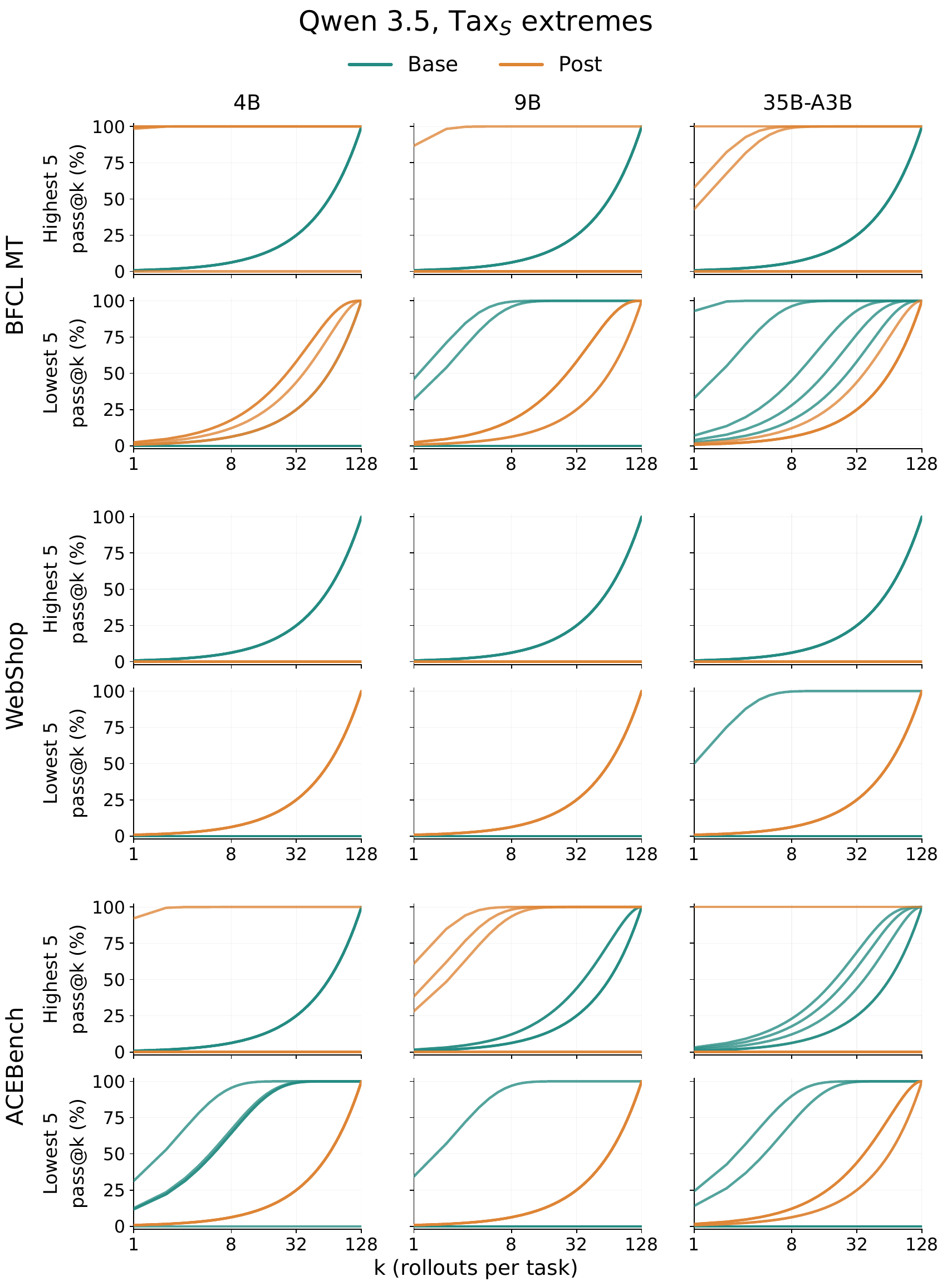}
    \caption{\textbf{Scaling curves of the tasks with the highest and lowest tax contributions for Qwen3.5} (4B, 9B, and 35B-A3B from left to right). For each benchmark (BFCL MT, WebShop, and ACEBench from top to bottom), the upper row shows the five tasks with the highest signed contributions to $\text{Tax}_{S}(128)$ (Eq.~\eqref{eq:tax_contribution}) and the lower row the five with the lowest, selected separately for each backbone. Each panel shows the pass@$k$ curves of the base (teal) and post-trained (orange) models on the same five tasks, computed from 128 rollouts per task.}
    \label{fig:tax_extremes_qwen35}
\end{figure}

\FloatBarrier
\subsection{Tax-Guided Routing between the Base and the Post-trained Policy}
\label{appendix:routing}

\begin{figure}[htb]
    \centering
    \includegraphics[width=\linewidth]{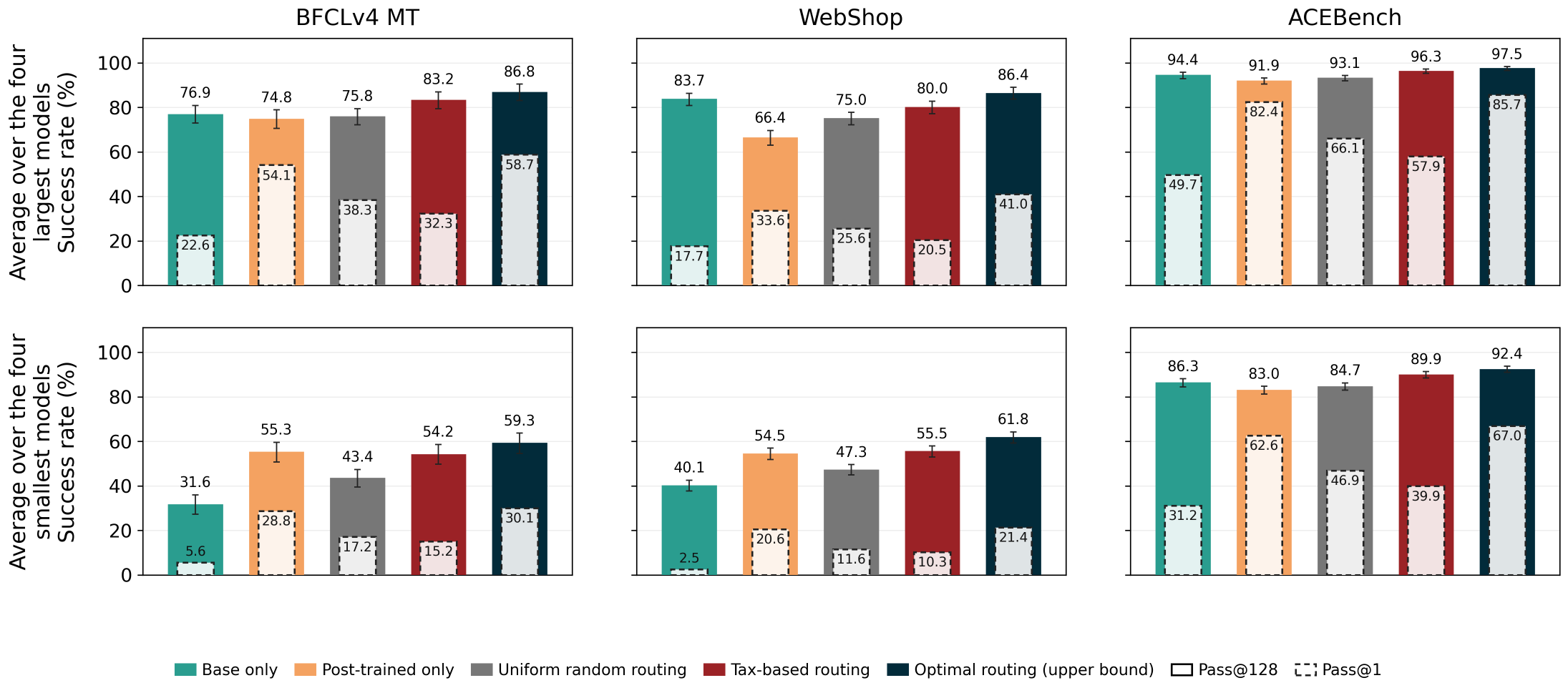}
    \caption{\textbf{Tax-guided routing between the base and post-trained policies.} Success rates averaged over the four largest (top) and the four smallest (bottom) backbones on BFCLv4 MT, WebShop, and ACEBench. Each task is served by one policy chosen by one of five rules, i.e., base only (teal), post-trained only (orange), uniform random routing (gray), tax-based routing (red), and optimal routing (navy). Tax-based routing chooses the policy for each task based on its $\text{Tax}_{S}(32)$ estimated from 32 early pilot rollouts. Optimal routing is an oracle selector that assigns each task to the better policy on the full 128 rollouts, which upper-bounds any router. Filled bars show pass@$128$ and dashed inset bars show pass@$1$ under the same assignments, both computed from all 128 rollouts of the selected policy including the pilot rollouts. Error bars show the bootstrap uncertainty of pass@$128$ (2{,}000 resampled tasks).}
    \label{fig:routing}
\end{figure}

\noindent\textbf{Setup.} If the tax tells us on which tasks retries pay off, a cheap estimate of it could also decide which of the two policies should receive the test-time budget for each task. For each base/post-trained pair and benchmark, we compare five ways of assigning each task to one policy. Besides the two single policies, we consider uniform random routing, which assigns each task to either policy with equal probability, tax-based routing, which chooses the policy for each task based on its calibrated tax $\text{Tax}_{S}(32)$ estimated from 32 early pilot rollouts, and optimal routing, a hindsight oracle that assigns each task to the better policy on the full 128 rollouts. The reported pass@$128$ uses all 128 rollouts of the selected policy including its pilot rollouts, so the pilot rollouts spent on the other policy are the overhead of routing. Figure~\ref{fig:routing} averages the results over the four largest and the four smallest backbones.

\noindent\textbf{Results.} Tax-based routing combines the strengths of the two policies. Its pass@$128$ exceeds that of the post-trained model in five of the six panels, by up to 13.6 points (WebShop, largest backbones), and the only exception (BFCL, smallest backbones) is within 1.1 points. At the same time, its pass@$1$ exceeds that of the base model in every panel. In four of the six panels, it reaches a higher pass@$128$ than both single policies, and in the other two, where one policy is far better than the other, it stays within a few points of the better one. It also beats uniform random routing in every panel, so the gain comes from the tax signal rather than from simply mixing the two policies, and it stays close to the performance of oracle routing. \textbf{A tax estimate from a small pilot thus recovers most of the coverage that post-training gives up, and often more, without falling back to the low single-shot accuracy of the base model.} Since its pass@$1$ lies between those of the two single policies, tax-based routing is best suited to large test-time budgets, in line with the budget dependence discussed in \S\ref{section:observation}.

\FloatBarrier
\section{Additional Results on PTGS}
\label{appendix:results_ptgs}

This section gives additional results for \S\ref{section:solution}. Here, we show the training dynamics behind Table~\ref{tab:ptgs-combined} (\S\ref{appendix:deferred_results}), an ablation on the temperature spread $\tau$ (\S\ref{appendix:ptgs_tau_ablation}), a comparison with fixed rollout temperatures and truncated decoding rules (\S\ref{appendix:ptgs_sampler_ablation}), and an inference-time variant of PTGS (\S\ref{appendix:ptgs_inference}).

\FloatBarrier
\subsection{Training Dynamics of PTGS}
\label{appendix:deferred_results}

\begin{figure*}[thb]
     \centering
     \includegraphics[width=0.98\textwidth]{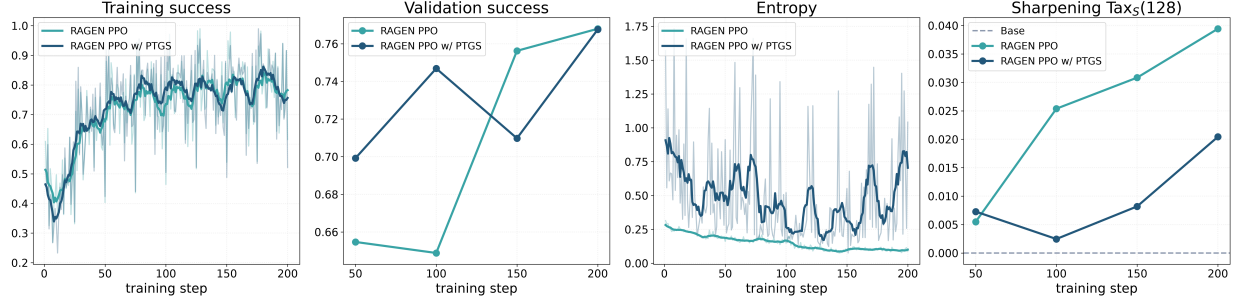}
     \vspace{-0.5em}
     \caption{\textbf{Training dynamics of PPO and PPO with PTGS on FrozenLake.} The panels show, from left to right, training success, validation success, average entropy of the output token distribution (thin lines for raw values and thick lines for the moving average), and $\text{Tax}_{S}(128)$ of the checkpoints at steps 50, 100, 150, and 200. The two methods have nearly identical training success and end at the same validation success, but PTGS keeps a much higher entropy throughout training and pays a smaller tax at every checkpoint after step 50.}
     \vspace{-0.5em}
     \label{fig:ptgs-training-fl}
\end{figure*}

\begin{figure*}[thb]
     \centering
     \includegraphics[width=0.98\textwidth]{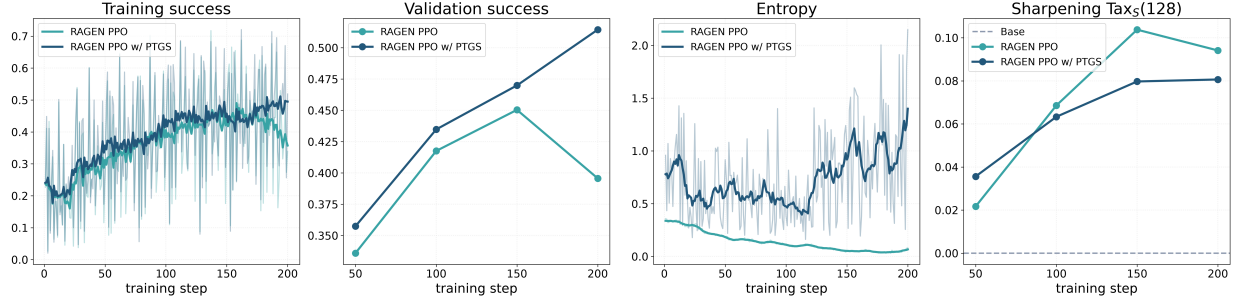}
     \vspace{-0.5em}
     \caption{\textbf{Training dynamics of PPO and PPO with PTGS on Sokoban.} The panels show, from left to right, training success, validation success, average entropy of the output token distribution (thin lines for raw values and thick lines for the moving average), and $\text{Tax}_{S}(128)$ of the checkpoints at steps 50, 100, 150, and 200. PTGS reaches higher validation success at every checkpoint and keeps improving until step 200, while PPO drops after step 150. The entropy of PPO decays toward zero while that of PTGS stays high, and PTGS pays a smaller tax than PPO from step 100 on.}
     \vspace{-0.5em}
     \label{fig:ptgs-training-sokoban}
\end{figure*}

\noindent\textbf{Training dynamics of PTGS.} Figures~\ref{fig:ptgs-training-fl} and~\ref{fig:ptgs-training-sokoban} show the training progress of PPO and PPO with PTGS behind the final results in Table~\ref{tab:ptgs-combined}. In both environments, the two methods have almost the same training success, so PTGS does not make the training tasks easier or harder on average. The other panels, however, show clear differences. First, the entropy of PPO decreases steadily toward zero, which is the usual entropy collapse of fixed-temperature RL, whereas PTGS keeps the entropy several times higher throughout training. The entropy of PTGS also rises repeatedly, which is consistent with the sampler heating the prompts that the policy keeps failing. Second, the tax of PPO grows over training in both environments, but PTGS pays a smaller tax at most checkpoints, and on Sokoban its tax stops growing after step 150. Third, PTGS matches PPO in final validation success on FrozenLake and clearly outperforms it on Sokoban, where PPO degrades after step 150 while PTGS keeps improving. Overall, \textbf{the per-prompt adaptive temperature of PTGS stabilizes training and yields a final policy that is at least as accurate, more exploratory, and less taxed}.

\FloatBarrier
\subsection{Ablation on the Temperature Spread \texorpdfstring{$\tau$}{tau} of PTGS}
\label{appendix:ptgs_tau_ablation}

\begin{figure}[thb]
     \centering
     {\small\textbf{(a) FrozenLake}}\\[2pt]
     \includegraphics[width=0.82\linewidth]{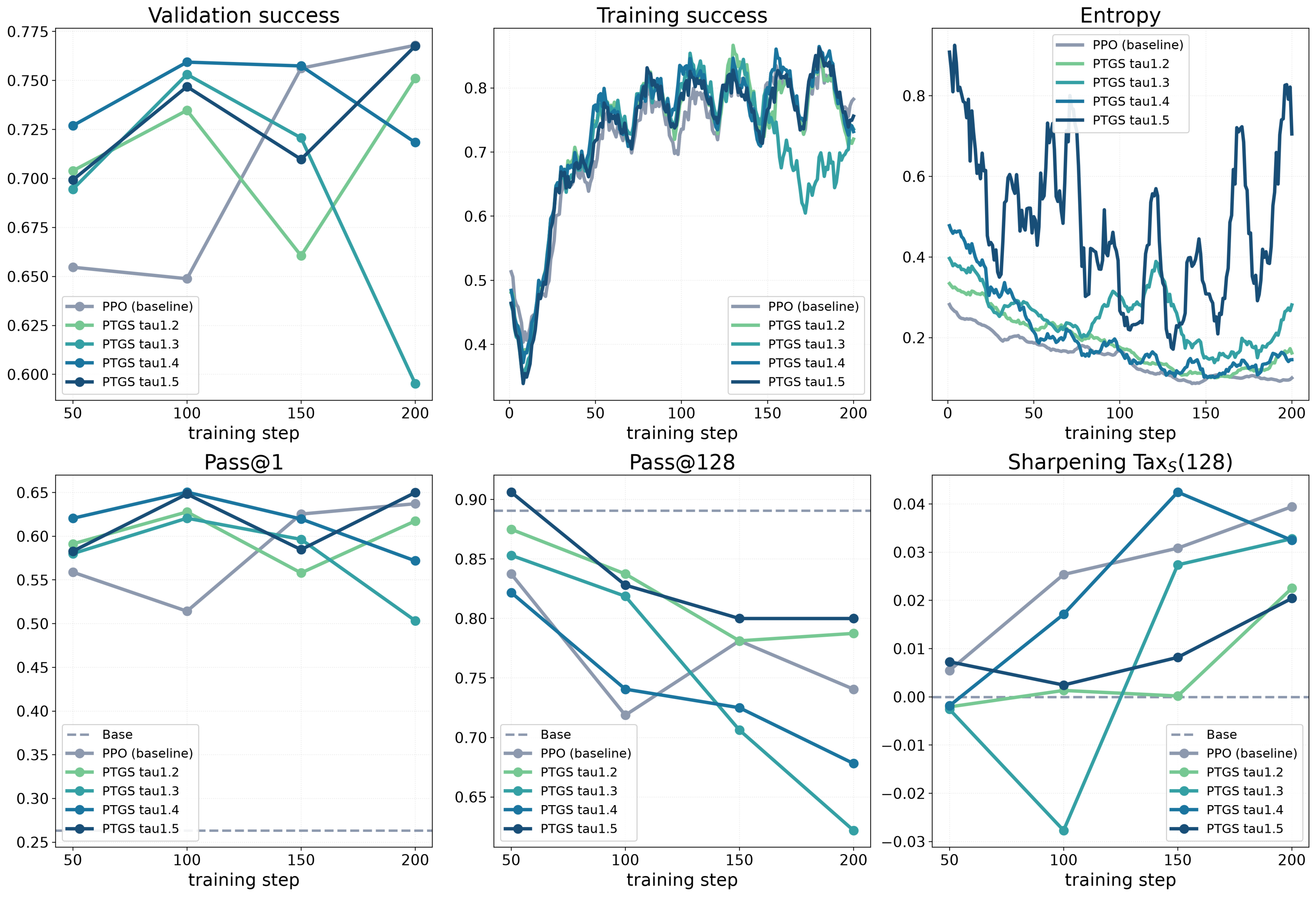}\\[6pt]
     {\small\textbf{(b) Sokoban}}\\[2pt]
     \includegraphics[width=0.82\linewidth]{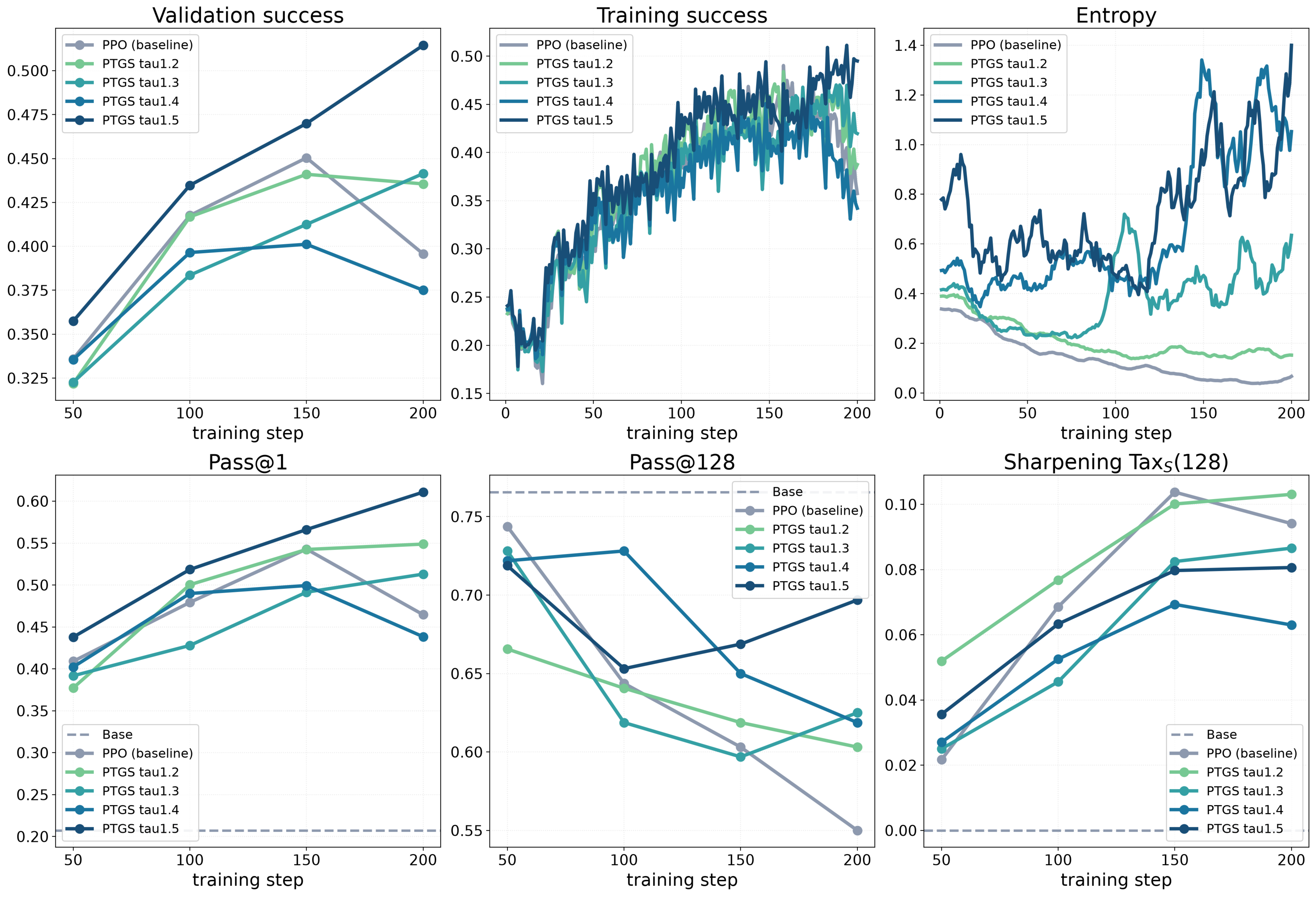}
     \vspace{-0.5em}
     \caption{\textbf{Ablation on the temperature spread $\tau$ of PTGS on FrozenLake (a) and Sokoban (b).} We compare PPO with PTGS for $\tau \in \{1.2, 1.3, 1.4, 1.5\}$ against the PPO baseline. In each block, the top row shows validation success, training success, and average entropy of the output token distribution over training, and the bottom row shows pass@$1$, pass@$128$, and $\text{Tax}_{S}(128)$ of the intermediate checkpoints, with dashed lines for the base model. A larger $\tau$ keeps a higher entropy in both environments. $\tau = 1.5$ achieves the best final pass@$1$ and pass@$128$ in both environments and the smallest tax on FrozenLake. On Sokoban, the smallest temperature $\tau = 1.2$ only slightly raises the entropy above PPO and pays a similar tax.}
     \label{fig:ptgs-tau}
\end{figure}

The spread $\tau$ sets the temperature range $[1/\tau, \tau]$ of PTGS. As Figure~\ref{fig:ptgs-tau} shows, a larger $\tau$ keeps a higher entropy throughout training while barely affecting training success. The spread needed depends on the environment: on FrozenLake, even $\tau=1.2$ improves pass@$128$ and reduces the tax over PPO, whereas on the harder Sokoban, $\tau=1.2$ barely raises the entropy and only $\tau\ge1.3$ improves pass@$128$. We use $\tau=1.5$ for PPO in \S\ref{section:solution}, since it gives the best accuracy and coverage in both environments.

\FloatBarrier
\subsection{Fixed Rollout-Temperature and Decoding Ablations of the PPO Baseline}
\label{appendix:ptgs_sampler_ablation}

\begin{table}[!htb]
\centering
\caption{\textbf{Fixed rollout-temperature ablation of PPO, compared with PTGS}. Each PPO row samples its training rollouts at a fixed temperature $T$, where $T{=}1.0$ (shaded) is the PPO baseline of the main text. Values are the mean over five runs $\pm$ the 95\% confidence interval.}
\label{tab:temp-ablation-ci}
\small
\setlength{\tabcolsep}{4pt}
\resizebox{\linewidth}{!}{%
\begin{tabular}{l cccc | cccc}
\toprule
& \multicolumn{4}{c}{Sokoban} & \multicolumn{4}{|c}{FrozenLake} \\
\cmidrule(lr){2-5}\cmidrule(lr){6-9}
Method & pass@$1$ & pass@$128$ & $\text{pass}^{128}$ & $\text{Tax}_{S}(128)$ & pass@$1$ & pass@$128$ & $\text{pass}^{128}$ & $\text{Tax}_{S}(128)$ \\
\midrule
Qwen2.5-7B-Instruct (base) & 20.7 & 76.6 & 0.0 & -- & 26.3 & 89.1 & 0.0 & -- \\
\midrule
PPO, $T$\,=\,0.5 & 45.6 $\pm$ 7.0 & 62.8 $\pm$ 10.7 & 26.2 $\pm$ 8.6 & \textbf{0.073 $\pm$ 0.025} & 58.7 $\pm$ 15.0 & 76.9 $\pm$ 12.1 & 41.2 $\pm$ 26.5 & \textbf{0.004 $\pm$ 0.053} \\
PPO, $T$\,=\,0.7 & 40.2 $\pm$ 26.8 & 51.9 $\pm$ 33.6 & 30.9 $\pm$ 23.1 & 0.075 $\pm$ 0.049 & 63.7 $\pm$ 7.1 & 70.3 $\pm$ 4.9 & \textbf{54.7 $\pm$ 11.4} & 0.046 $\pm$ 0.032 \\
\rowcolor{gray!10}
PPO, $T$\,=\,1.0 (default) & 46.5 $\pm$ 14.7 & 55.0 $\pm$ 11.3 & 36.2 $\pm$ 13.1 & 0.094 $\pm$ 0.016 & 63.7 $\pm$ 5.9 & 74.1 $\pm$ 7.2 & 50.9 $\pm$ 10.1 & 0.039 $\pm$ 0.026 \\
PPO, $T$\,=\,1.3 & 53.6 $\pm$ 7.6 & 58.8 $\pm$ 10.6 & 47.2 $\pm$ 4.8 & 0.100 $\pm$ 0.022 & 60.9 $\pm$ 12.1 & 70.3 $\pm$ 12.0 & 47.5 $\pm$ 17.5 & 0.042 $\pm$ 0.026 \\
PPO, $T$\,=\,1.5 & 57.4 $\pm$ 11.6 & 64.4 $\pm$ 11.1 & \textbf{52.2 $\pm$ 7.8} & 0.102 $\pm$ 0.017 & 62.8 $\pm$ 6.8 & 68.4 $\pm$ 6.9 & \textbf{54.7 $\pm$ 10.0} & 0.055 $\pm$ 0.015 \\
\midrule
\textbf{PPO with PTGS} & \textbf{61.1 $\pm$ 8.6} & \textbf{69.7 $\pm$ 9.7} & 49.7 $\pm$ 14.9 & 0.081 $\pm$ 0.039 & \textbf{65.0 $\pm$ 3.2} & \textbf{80.0 $\pm$ 8.3} & 42.5 $\pm$ 15.2 & 0.020 $\pm$ 0.047 \\
\bottomrule
\end{tabular}%
}
\end{table}

\begin{table}[!htb]
\centering
\caption{\textbf{Rollout decoding ablation of PPO, compared with PTGS}. Each PPO row applies one truncation rule to the training rollouts at $T{=}1.0$, and the shaded row is the PPO baseline with ancestral sampling (the same runs as $T{=}1.0$ in Table~\ref{tab:temp-ablation-ci}). Values are the mean over five runs $\pm$ the 95\% confidence interval.}
\label{tab:sampler-ablation-ci}
\small
\setlength{\tabcolsep}{4pt}
\resizebox{\linewidth}{!}{%
\begin{tabular}{l cccc | cccc}
\toprule
& \multicolumn{4}{c}{Sokoban} & \multicolumn{4}{|c}{FrozenLake} \\
\cmidrule(lr){2-5}\cmidrule(lr){6-9}
Method & pass@$1$ & pass@$128$ & $\text{pass}^{128}$ & $\text{Tax}_{S}(128)$ & pass@$1$ & pass@$128$ & $\text{pass}^{128}$ & $\text{Tax}_{S}(128)$ \\
\midrule
Qwen2.5-7B-Instruct (base) & 20.7 & 76.6 & 0.0 & -- & 26.3 & 89.1 & 0.0 & -- \\
\midrule
\rowcolor{gray!10}
PPO, ancestral (default) & 46.5 $\pm$ 14.7 & 55.0 $\pm$ 11.3 & 36.2 $\pm$ 13.1 & 0.094 $\pm$ 0.016 & 63.7 $\pm$ 5.9 & 74.1 $\pm$ 7.2 & 50.9 $\pm$ 10.1 & 0.039 $\pm$ 0.026 \\
PPO, top-$p$\,0.9 & 42.5 $\pm$ 12.7 & 46.6 $\pm$ 15.2 & 33.1 $\pm$ 9.2 & 0.111 $\pm$ 0.012 & \textbf{66.4 $\pm$ 6.2} & 74.1 $\pm$ 7.1 & \textbf{52.8 $\pm$ 13.3} & 0.041 $\pm$ 0.039 \\
PPO, top-$p$\,0.95 & 47.3 $\pm$ 2.8 & 55.9 $\pm$ 11.4 & 36.6 $\pm$ 9.9 & 0.091 $\pm$ 0.054 & 61.9 $\pm$ 8.0 & 72.5 $\pm$ 4.5 & 46.2 $\pm$ 16.4 & 0.030 $\pm$ 0.027 \\
PPO, top-$k$\,20 & 45.5 $\pm$ 9.3 & 55.3 $\pm$ 14.2 & 37.5 $\pm$ 7.3 & 0.092 $\pm$ 0.028 & 61.3 $\pm$ 7.3 & 75.6 $\pm$ 13.2 & 40.0 $\pm$ 24.7 & 0.031 $\pm$ 0.034 \\
PPO, top-$k$\,40 & 45.1 $\pm$ 19.4 & 57.2 $\pm$ 15.3 & 27.8 $\pm$ 18.3 & 0.082 $\pm$ 0.009 & 64.6 $\pm$ 1.2 & 75.6 $\pm$ 6.2 & 51.9 $\pm$ 9.4 & 0.034 $\pm$ 0.024 \\
PPO, min-$p$\,0.05 & 44.7 $\pm$ 9.6 & 55.0 $\pm$ 18.0 & 36.9 $\pm$ 6.9 & \textbf{0.076 $\pm$ 0.048} & 61.3 $\pm$ 6.7 & 75.3 $\pm$ 15.3 & 44.1 $\pm$ 12.8 & \textbf{0.011 $\pm$ 0.080} \\
PPO, min-$p$\,0.1 & 41.1 $\pm$ 13.4 & 55.0 $\pm$ 19.1 & 29.1 $\pm$ 10.5 & 0.081 $\pm$ 0.039 & 60.0 $\pm$ 8.5 & 75.0 $\pm$ 9.9 & 40.6 $\pm$ 14.2 & 0.028 $\pm$ 0.033 \\
\midrule
\textbf{PPO with PTGS} & \textbf{61.1 $\pm$ 8.6} & \textbf{69.7 $\pm$ 9.7} & \textbf{49.7 $\pm$ 14.9} & 0.081 $\pm$ 0.039 & 65.0 $\pm$ 3.2 & \textbf{80.0 $\pm$ 8.3} & 42.5 $\pm$ 15.2 & 0.020 $\pm$ 0.047 \\
\bottomrule
\end{tabular}%
}
\end{table}

\noindent\textbf{Is PTGS just a hotter or truncated sampler?} PTGS changes only the temperature of the training rollouts, so one may ask whether a fixed rollout temperature or a fixed truncation rule gives the same benefit. Tables~\ref{tab:temp-ablation-ci} and~\ref{tab:sampler-ablation-ci} vary only the rollout sampler of PPO under the setup of Table~\ref{tab:ptgs-combined} (Appendix~\ref{appendix:ptgs_setup}). The first table uses a fixed temperature $T \in \{0.5, 0.7, 1.0, 1.3, 1.5\}$, where $T{=}1.5$ is the highest temperature that PTGS can assign. The second keeps $T{=}1.0$ and truncates the rollout distribution with top-$p$, top-$k$, or min-$p$ decoding. The results give several implications. First, a fixed rollout temperature moves the policy along the accuracy-coverage trade-off without improving both. On Sokoban, a higher temperature tends to improve pass@$1$ and consistency, but it also increases the tax. On FrozenLake, every temperature above the default lowers pass@$128$ and increases the tax, and the only fixed temperature with a clearly smaller tax ($T{=}0.5$) has the lowest pass@$1$. Second, truncation rules barely change the baseline. Most variants stay within about two points of the ancestral baseline in pass@$128$, and none of them gives a clear gain in pass@$1$ or pass@$128$. Third, PTGS is the only sampler that improves accuracy and coverage at the same time. It gives the best pass@$128$ in both environments and the best or near-best pass@$1$, with a smaller tax than every fixed temperature $T \ge 1.0$. Its consistency on FrozenLake is lower than that of the default, although the intervals overlap. \textbf{Prompt-adaptive tempering is therefore not equivalent to a global temperature or a fixed truncation of the rollout distribution}. This agrees with Theorem~\ref{thm:ptgs-signal}, which suggests that what matters is heating the prompts that the policy keeps failing, rather than heating or truncating all prompts in the same way.

\FloatBarrier
\subsection{Inference-Time PTGS Application}
\label{appendix:ptgs_inference}

\noindent\textbf{Estimate-then-sample decoding.} The tempering rule of PTGS needs only a per-prompt difficulty estimate, so it can also be applied to a frozen policy at test time. For each test prompt $x$, we split a budget of $B$ rollouts into $m$ pilot rollouts and $n = B - m$ solution rollouts. We draw $m$ pilot rollouts at a reference temperature $T_{\mathrm{ref}}$ and count their successes $s_x$, form the posterior $\mathrm{Beta}\bigl(2\tilde{p} + s_x,\ 2(1-\tilde{p}) + m - s_x\bigr)$ with the same prior centered on the target success rate as in training, draw $\hat{p}_x$ from it by Thompson sampling, and generate the $n$ solution rollouts at $T_x = T_{\mathrm{ref}}\,\tau^{h(\hat{p}_x)} \in [T_{\mathrm{ref}}/\tau,\ \tau T_{\mathrm{ref}}]$ with $h$ from Eq.~\eqref{eq:ptgs}. Unlike training-time PTGS, the posterior is built from scratch for each prompt and discarded afterwards, without a forgetting factor or a $\tilde{p}$ schedule, and no log-probability correction is needed because the policy is not updated. Only the $n$ solution rollouts are scored, and pass@$k$ for $k \le n$ is computed from their success count with the unbiased estimator of \S\ref{section:prelim}.

\noindent\textbf{Setup.} We apply inference-time PTGS to the frozen \texttt{Qwen2.5-7B-Instruct}, i.e., the base model of Table~\ref{tab:ptgs-combined}, on the same 64 offline evaluation tasks of Sokoban and FrozenLake (\S\ref{appendix:ptgs_setup}). We use $B = 256$ with $m = n = 128$, $T_{\mathrm{ref}} = 0.5$ (the RAGEN evaluation default), and $\tilde{p} = 0.5$, and vary $\tau \in \{1.2, 1.3, \dots, 2.0\}$. The pilot rollouts are scored by the same binary environment success as the evaluation, so this setting measures what per-prompt adaptivity can offer with a reliable difficulty signal. As a baseline, we evaluate a global temperature $T \in \{0.1, 0.3, 0.5, 0.7, 1.0, 1.3, 1.5\}$ with ancestral sampling and 128 scored rollouts per task.

\begin{figure}[thb]
     \centering
     \includegraphics[width=0.9\textwidth]{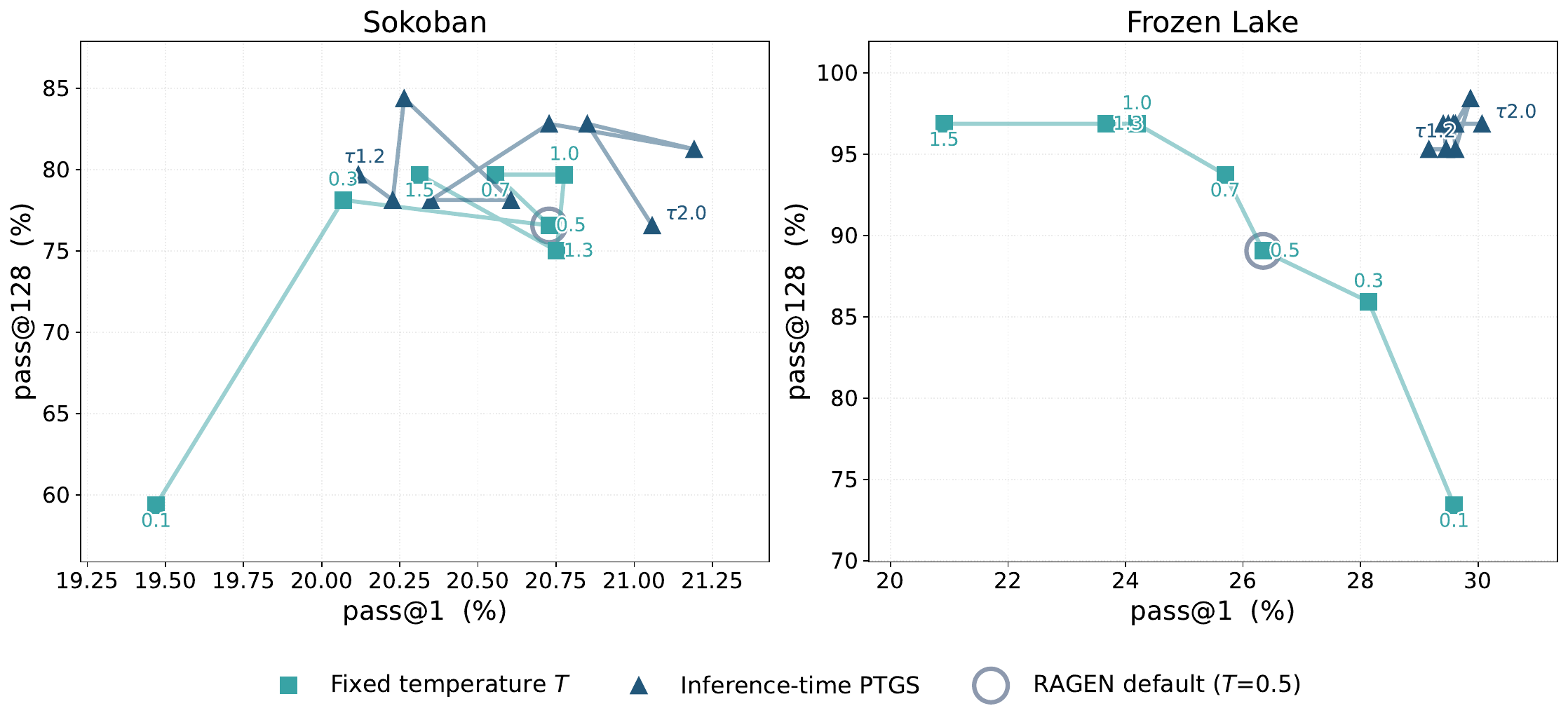}
     \vspace{-0.5em}
     \caption{\textbf{Inference-time PTGS versus global temperature scaling.} pass@$1$ ($x$-axis) and pass@$128$ ($y$-axis) of the frozen \texttt{Qwen2.5-7B-Instruct} on 64 tasks of each environment, where the upper right is better. Squares show global temperatures $T \in [0.1, 1.5]$, triangles show inference-time PTGS with $\tau \in [1.2, 2.0]$ and $T_{\mathrm{ref}} = 0.5$, and the ring marks the default inference temperature $T = 0.5$. PTGS lies beyond the global-temperature curve in both environments.}
     \vspace{-0.5em}
     \label{fig:ptgs-inference}
\end{figure}

\noindent\textbf{Can per-prompt tempering improve the trade-off without training?} Appendix~\ref{appendix:global_temperature} showed that a global temperature cannot improve accuracy and coverage together. Figure~\ref{fig:ptgs-inference} tests whether inference-time PTGS does better on a frozen policy. On FrozenLake, the global temperature gives a clear monotone trade-off, where a lower temperature improves pass@$1$ but reduces pass@$128$. All PTGS configurations lie beyond this curve and get close to the accuracy of the coldest temperature and the coverage of the hottest one at the same time. On Sokoban, the global temperature shows no clear trade-off above $T{=}0.1$, and only two PTGS configurations exceed every fixed temperature on both axes, by one or two of the 64 tasks in pass@$128$.

\section{Deeper Analysis of PTGS}
\suppressfloats[t]
\subsection{Accuracy of the Difficulty Estimate and Realized Temperatures}
\label{appendix:ptgs_mech_1}

\begin{figure}[t]
\centering
\begin{minipage}[t]{0.485\textwidth}
\centering
\includegraphics[height=2in]{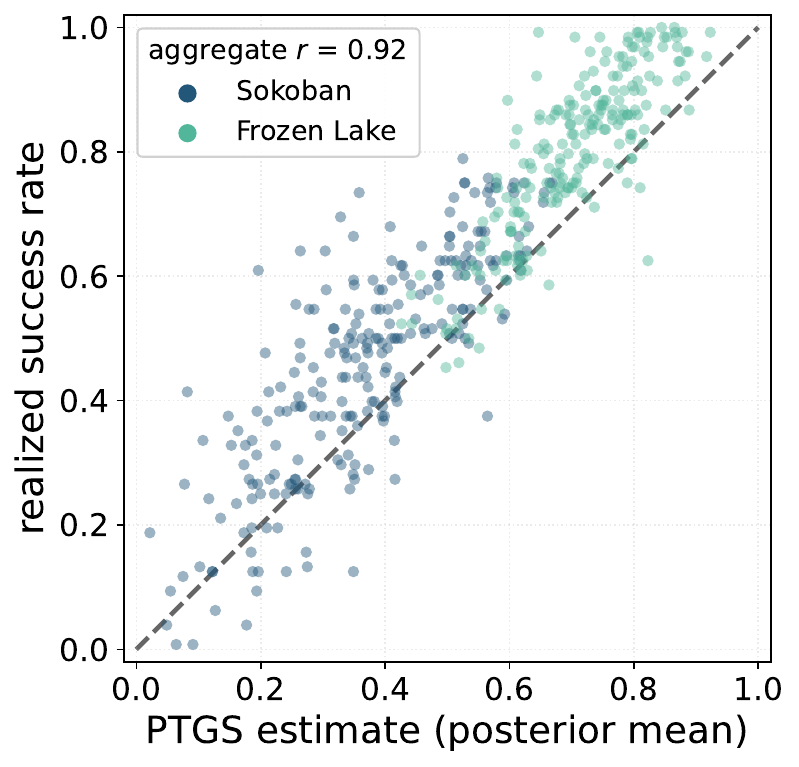}
\vspace{-0.5em}
\caption{\textbf{The difficulty estimate of PTGS tracks the realized success at the end of training.} Each point is one of the last 50 training steps of one of five PPO w/ PTGS runs per environment. The $x$-axis is the mean posterior estimate $\bar{p}_x$ over the 8 prompts of the step, recorded before the rollouts are drawn, and the $y$-axis is the realized success rate of its 128 rollouts.}
\label{fig:ptgs-mech-est-scatter}
\end{minipage}\hfill
\begin{minipage}[t]{0.485\textwidth}
\centering
\includegraphics[height=2in]{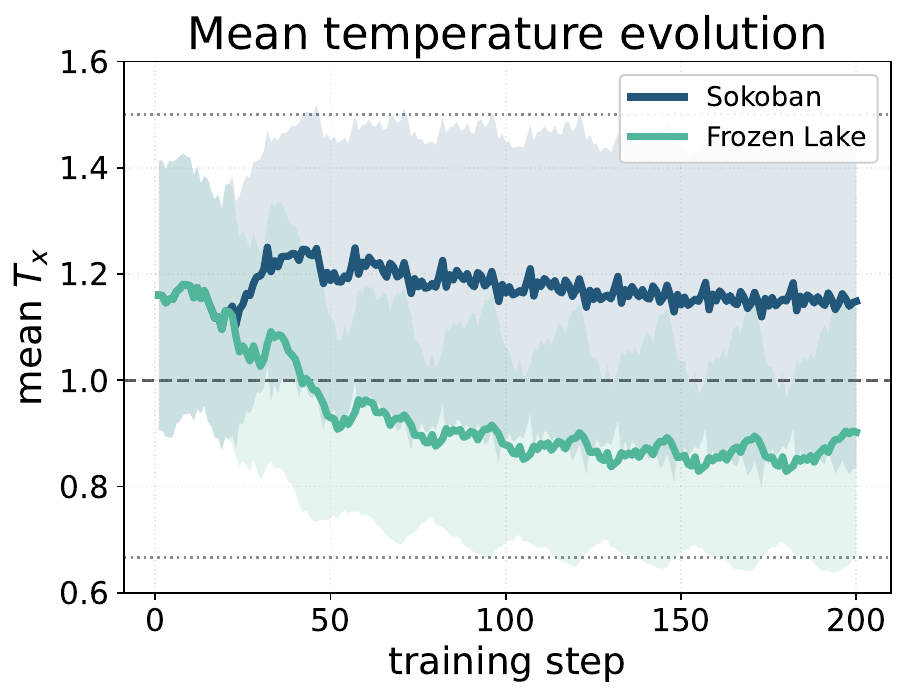}
\vspace{-0.5em}
\caption{\textbf{Mean rollout temperature of PTGS over training on Sokoban and FrozenLake.} Mean temperature $T_x$ over the 8 prompts of each step (9-step moving average over five runs). Bands show $\pm1$ standard deviation of $T_x$ across the prompts of a step. On average, PTGS heats on the harder Sokoban and cools on FrozenLake once the policy solves most prompts, while using a range of temperatures across prompts in both environments.}
\label{fig:ptgs-mech-temp}
\end{minipage}
\end{figure}

\begin{figure}[t]
\centering
\begin{minipage}[c]{0.47\textwidth}
\centering
\includegraphics[width=\linewidth]{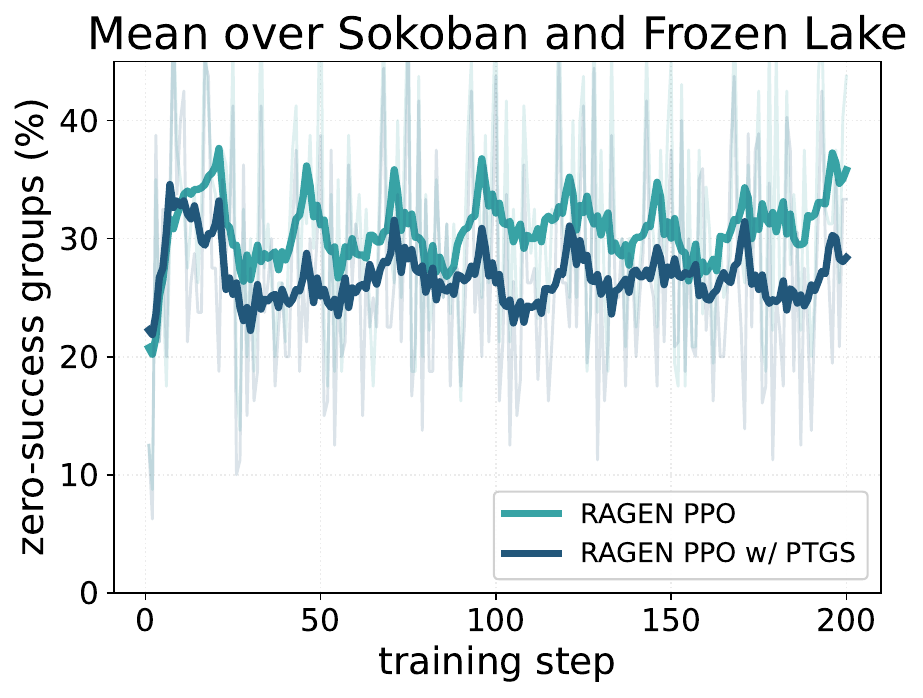}
\end{minipage}\hfill
\begin{minipage}[c]{0.5\textwidth}
\caption{\textbf{PTGS reduces the proportion of zero-success rollout groups.} Percentage of the rollout groups at each training step in which none of the 16 rollouts succeeds, i.e., groups that carry no success signal for the update, for PPO and PPO with PTGS. The proportion is measured before the reward-variance filter and averaged over Sokoban and FrozenLake. Thin lines show the mean over the ten runs of each method (five per environment) with a 9-step moving average as the thick line.}
\label{fig:ptgs-mech-zero}
\end{minipage}
\end{figure}

\noindent\textbf{Is the difficulty estimate of PTGS accurate?}
PTGS chooses the temperature of each prompt from a running Beta posterior over its success rate, so its benefit depends on whether this estimate predicts the actual success of the policy. To check this, we record the posterior mean $\bar{p}_x$ of each prompt before its rollout group is drawn and compare it with the realized success rate of the group. From early in training, the estimate separates the 8 prompts of each step well by their realized success, with an average Pearson correlation of 0.81 on Sokoban and 0.71 on FrozenLake. On FrozenLake, the correlation becomes stronger as the posterior accumulates evidence (Figure~\ref{fig:ptgs-mech-est-scatter}). Note that the estimate is conservative (the realized success is about 0.1 higher on average) because the discounted counts still reflect the weaker policy of earlier steps and the prior centered on $\tilde{p} < 1/2$. Since a conservative estimate errs toward heating, PTGS tends to explore more.

\newpage
\noindent\textbf{How hot does PTGS actually sample?}
Figure~\ref{fig:ptgs-mech-temp} shows that PTGS does not simply sample at a higher temperature than PPO. On Sokoban, where the policy remains weak, the mean temperature stays above $T{=}1$ throughout training. On FrozenLake, where the policy solves most prompts after the first 50 steps, PTGS cools on average, from about 1.1 early in training to below 0.9 at the end. In both environments, the temperatures of prompts within the same step vary widely, with a standard deviation of about 0.2 to 0.3. In other words, PTGS heats the prompts that the policy fails and cools the ones it solves at the same time, instead of shifting a single global temperature.

\noindent\textbf{Does PTGS give RL more to learn from?}
A rollout group in which none of the 16 rollouts succeeds carries no success signal for the update, and Theorem~\ref{thm:ptgs-signal}~(i) states that heating a hard prompt makes such groups rarer. Averaged over the two environments, PTGS keeps the proportion of these zero-success groups below that of PPO for most of training (Figure~\ref{fig:ptgs-mech-zero}). That is, PTGS indeed turns more of the rollout budget into groups that carry a learning signal.

\subsection{Does PTGS Induce a More Exploratory Policy?}
\label{appendix:ptgs_mech_2}

\noindent\textbf{Is the higher entropy just a higher temperature?}
The entropy curves in Appendix~\ref{appendix:deferred_results} compare PTGS with PPO at the default temperature $T{=}1$. However, that entropy is computed at the rollout temperature of each run, so part of the gap could come from PTGS sampling at higher temperatures. Therefore, we also compare PTGS with PPO runs that sample every rollout at $T{=}1.3$ or $T{=}1.5$ (Figure~\ref{fig:ptgs-mech-explore}). On Sokoban, PTGS keeps a much higher entropy than both heated PPO variants over the last 50 steps. On FrozenLake, PPO at $T{=}1.5$ reaches an entropy similar to that of PTGS, but only by sampling every rollout at a temperature about 60\% above the mean temperature of PTGS.

\noindent\textbf{Does the extra exploration improve coverage?}
The extra entropy also has different effects under global heating and under PTGS (Figure~\ref{fig:ptgs-mech-explore}, middle and right). On Sokoban, heating every rollout to $T{=}1.5$ improves both pass@$1$ and pass@$128$ over the default, but both remain below PTGS. On FrozenLake, heating even lowers pass@$128$ below the default. PTGS instead achieves the highest pass@$1$ and pass@$128$ at step 200 in both environments. In summary, \textbf{the higher entropy of PTGS is not simply due to a higher sampling temperature, and unlike the entropy from a global temperature, it comes with higher coverage}. 

\begin{figure*}[t]
\centering
\includegraphics[width=0.9\textwidth]{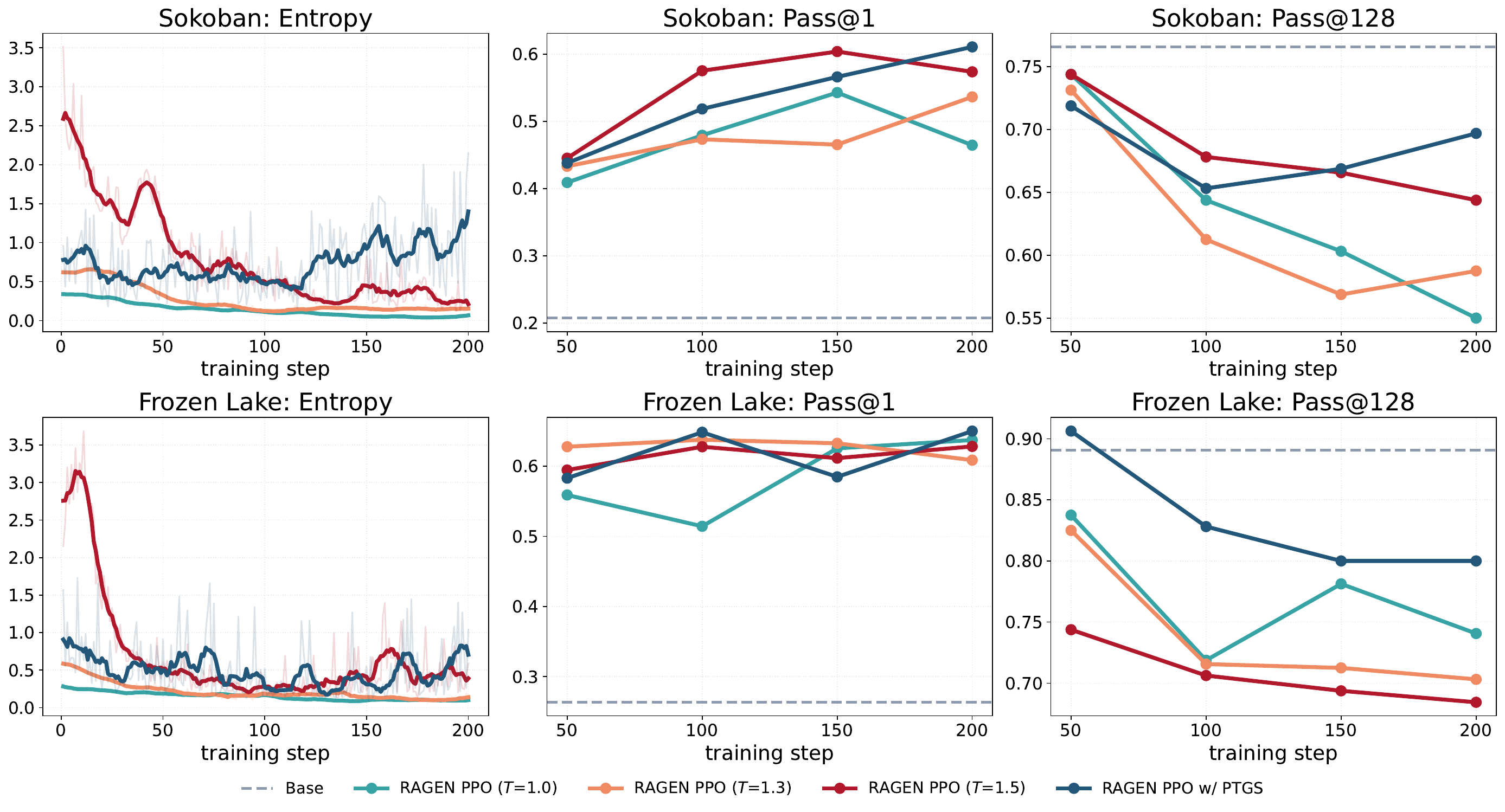}
\vspace{-0.5em}
\caption{\textbf{PTGS versus PPO with a globally higher temperature} on Sokoban (top) and FrozenLake (bottom). We compare PPO at the default rollout temperature ($T{=}1.0$), PPO with every rollout sampled at $T{=}1.3$ or $T{=}1.5$ (the highest temperature PTGS can assign), and PPO with PTGS ($\tau{=}1.5$). The left panels show the average entropy of the output token distribution during training, computed at the rollout temperature of each run, with thin lines for the mean over five runs and thick lines for a 9-step moving average. The middle and right panels show pass@$1$ and pass@$128$ of the intermediate checkpoints, averaged over five runs, and the dashed line marks the base model.}
\label{fig:ptgs-mech-explore}
\end{figure*}

\subsection{Case Study: Tasks That PPO Loses and PTGS Recovers}
\label{appendix:ptgs_mech_3}

We say that an RL post-training run \emph{loses} a task if the base model solves it at least once in 128 rollouts but the run never does. Among the tasks lost by a PPO run, a PPO w/ PTGS run solves about half, whereas the reverse rate is much lower. Figures~\ref{fig:ptgs-mech-example-sokoban} and~\ref{fig:ptgs-mech-example-frozenlake} show one such task per environment. In both cases, the base model solves the task only occasionally, and PPO never solves it in four of its five runs. PTGS, in contrast, solves it in most runs and masters it in some. This is consistent with Appendix~\ref{appendix:ptgs_signal}, i.e., by heating a prompt that the policy keeps failing, PTGS keeps it contributing to learning until the policy solves it reliably.

\begin{figure*}[t]
\centering
\begin{taskbanner}
\small\textbf{Sokoban task \texttt{sokoban\_25}:}\; push the box onto the target on a $6\times6$ grid (walls in dark gray, box in brown, target circled).\\
\textbf{Base model's winning opening:} \emph{``The box is one left of the target. Pushing it directly would move it off the target. Moving right first would place the box next to the target, then push.''}\; \texttt{Right \textbar\textbar Push Right}
\end{taskbanner}
\vspace{1mm}
\includegraphics[width=0.74\textwidth]{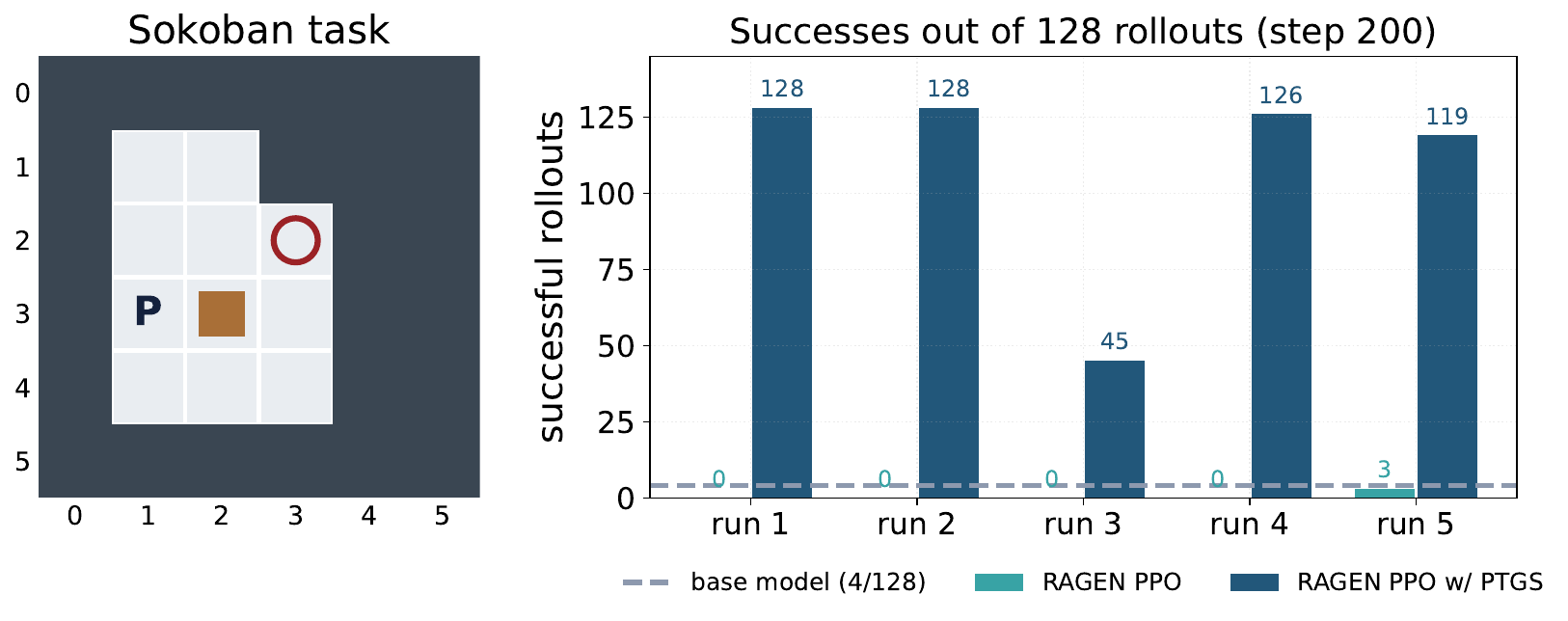}
\vspace{-0.5em}
\caption{\textbf{A Sokoban task that PPO loses and PTGS recovers.} The above example case describes the task and the opening of a successful rollout of the base model, and the left panel shows the initial state. The right panel shows the number of successful rollouts out of 128 at step 200 in each of the five training runs of PPO and PPO with PTGS, and the dashed line marks the base model, which solves the task in 4 of 128 rollouts. PPO never solves the task in four of its five runs, but PPO w/ PTGS solves it in all five runs and masters it in two.}
\label{fig:ptgs-mech-example-sokoban}

\vspace{1em}
\begin{taskbanner}
\small\textbf{FrozenLake task \texttt{frozen\_lake\_43}:}\; reach the goal \textsf{G} on slippery $4\times4$ ice without falling into the hole (dark circle).\\
\textbf{Base model's winning opening:} \emph{``The player is at (2,2) and needs to move left twice to reach the goal at (3,0). However, be cautious of the hole at (3,1).''}\; \texttt{Left \textbar\textbar Left}
\end{taskbanner}
\vspace{1mm}
\includegraphics[width=0.74\textwidth]{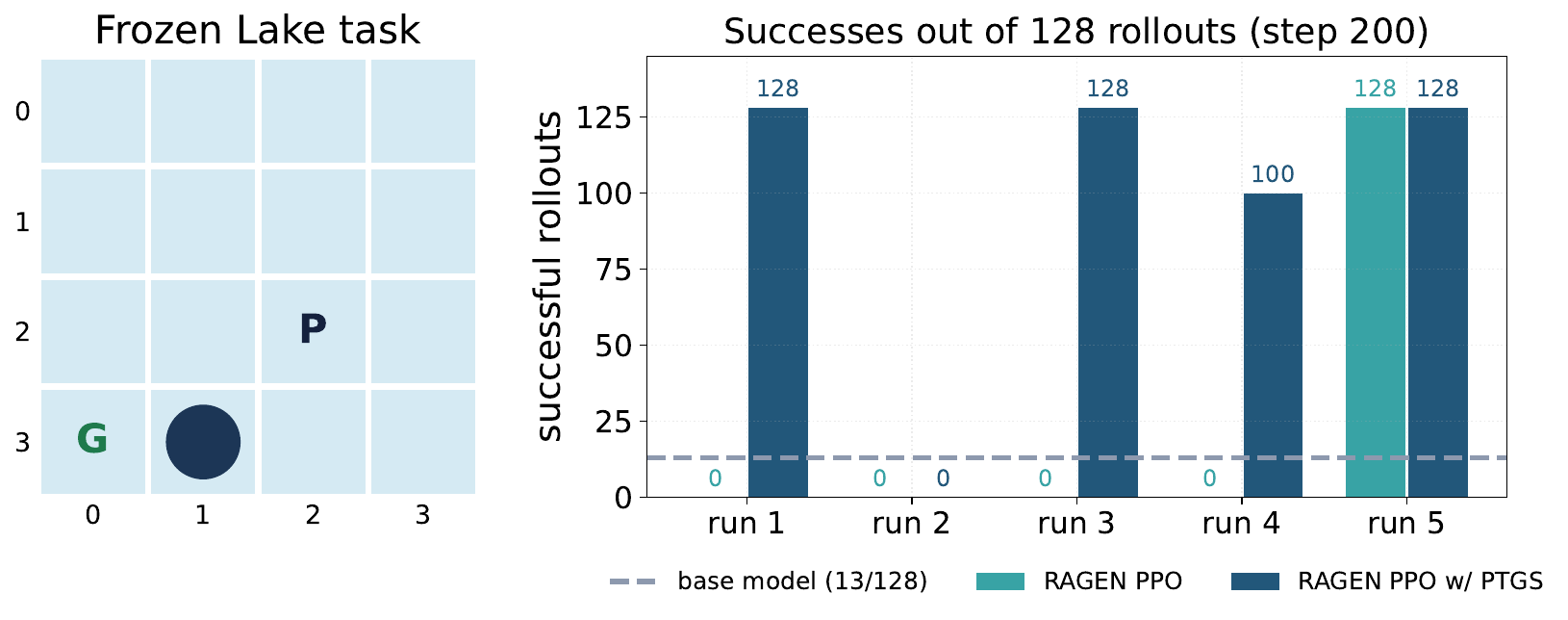}
\vspace{-0.5em}
\caption{\textbf{A FrozenLake task that PPO loses and PTGS recovers.} The above example case describes the task and the opening of a successful rollout of the base model, and the left panel shows the initial state. The right panel shows the number of successful rollouts out of 128 at step 200 in each of the five training runs of PPO and PPO with PTGS, and the dashed line marks the base model, which solves the task in 13 of 128 rollouts. PPO never solves the task in four of its five runs, but PPO w/ PTGS solves it in four runs and masters it in three. }
\label{fig:ptgs-mech-example-frozenlake}
\end{figure*}

\end{document}